\documentclass[10pt]{IEEEtran}
\usepackage{amsmath,graphicx,psfrag,epsfig,color,amsfonts,subfigure}
\usepackage{amsthm,amssymb}
\usepackage{version}

\graphicspath{{./fig/}}

\def\prob{\mathbb{P}}
\def\expt{\mathbb{E}}
\def\real{\mathbb{R}}

\def\naturals{\mathbb{N}}
\newcommand{\bbN}{\ensuremath{\mathbb{N}}}

\newcommand{\mcN}{\ensuremath{\mathcal{N}}}
\newcommand{\dd}{\mathrm{d}}

\newcommand{\until}[1]{\{1,\dots, #1\}}

\newcommand{\supscr}[2]{#1^{\textup{#2}}}
\newcommand{\setdef}[2]{\{#1 \; | \; #2\}}
\newcommand{\seqdef}[2]{\{#1\}_{#2}}

\newcommand{\map}[3]{#1: #2 \rightarrow #3}
\newcommand{\union}{\operatorname{\cup}}

\newcommand{\argmax}{\operatorname{argmax}}

\renewcommand{\Pr}[1]{\ensuremath{\mathbb{P} \left( #1 \right)}}
\newcommand{\E}[1]{\ensuremath{\mathbb{E}\left[ #1 \right]}}
\newcommand{\indicator}[1]{\ensuremath{\mathbf{1}\! \left( #1 \right)}}

\usepackage[vlined,ruled]{algorithm2e}

\newcommand\oprocendsymbol{\hbox{$\square$}}
\newcommand\oprocend{\relax\ifmmode\else\unskip\hfill\fi\oprocendsymbol}

\newcommand{\beq}{\begin{equation}}
\newcommand{\eeq}{\end{equation}}
\newcommand{\bal}{\begin{align}}
\newcommand{\eal}{\end{align}}

\newtheorem{theorem}{Theorem}

\newtheorem{remark}[theorem]{Remark}

\newtheorem{prop}[theorem]{Proposition}

\def\mc{\mathcal}

\def\bs{\boldsymbol}

\def \etal {\emph{et al.}}

\newcommand\bit[1]{\textit{\textbf{#1}}}

\title{\huge Modeling Human Decision-making \\
in Generalized Gaussian Multi-armed Bandits
\thanks{This research has been supported in part by ONR grant N00014-09-1-1074 and ARO grant W911NG-11-1-0385. P. Reverdy is supported through an NDSEG Fellowship. 
Preliminary versions of parts of this work were presented at IEEE CDC 2012~\cite{PR-RCW-PH-NEL:12} and Allerton 2013~\cite{VS-PR-NEL:13}.  In addition to improving on the ideas in~\cite{PR-RCW-PH-NEL:12,VS-PR-NEL:13}, this paper improves the analysis of algorithms and compares the performance of these algorithms against empirical data. The human behavioral experiments were approved under Princeton University Institutional Review Board protocol number 4779.
} 
}
\author{Paul Reverdy \hspace{0.5in} Vaibhav Srivastava \hspace{0.5in} Naomi Ehrich Leonard
\thanks{P. Reverdy, V. Srivastava, and  N. E. Leonard are with Department of Mechanical and Aerospace Engineering, Princeton University, Princeton, NJ 08544, USA {\tt \{preverdy, vaibhavs, naomi\} @ princeton.edu}.
}
}

\begin{document}
\maketitle

\begin{abstract}
We present a formal model of human decision-making in explore-exploit tasks using the context of multi-armed bandit problems, where the decision-maker must choose among multiple options with uncertain rewards.   We address the standard multi-armed bandit problem, the multi-armed bandit problem with transition costs, and the multi-armed bandit problem on graphs.  We focus on the case of Gaussian rewards  in a  setting where the decision-maker uses Bayesian inference to estimate the reward values.   We model the decision-maker's prior knowledge with the Bayesian prior on the mean reward.  We develop the upper credible limit (UCL) algorithm for the standard multi-armed bandit problem and show that this deterministic algorithm achieves logarithmic cumulative expected regret, which is optimal performance for uninformative priors.  We show how good priors and good assumptions on the correlation structure among arms can greatly enhance decision-making performance, even over short time horizons.  We extend to the stochastic UCL algorithm and draw several connections to human decision-making behavior.   We present empirical data from human experiments and show that human performance is efficiently captured by the stochastic UCL algorithm with appropriate parameters. For the multi-armed bandit problem with transition costs and the multi-armed bandit problem on graphs, we generalize the UCL algorithm to the block UCL algorithm and the graphical block UCL algorithm, respectively. We show that these algorithms also achieve logarithmic cumulative expected regret and require a sub-logarithmic expected number of transitions among arms. We further illustrate the performance of these algorithms with numerical examples.

NB: Appendix G included in this version details minor modifications that correct for an oversight in the previously-published proofs. The remainder of the text reflects the published work.
\end{abstract}


\section{Introduction}
Imagine the following scenario: you are reading the menu in a new restaurant, deciding which dish to order. Some of the dishes are familiar to you, while others are completely new. Which dish do you ultimately order: a familiar one that you are fairly certain to enjoy, or an unfamiliar one that looks interesting but you may dislike?

Your answer will depend on a multitude of factors, including your mood that day (Do you feel adventurous or conservative?), your knowledge of the restaurant and its cuisine (Do you know little about African cuisine, and everything looks new to you?), and the number of future decisions the outcome is likely to influence (Is this a restaurant in a foreign city you are unlikely to visit again, or is it one that has newly opened close to home, where you may return many times?). This scenario encapsulates many of the difficulties faced by a decision-making agent interacting with his/her environment, e.g. the role of prior knowledge and the number of future choices (time horizon).

The problem of learning the optimal way to interact with an uncertain environment is common to a variety of areas of study in engineering  such as adaptive control and reinforcement learning \cite{FLL-DV-KGV:12}. Fundamental to these problems is the tradeoff between exploration (collecting more information to reduce uncertainty) and exploitation (using the current information to maximize the immediate reward). Formally, such problems are often formulated as Markov Decision Processes (MDPs). MDPs are decision problems in which the decision-making agent is required to make a sequence of choices along  a process evolving in time~\cite{RSS-AGB:98}. The theory of dynamic programming \cite{RB:57},~\cite{PLK-MLL-AWM:96} provides methods to find optimal solutions to generic MDPs, but is subject to the so-called \emph{curse of dimensionality} \cite{RSS-AGB:98}, where the size of the problem often grows exponentially in the number of states.

The curse of dimensionality makes finding the optimal solution difficult, and in general intractable for finite-horizon problems of any significant size. Many engineering solutions of MDPs consider the infinite-horizon case, i.e., the limit where the agent will be required to make an infinite sequence of decisions. In this case, the problem simplifies significantly and a variety of reinforcement learning methods can be used to converge to the optimal solution, for example \cite{CW-PD:92,PLK-MLL-AWM:96,RSS-AGB:98,FLL-DV-KGV:12}. However, these methods only converge to the optimal solution asymptotically at a rate that is difficult to analyze. The UCRL algorithm \cite{PA-RO:07} addressed this issue by deriving a heuristic-based reinforcement learning algorithm with a provable learning rate.  

However, the infinite-horizon limit may be inappropriate for finite-horizon tasks. In particular, optimal solutions to the finite-horizon problem may be strongly dependent on the task horizon. Consider again our restaurant scenario. If the decision is a one-off, we are likely to be conservative, since selecting an unfamiliar option is risky and even if we choose an unfamiliar dish and like it, we will have no further opportunity to use the information in the same context. However, if we are likely to return to the restaurant many times in the future, discovering new dishes we enjoy is valuable.

Although the finite-horizon problem may be intractable to computational analysis, humans are confronted with it all the time, as evidenced by our restaurant example. The fact that they are able to find efficient solutions quickly with inherently limited computational power suggests that humans employ relatively sophisticated heuristics for solving these problems. Elucidating these heuristics is of interest both from a psychological point of view where they may help us understand human cognitive control and from an engineering point of view where they may lead to development of improved algorithms to solve MDPs \cite{JDC-SMM-AJY:07}. In this paper, we seek to elucidate the behavioral heuristics at play with a model that is both mathematically rigorous and computationally tractable.

Multi-armed bandit problems~\cite{JG-KG-RW:11} constitute a class of MDPs that is well suited to our goal of connecting biologically plausible heuristics with mathematically rigorous algorithms.  In the mathematical context, multi-armed bandit problems have been studied in both the infinite-horizon and finite-horizon cases. There is a well-known optimal solution to the infinite-horizon problem \cite{JCG:79}.  For the finite-horizon problem, the policies are designed to match the best possible performance established in \cite{TLL-HR:85}.   In the biological context, the decision-making behavior and performance of both animals and humans have been studied using the multi-armed bandit framework.

In a multi-armed bandit problem, a decision-maker allocates a single resource by sequentially choosing one among a set of competing alternative options called arms. In the so-called stationary multi-armed bandit problem, a decision-maker at each discrete time instant chooses an arm and collects a reward drawn from an unknown stationary probability distribution associated with the selected arm. The objective of the decision-maker is to maximize the total reward aggregated over the sequential allocation process.  We will refer to this as the {\em standard} multi-armed bandit problem, and we will consider variations that add transition costs or spatial unavailability of arms.   A classical example of a standard multi-armed bandit problem is the evaluation of clinical trials with medical patients described in~\cite{WRT:33}.   The decision-maker is a doctor and the options are different treatments with unknown effectiveness for a given disease.  Given patients that arrive and get treated sequentially, the objective for the doctor is to maximize the number of cured patients, using information gained from successive outcomes. 

Multi-armed bandit problems capture the fundamental exploration-exploitation tradeoff.   Indeed, they model a wide variety of real-world decision-making scenarios including those associated with foraging and search in an uncertain environment.   The rigorous examination in the present paper of the heuristics that humans use in multi-armed bandit tasks can help in understanding and enhancing both natural and engineered strategies and performance in these kinds of tasks.    For example, a trained human operator can quickly learn the relevant features of a new environment, and an efficient model for human decision-making in a multi-armed bandit task may facilitate a means to learn a trained operator's task-specific knowledge for use in an autonomous decision-making algorithm.   Likewise, such a model may help in detecting weaknesses in a human operator's strategy and deriving computational means to augment human performance.  

Multi-armed bandit problems became popular following the seminal paper by Robbins~\cite{HR:52} and found application in diverse areas including controls, robotics, machine learning, economics, ecology, and operational research~\cite{MB-YS-AS:09, FR-RK-TJ:08,JLN-MD-EF:08,BPM-JJM:87,MYC-JL-FSH:13}.  For example, in ecology the multi-armed bandit problem was used to study the foraging behavior of birds in an unknown environment~\cite{JRK-AK-PT:78}.  The authors showed that the optimal policy for the two-armed bandit problem captures well the observed foraging behavior of birds.  
Given the limited computational capacity of birds, it is likely they use  simple heuristics to achieve near-optimal performance. The development of  simple heuristics in this and other contexts has spawned a wide literature. 

Gittins~\cite{JCG:79} studied the infinite-horizon multi-armed bandit problem and developed a dynamic allocation index (Gittins' index) for each arm.  He showed that selecting an arm with the highest index at the given time results in the optimal policy. The dynamic allocation index, while a powerful idea, suffers from two drawbacks: (i) it is hard to compute, and  (ii) it does not provide insight into the nature of the optimal policies. 

Much recent work on  multi-armed bandit problems focuses on a quantity termed \emph{cumulative expected regret}. The cumulative expected regret of a sequence of decisions is simply the cumulative difference between the expected reward of the options chosen and the maximum reward possible. In this sense, expected regret plays the same role as expected value in standard reinforcement learning schemes:  maximizing expected value is equivalent to minimizing cumulative expected regret. Note that this definition of regret is in the sense of an omniscient being who is aware of the expected values of all options, rather than in the sense of an agent playing the game. As such, it is not a quantity of direct psychological relevance but rather an analytical tool that allows one to characterize performance.

In a ground-breaking work, Lai and Robbins~\cite{TLL-HR:85} established a logarithmic lower bound on the expected number of times a sub-optimal arm needs to be sampled by an optimal policy, thereby showing that cumulative expected regret is bounded below by a logarithmic function of time. Their work established the best possible performance of any solution to the standard multi-armed bandit problem.  They also developed an algorithm based on an upper confidence bound on estimated reward and showed that this algorithm  achieves the performance bound asymptotically. In the following, we use the phrase \emph{logarithmic regret} to refer to cumulative expected regret being bounded above by a logarithmic function of time, i.e., having the same order of growth rate as the optimal solution. The calculation of the upper confidence bounds in~\cite{TLL-HR:85} involves tedious computations.  Agarwal~\cite{RA:95} simplified these computations to develop sample mean-based upper confidence bounds, and showed that the policies in~\cite{TLL-HR:85}  with these upper confidence bounds achieve logarithmic regret asymptotically.

In the context of bounded multi-armed bandits, i.e., multi-armed bandits in which the reward is sampled from a distribution with a bounded support,  Auer~\etal~\cite{PA-NCB-PF:02} developed upper confidence bound-based algorithms that achieve logarithmic regret uniformly in time; see~\cite{SB-NCB:12} for an extensive survey of upper confidence bound-based algorithms. Audibert~\etal~\cite{audibert2009exploration} considered upper confidence bound-based algorithms that take into account the empirical variance of the various arms. In a related work, Cesa-Bianchi~\etal~\cite{NCB-PF:98} analyzed a Boltzman allocation rule for bounded multi-armed bandit problems. Garivier~\etal~\cite{AG-OC:11} studied the KL-UCB algorithm, which uses upper confidence bounds based on the Kullback-Leibler divergence, and advocated its use in multi-armed bandit problems where the rewards are distributed according to a known exponential family.

The works cited above adopt a frequentist perspective, but a number of researchers have also considered MDPs and multi-armed bandit problems from a Bayesian perspective. Dearden~\etal~\cite{RD-NF-SR:98} studied general MDPs and showed that a Bayesian approach can substantially improve performance in some cases. Recently, Srinivas~\etal~\cite{NS-AK-SMK-MS:12} developed asymptotically optimal upper confidence bound-based algorithms for Gaussian process optimization. Agrawal~\etal~\cite{SA-NG:12} proved that a Bayesian algorithm known as Thompson Sampling is near-optimal for binary bandits with a uniform prior. Kauffman~\etal~\cite{EK-OC-AG:12} developed a generic Bayesian upper confidence bound-based algorithm and established its optimality for binary bandits with a uniform prior.  In the present paper we develop a similar Bayesian upper confidence bound-based algorithm for Gaussian multi-armed bandit problems and show that it achieves logarithmic regret for uninformative priors uniformly in time. 

Some variations of these multi-armed bandit problems have been studied as well. Agarwal~\etal~\cite{RA-MVH-DT:88} studied multi-armed bandit problems with transition costs, i.e., the multi-armed bandit problems in which a certain penalty is imposed each time the decision-maker switches from the currently selected arm. To address this problem, they developed an asymptotically optimal block allocation algorithm. Banks and Sundaram~\cite{JSB-RKS:94} show that, in general, it is not possible to define dynamic allocation indices (Gittins' indices) which lead to an optimal solution of the multi-armed bandit problem with switching costs. However, if the cost to switch to an arm from any other arm is a stationary random variable, then such indices exist. Asawa~and~Teneketzis~\cite{MA-DT:96} characterize qualitative properties of the optimal solution to the multi-armed bandit problem with switching costs, and establish sufficient conditions for the optimality of limited lookahead based techniques. A survey of multi-armed bandit problems with switching costs is presented in~\cite{TJ:04}.  In the present paper, we consider Gaussian multi-armed bandit problems with transition costs and develop a block allocation algorithm that achieves logarithmic regret for uninformative priors uniformly in time. Our block allocation scheme is similar to the scheme in~\cite{RA-MVH-DT:88}; however, our scheme incurs a smaller expected cumulative transition cost than the scheme in~\cite{RA-MVH-DT:88}. Moreover, an asymptotic analysis is considered in~\cite{RA-MVH-DT:88}, while our results hold uniformly in time.

Kleinberg~\etal~\cite{RK-AMN-YS:10} considered multi-armed bandit problems in which arms are not all available for selection at each time (sleeping experts) and analyzed the performance of upper confidence bound-based algorithms. In contrast to the temporal unavailability of arms in~\cite{RK-AMN-YS:10}, we consider a spatial unavailability of arms. In particular, we propose a novel multi-armed bandit problem, namely, the \emph{graphical multi-armed bandit} problem in which only a subset of the arms can be selected at the next allocation instance given the currently selected arm.  We develop a block allocation algorithm for such problems that achieves logarithmic regret for uninformative priors uniformly in time. 

Human decision-making in multi-armed bandit problems has also been studied in the cognitive psychology literature.  Cohen~\etal~\cite{JDC-SMM-AJY:07} surveyed  the exploration-exploitation tradeoff in humans and animals and discussed the mechanisms in the brain that mediate this tradeoff.   Acu{\~n}a~\etal~\cite{DA-PS:08} studied human decision-making in multi-armed bandits from a Bayesian perspective. They modeled the human subject's prior knowledge about the reward structure using conjugate priors to the reward distribution. They concluded that a policy using Gittins' index, computed from approximate Bayesian inference based on limited memory and finite step look-ahead, captures the empirical behavior in certain multi-armed bandit tasks. In a subsequent work~\cite{DEA-PS:10}, they showed that a critical feature of human decision-making in multi-armed bandit problems is structural learning, i.e., humans learn the correlation structure among different arms. 

Steyvers~\etal~\cite{MS-MDL-EJW:09} considered Bayesian models for multi-armed bandits parametrized by human subjects' assumptions about reward distributions and observed that there are individual differences that determine the extent to which people use optimal models rather than simple heuristics. 
In a subsequent work, Lee~\etal~\cite{MDL-SZ-MM-MS:11} considered latent models in which there is a latent mental state that determines if the human subject should explore or exploit. Zhang~\etal~\cite{SZ-JYA:13} considered multi-armed bandits with Bernoulli rewards and concluded that, among the models considered,  the knowledge gradient algorithm best captures the trial-by-trial performance of human subjects.

Wilson~\etal~\cite{RCW-AG-etal:11} studied human performance in two-armed bandit problems and showed that at each arm selection instance the decision is based on a linear combination of the estimate of the mean reward of each arm and an ambiguity bonus that depends on the value of the information from that arm. 
Tomlin~\etal~\cite{DT-AN-etal:12} studied human performance on multi-armed bandits that are located on a spatial grid; at each arm selection instance, the decision-maker can only select the current arm or one of the neighboring arms. 

In this paper, we study multi-armed bandits with Gaussian rewards in a Bayesian setting, and we develop upper credible limit (UCL)-based algorithms that achieve efficient performance.   We propose a deterministic UCL algorithm and a stochastic UCL algorithm for the standard multi-armed bandit problem.   We propose a block UCL algorithm and a graphical block UCL algorithm for the multi-armed bandit problem with transitions costs and the multi-armed problem on graphs, respectively.  We analyze the proposed algorithms in terms of the cumulative expected regret, i.e., the cumulative difference between the expected received reward and the maximum expected reward that could have been received. We compare human performance in multi-armed bandit tasks with the performance of the proposed stochastic UCL algorithm and show that the algorithm with the right choice of parameters efficiently models human decision-making performance. The major contributions of this work are fourfold.

First, we develop and analyze the deterministic UCL algorithm for multi-armed bandits with Gaussian rewards.
We derive a novel upper bound on the inverse cumulative distribution function for the standard Gaussian distribution, and we use it to show that for an uninformative prior on the rewards, the proposed algorithm achieves logarithmic  regret. To the best of our knowledge, this is the first confidence bound-based algorithm that provably achieves logarithmic cumulative expected regret uniformly in time for multi-armed bandits with Gaussian rewards.

We further define a {\em quality} of priors on rewards and show that for small values of this quality, i.e., good priors, the proposed algorithm  achieves logarithmic   regret uniformly in time.
Furthermore, for good priors with small variance, a slight modification of the algorithm yields sub-logarithmic regret uniformly in time. Sub-logarithmic refers to a rate of expected regret that is even slower than  logarithmic, and thus performance is better than with uninformative priors. For large values of the quality, i.e., bad priors, the proposed algorithm can yield performance significantly worse than with uninformative priors. Our analysis also highlights the impact of the correlation structure among the rewards from different arms on the performance of the algorithm as well as the performance advantage when the prior includes a good model of the correlation structure.

Second, to capture the inherent noise in human decision-making, we develop the stochastic UCL algorithm, a stochastic arm selection version of the deterministic UCL algorithm. We model the stochastic arm selection using softmax arm selection~\cite{RSS-AGB:98}, and show that there exists a feedback law for the cooling rate in the softmax function such that for an uninformative prior the stochastic arm selection policy achieves logarithmic  regret uniformly in time.

Third, we compare the stochastic UCL algorithm with the data obtained from our human behavioral experiments.  We show that the observed empirical behaviors can be reconstructed by varying only a few parameters in the algorithm.

Fourth, we study the multi-armed bandit problem with transition costs in which a stationary random cost is incurred each time an arm other than the current arm is selected.  We also study the graphical  multi-armed bandit problem in which the arms are located at the vertices of a graph and only the current arm and its neighbors  can be selected at each time. For these multi-armed bandit problems, we extend the deterministic UCL algorithm to block allocation algorithms that for uninformative priors achieve logarithmic regret uniformly in time.

In summary, the main contribution of this work is to provide a formal algorithmic model (the UCL algorithms) of choice behavior in  the exploration-exploitation tradeoff using the context of the multi-arm bandit problem. In relation to cognitive dynamics, we expect that this model could be used to explain observed choice behavior and thereby quantify the underlying computational anatomy in terms of key model parameters. The fitting of such models of choice behavior to empirical performance is now standard in cognitive neuroscience. We  illustrate the potential of our model to categorize individuals in terms of a small number of model parameters by showing that the stochastic UCL algorithm can reproduce canonical classes of performance observed in large numbers of  subjects.

The remainder of the paper is organized as follows. The standard multi-armed bandit problem is described in Section II. The salient features of human decision-making in bandit tasks are discussed in Section III. In Section IV we propose and analyze the regret of the deterministic UCL and stochastic UCL algorithms. In Section V we describe an experiment with human participants and a spatially-embedded multi-armed bandit task.  We show that human performance in that task tends to fall into one of several categories, and we demonstrate that the stochastic UCL algorithm can capture these categories with a small number of parameters. We consider an extension of the  multi-armed bandit problem to include transition costs and describe and analyze the block UCL algorithm in Section VI. In Section VII we consider an extension to the graphical multi-armed bandit problem, and we propose and analyze the graphical block UCL algorithm. Finally, in Section VIII we conclude and present avenues for future work.

\section{A review of multi-armed bandit problems}\label{sec:gaussian-bandit}

Consider a set of $N$ options, termed \emph{arms} in analogy with the lever of a slot machine. A single-levered slot machine is termed a \emph{one-armed bandit}, so the case of $N$ options is often called an $N$-armed bandit. The $N$-armed bandit problem refers to the choice among the $N$ options that a decision-making agent should make to maximize the cumulative reward.

The agent collects reward $r_t\in \real$ by choosing arm $i_t$ at each time $t \in \until{T}$, where $T\in \naturals$ is the horizon length for the sequential decision process.  The reward from option  $i \in \until{N}$  is sampled from a stationary distribution $p_i$ and has an unknown mean $m_i \in \real$.     
The decision-maker's objective is to maximize the cumulative expected reward $\sum_{t=1}^T m_{i_t}$ by selecting a sequence of arms $\seqdef{i_t}{t\in\until{T}}$.
Equivalently, defining $m_{i^*} = \max \setdef{m_i}{i\in\until{N}}$ and $R_t = m_{i^*}-m_{i_t}$ as the expected \emph{regret} at time $t$, the objective can be formulated as minimizing the cumulative expected regret defined by
\begin{align*} 
\sum_{t=1}^T R_t = Tm_{i^*} - \sum_{i=1}^N m_i \E{n_{i}^T}= \sum_{i=1}^N \Delta_i \E{n_{i}^T},
\end{align*}
where $n_{i}^T$ is the total number of times option $i$ has been chosen until time $T$ and $\Delta_i = m_{i^*}-m_i$ is the expected regret due to picking arm $i$ instead of arm $i^*$. Note that in order to minimize the cumulative expected regret, it suffices to minimize the expected number of times any suboptimal option $i \in \until{N}\setminus\{ i^*\}$ is selected. 

The multi-armed bandit problem is a canonical example of the exploration-exploitation tradeoff common to many problems in controls and machine learning. In this context, at time $t$, exploitation refers to picking arm $i_t$ that is estimated to have the highest mean at time $t$, and exploration refers to picking any other arm. A successful policy balances the exploration-exploitation tradeoff by exploring enough to learn which arm is most rewarding and exploiting that information by picking the best arm often.

\subsection{Bound on optimal performance}
Lai and Robbins \cite{TLL-HR:85} showed that, for any algorithm solving the multi-armed bandit problem,
the expected number of times a suboptimal arm is selected is at least logarithmic in time, i.e., 
\beq
\E{n_{i}^T} \geq \left(\frac{1}{D(p_i||p_{i^*})} + o(1)\right) \log T,\label{eq:optBound}
\eeq
for each $i\in \until{N}\setminus\{i^*\}$, 
where $o(1) \to 0$ as $T \to +\infty$.   $D(p_i||p_{i^*}) := \int p_i(r) \log \frac{p_i(r)}{p_{i^*}(r)} \mathrm{d}r$ is the Kullback-Leibler divergence between the reward density $p_i$ of any suboptimal arm and the reward density $p_{i^*}$ of the optimal arm. The bound on $\E{n_{i}^T}$ implies that the cumulative expected regret must grow at least logarithmically in time.

\subsection{The Gaussian multi-armed bandit task}

For the Gaussian multi-armed bandit problem considered in this paper, the reward density $p_i$ is Gaussian with mean $m_i$ and variance $\sigma_s^2$. The  variance $\sigma_s^2$ is assumed known, e.g., from previous observations or known characteristics of the reward generation process.
Therefore
\beq
D(p_i||p_{i^*}) = \frac{\Delta_i^2}{2 \sigma_s^2},\label{eq:KLGaussian}
\eeq
and accordingly, the bound \eqref{eq:optBound} is
\beq
\E{n_{i}^T} \geq \left( \frac{2 \sigma_s^2}{\Delta_i^2} + o(1)\right) \log T. \label{eq:optBoundGaussian}
\eeq
The insight from (\ref{eq:optBoundGaussian}) is that for a fixed value of $\sigma_s$, a suboptimal arm $i$ with higher $\Delta_i$ is easier to identify, and thus chosen less often, since it yields a lower average reward. Conversely, for a fixed value of $\Delta_i$, higher values of $\sigma_s$ make the observed rewards more variable, and thus it is more difficult to distinguish the optimal arm $i^*$ from the suboptimal ones.

\subsection{The Upper Confidence Bound algorithms}
For multi-armed bandit problems with bounded rewards, Auer~\etal~\cite{PA-NCB-PF:02} developed upper confidence bound-based algorithms, known as the UCB1 algorithm and its variants, that achieve logarithmic regret uniformly in time. UCB1 is a heuristic-based algorithm that at each time $t$ computes a heuristic value $Q_{i}^t$ for each option $i$.  This value provides an upper bound for the expected reward to be gained by selecting that option:
\beq
Q_{i}^t = \bar{m}_{i}^t + C_{i}^t,\label{eq:UCBHeuristic}
\eeq
where $\bar{m}_{i}^t$ is the empirical mean reward and $C_{i}^t$ is a measure of uncertainty in the reward of arm $i$ at time $t$. The UCB1 algorithm  picks the option $i_t$ that maximizes $Q_{i}^t$. Figure \ref{fig:UCB} depicts this logic: the confidence intervals represent uncertainty in the algorithm's  estimate of the true value of $m_i$ for each option, and the algorithm optimistically chooses the option with the highest upper confidence bound. This is an example of a general heuristic known in the bandit literature as \emph{optimism in the face of uncertainty} \cite{SB-NCB:12}. The idea is that one should formulate the set of possible environments that are consistent with the observed data, then act as if the true environment were the most favorable one in that set.

\begin{figure}[h]
   \centering
   \includegraphics[width=3.5in]{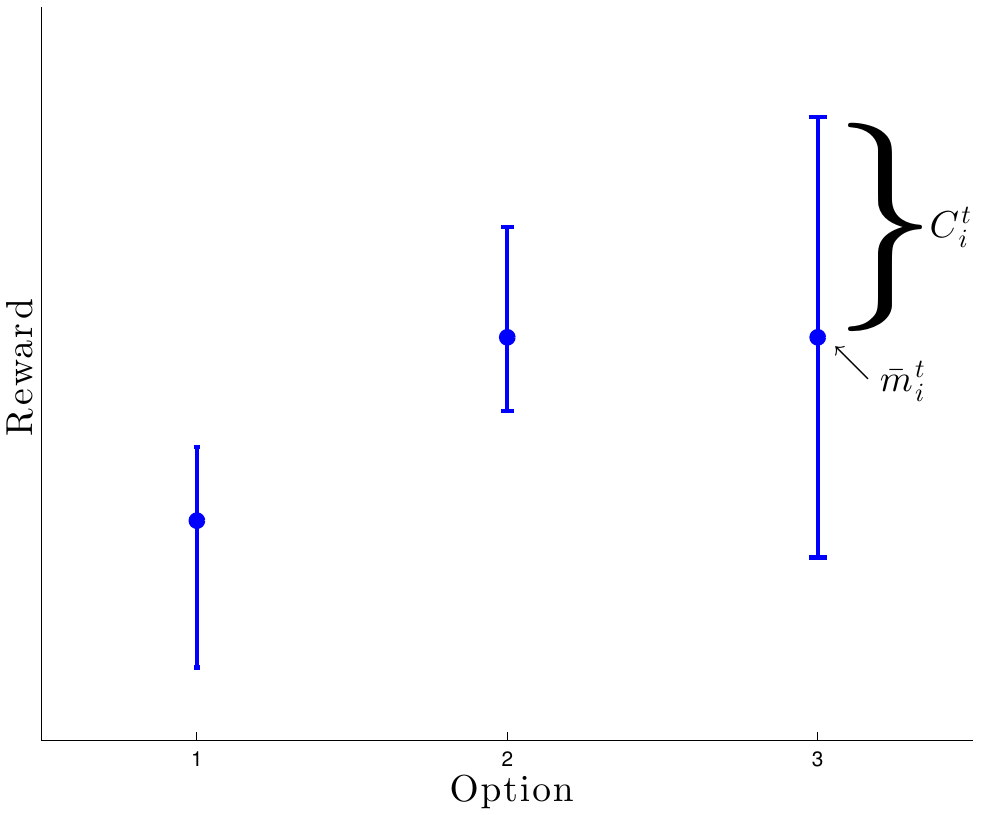}
   \caption{Components of the UCB1 algorithm in an $N=3$ option (arm) case. The algorithm forms a confidence interval for the mean reward $m_i$ for each option $i$ at each time $t$. The heuristic value $Q_i^t = \bar{m}_i^t + C_i^t$ is the upper limit of this confidence interval, representing an optimistic estimate of the true mean reward. In this example, options 2 and 3 have the same mean $\bar{m}$ but option 3 has a larger uncertainty $C$, so the algorithm chooses option 3.}
   \label{fig:UCB}
\end{figure}

Auer~\etal~\cite{PA-NCB-PF:02} showed that for an appropriate choice of the uncertainty term $C_i^t$, the UCB1 algorithm achieves logarithmic  regret uniformly in time, albeit with a larger leading constant than the optimal one \eqref{eq:optBound}. They also provided a slightly more complicated policy, termed UCB2, that brings the factor multiplying the logarithmic term arbitrarily close to that of \eqref{eq:optBound}. Their analysis relies on Chernoff-Hoeffding bounds which apply to probability distributions with bounded support.

They also considered the case of multi-armed bandits with Gaussian rewards, where both the mean ($m_i$ in our notation) and sample variance ($\sigma_s^2$) are unknown. In this case they constructed an algorithm, termed UCB1-Normal, that achieves logarithmic regret. Their analysis of the regret in this case cannot appeal to Chernoff-Hoeffding bounds because the reward distribution has unbounded support.  Instead their analysis relies on certain bounds on the tails of the $\chi^2$ and the Student t-distribution that they could only verify numerically. Our work improves on their result in the case $\sigma_s^2$ is known by constructing a UCB-like algorithm that provably achieves logarithmic  regret. The proof relies on new tight bounds on the tails of the Gaussian distribution that will be stated in Theorem~\ref{thm:tailBounds}.

\subsection{The Bayes-UCB algorithm}
UCB algorithms rely on a frequentist estimator $\bar{m}_{i}^t$ of $m_i$ and therefore must sample each arm at least once in an initialization step, which requires a sufficiently long horizon, i.e., $N < T$. Bayesian estimators allow the integration of prior beliefs into the decision process.   This enables a Bayesian UCB algorithm to treat the case $N > T$ as well as to capture the initial beliefs of an agent, informed perhaps through prior experience. Kauffman~\etal~\cite{EK-OC-AG:12} considered the $N$-armed bandit problem from a Bayesian perspective and proposed the quantile function of the posterior reward distribution as the heuristic function (\ref{eq:UCBHeuristic}).

For every random variable $X \in \real \union \{\pm \infty\}$ with probability distribution function (pdf) $f(x)$, the associated cumulative distribution function (cdf) $F(x)$ gives the probability that the random variable takes a value of at most $x$, i.e., $F(x) = \Pr{X \leq x}$. See Figure \ref{fig:pdf}. Conversely, the \emph{quantile} function $F^{-1}(p)$ is defined by
\[ \map{F^{-1}}{[0,1]}{\real \union \{\pm \infty \}},\]
i.e., $F^{-1}(p)$ inverts the cdf to provide an upper bound for the value of the random variable $X \sim f(x)$:
\beq
\Pr{X \leq F^{-1}(p)} = p.
\eeq
In this sense, $F^{-1}(p)$ is an \emph{upper confidence bound}, i.e., an upper bound that holds with probability, or \emph{confidence level}, $p$. Now suppose that $F(r)$ is the cdf for the reward distribution $p_i(r)$ of option $i$.   Then, $Q_i = F^{-1}(p)$ gives a bound such that $\Pr{m_i>Q_i} = 1-p$.  If $p \in (0,1)$ is chosen large, then $1-p$ is small, and it is unlikely that the true mean reward for option $i$ is higher than the bound. See Figure \ref{fig:cdf}.

\begin{figure}[h]
   \centering
   \includegraphics[width=3.5in]{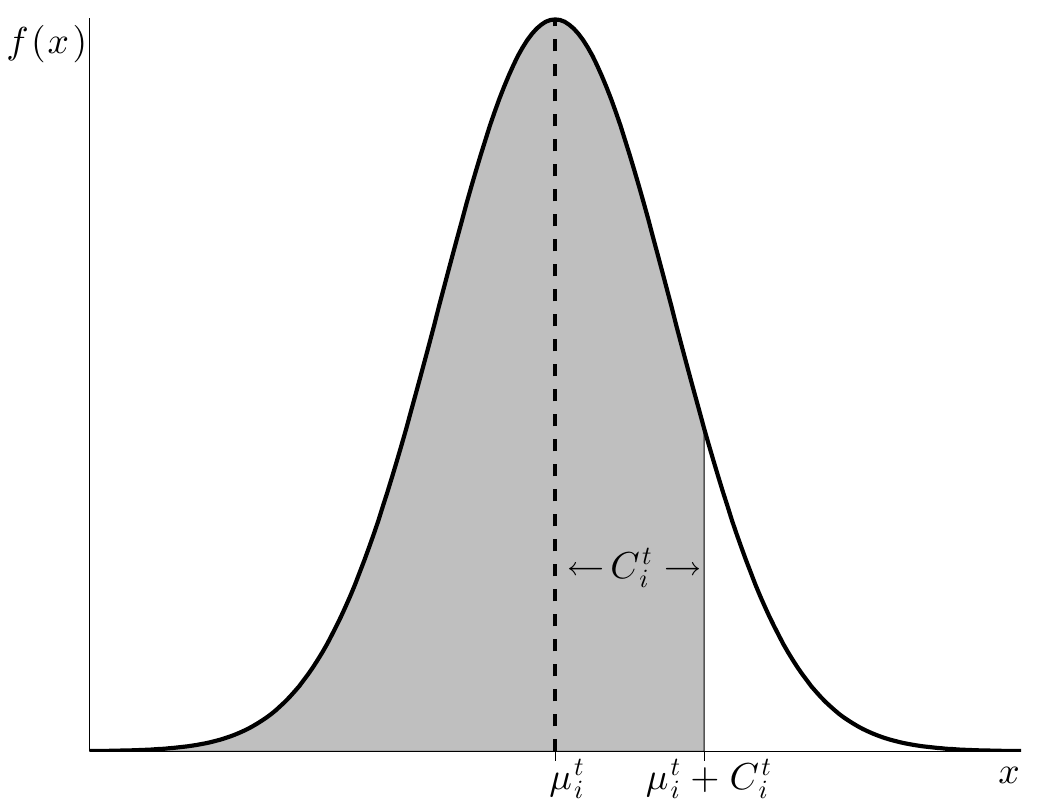}
   \caption{The pdf $f(x)$ of a Gaussian random variable $X$ with mean $\mu_i^t$. The probability that $X \leq x$ is $\int_{-\infty}^x f(X)\, \mathrm{d}X = F(x).$ The area of the shaded region is $F(\mu_i^t + C_i^t) = p$, so the probability that $X \leq \mu_i^t + C_i^t$ is $p$. Conversely, $X \geq \mu_i^t + C_i^t$ with probability $1-p$, so if $p$ is close to 1, $X$ is almost surely less than $\mu_i^t + C_i^t$.}
   \label{fig:pdf}
\end{figure}

\begin{figure}[h]
   \centering
   \includegraphics[width=3.5in]{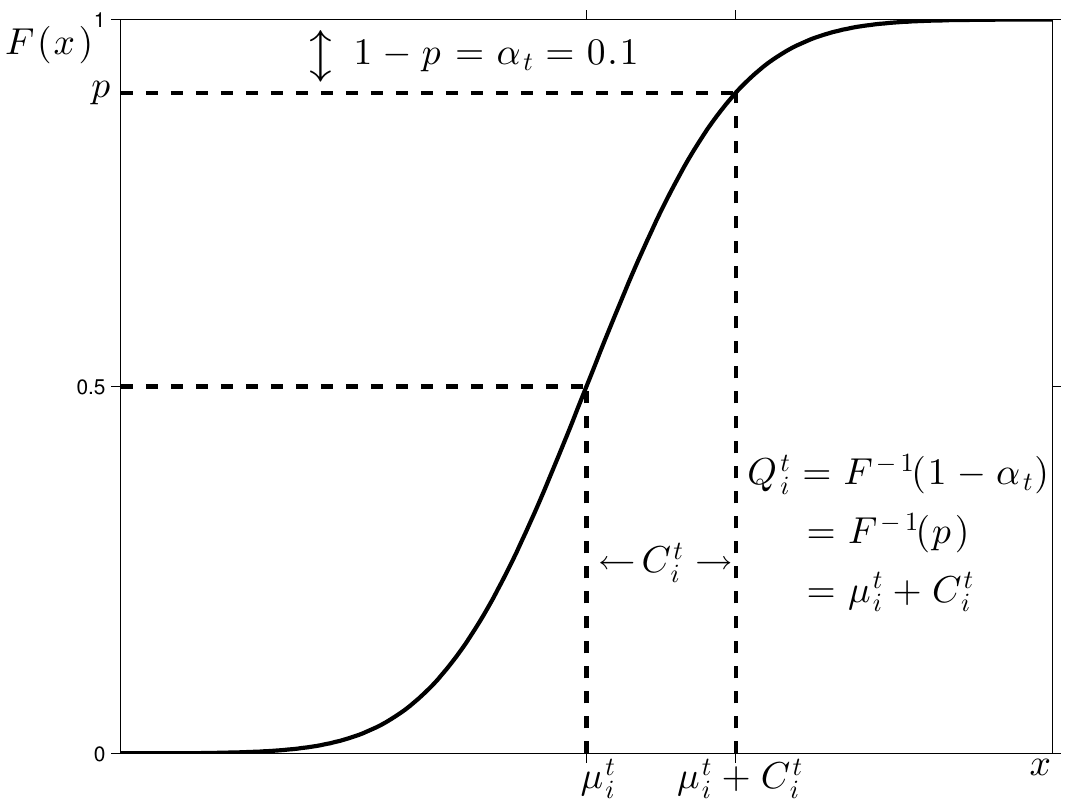}
   \caption{Decomposition of the Gaussian cdf $F(x)$ and relation to the UCB/Bayes-UCB heuristic value. For a given value of $\alpha_t$ (here equal to $0.1$), $F^{-1}(1-\alpha_t)$ gives a value $Q_i^t = \mu_i^t + C_i^t$ such that the Gaussian random variable $X \leq Q_i^t$ with probability $1-\alpha_t$. As $\alpha_t \to 0$, $Q_i^t \to +\infty$ and $X$ is almost surely less than $Q_i^t$.}
   \label{fig:cdf}
\end{figure}

In order to be increasingly sure of choosing the optimal arm as time goes on, \cite{EK-OC-AG:12} sets $p = 1-\alpha_t$ as a function of time with $\alpha_t = 1/(t(\log T)^c)$, so that $1-p$ is of order $1/t$. The authors termed the resulting algorithm Bayes-UCB.  In the case that the rewards are Bernoulli distributed, they proved that with $c\geq 5$ Bayes-UCB achieves the bound (\ref{eq:optBound}) for uniform (uninformative) priors.

The choice of $1/t$ as the functional form for $\alpha_t$ can be motivated as follows. Roughly speaking, $\alpha_t$ is the probability of making an error (i.e., choosing a suboptimal arm) at time $t$. If a suboptimal arm is chosen with probability $1/t$, then the expected number of times it is chosen until time $T$ will follow the integral of this rate, which is $\sum_1^T 1/t \approx \log T$, yielding a logarithmic functional form.

\section{Features of human decision-making in multi-armed bandit tasks}
As discussed in the introduction, human decision-making in the multi-armed bandit task has been the subject of numerous studies in the cognitive psychology literature. We list the five salient features of human decision-making in this literature that we wish to capture with our model.

\medskip

\noindent
(i) {\bf  Familiarity with the environment:}
Familiarity with the environment and its structure plays a critical role in human decision-making~\cite{JDC-SMM-AJY:07,MS-MDL-EJW:09}. In the context of multi-armed bandit tasks, familiarity with the environment translates to prior knowledge about the mean rewards from each arm.

\medskip

\noindent 
(ii) {\bf Ambiguity bonus:}
Wilson~\etal~\cite{RCW-AG-etal:11} showed that the decision at time $t$ is based on a linear combination of the estimate of the mean reward of each arm and an \emph{ambiguity bonus} that captures the value of information from that arm. In the context of UCB and related algorithms, the ambiguity bonus can be interpreted similarly to the $C_{i}^t$ term of \eqref{eq:UCBHeuristic} that defines the size of the upper bound on the estimated reward.

\medskip

\noindent
(iii) {\bf Stochasticity:} Human decision-making is inherently noisy~\cite{JDC-SMM-AJY:07,DA-PS:08,MS-MDL-EJW:09,SZ-JYA:13,RCW-AG-etal:11}. This is possibly due to inherent limitations in human computational capacity, or it could be the signature of noise being used as a cheap, general-purpose problem-solving algorithm. In the context of algorithms for solving the multi-armed bandit problem, this can be interpreted as picking arm $i_t$ at time $t$ using a stochastic arm selection strategy rather than a deterministic one.

\medskip

\noindent
(iv) {\bf Finite-horizon effects:}
Both the level of decision noise and the exploration-exploitation tradeoff are sensitive to the time horizon $T$ of the bandit task~\cite{JDC-SMM-AJY:07, RCW-AG-etal:11}. This is a sensible feature to have, as shorter time horizons mean less time to take advantage of information gained by exploration, therefore biasing the optimal policy towards exploitation. The fact that both decision noise and the exploration-exploitation tradeoff (as represented by the ambiguity bonus) are affected by the time horizon suggests that they are both working as mechanisms for exploration, as investigated in \cite{PR-RCW-PH-NEL:12}. In the context of algorithms, this means that the uncertainty term $C_{i}^t$ and the stochastic arm selection scheme should be functions of the horizon $T$.

\medskip

\noindent
(v) {\bf Environmental structure effects:}
Acu{\~n}a~\etal~\cite{DEA-PS:10} showed that an important aspect of human learning in multi-armed bandit tasks is structural learning, i.e., humans learn the correlation structure among different arms, and utilize it to improve their decision.

In the following, we develop a plausible model for human decision-making that captures these features. 
Feature (i) of human decision-making is captured through priors on the mean rewards from the arms.
The introduction of priors in the decision-making process suggests that non-Bayesian upper confidence bound algorithms~\cite{PA-NCB-PF:02} cannot be used, and therefore, we focus on Bayesian upper confidence bound (upper credible limit) algorithms~\cite{EK-OC-AG:12}.
Feature (ii) of  human decision-making is captured by making decisions based on a metric that comprises two components, namely, the estimate of the mean reward from each arm, and the width of a credible set. It is well known that the width of a credible set is a good measure of the uncertainty in the estimate of the reward.  Feature (iii) of human decision-making is captured by introducing a stochastic arm selection strategy in place of the standard deterministic arm selection 
strategy~\cite{PA-NCB-PF:02,EK-OC-AG:12}. In the spirit of Kauffman~\etal~\cite{EK-OC-AG:12}, we choose the credibility parameter $\alpha_t$ as a function of the horizon length to capture feature (iv) of human decision-making.
Feature (v) is captured through the correlation structure of the prior used for the Bayesian estimation. For example, if the arms of the bandit are spatially embedded, it is natural to think of a covariance structure defined by $\Sigma_{ij}= \sigma_0^2 \exp(-|x_i-x_j|/\lambda)$, where $x_i$ is the location of arm $i$ and $\lambda \ge 0$ is the correlation length scale parameter that encodes the spatial smoothness of the rewards.

\section{The Upper Credible Limit (UCL) Algorithms for Gaussian Multi-armed Bandits}
In this section, we construct a Bayesian UCB algorithm that captures the features of human decision-making described above. We begin with the case of deterministic decision-making and show that for an uninformative prior the resulting algorithm achieves logarithmic  regret. We then extend the algorithm to the case of stochastic decision-making using a Boltzmann (or softmax) decision rule, and show that there exists a feedback rule for the temperature of the Boltzmann distribution such that the stochastic algorithm achieves logarithmic  regret.  In both cases we first consider uncorrelated priors and then extend to correlated priors.

\subsection{The deterministic UCL algorithm with uncorrelated priors}\label{subsec:deterministic-ucl-uncorr}
Let the prior on the mean reward at arm $i$ be a Gaussian random variable with mean $\mu_{i}^0$ and variance $\sigma_0^2$. We are particularly interested in the case of an uninformative prior, i.e., $\sigma_0^2 \to +\infty$.
Let the number of times arm $i$ has been selected until time $t$ be denoted by $n_i^t$. 
Let the empirical mean of the rewards from arm $i$ until time $t$ be $\bar m_i^t$.
Conditioned on the number of visits $n_i^t$ to arm $i$ and the empirical mean $\bar m_i^t$, the mean reward  at arm $i$  at time $t$ is a Gaussian random variable ($M_i$) with mean
and variance
\begin{align*}
\mu_{i}^t&:=\expt[M_i|n_i^t, \bar m_i^t] = \frac{\delta^2 \mu_{i}^0 + n_i^t \bar m_i^t}{\delta^2+n_{i}^t}, \; \text{and}\;\\
\left({\sigma_i^{t}}\right)^2 &:= \text{Var}[M_i|n_i^t, \bar m_i^t] =\frac{\sigma_s^2}{\delta^2 + n_{i}^t} ,
\end{align*}
respectively, where $\delta^2=\sigma_s^2/\sigma_0^2$. Moreover,
\[
\expt[\mu_{i}^t|n_i^t]= \frac{\delta^2 \mu_{i}^0 + n_i^t m_i}{\delta^2+n_{i}^t} \; \text{and}\;
\text{Var}[\mu_{i}^t|n_i^t] = \frac{ n_i^t \sigma_s^2}{(\delta^2+n_{i}^t)^2}.
\]

We now propose the UCL algorithm for the Gaussian multi-armed bandit problem. At each decision instance $t \in \until{T}$, the UCL algorithm selects an arm  with the maximum value of the upper limit of the smallest $(1-1/Kt)$-credible interval, i.e., it selects an arm $i_t=\text{argmax}\setdef{Q_i^t}{i\in\until{N}}$, where
\[
Q_i^t =  \mu_i^t + \sigma_i^t  \Phi^{-1}(1-1/K t),
\]
$\map{\Phi^{-1}}{(0,1)}{\real}$ is the inverse cumulative distribution function for the standard Gaussian random variable, and $K \in \real_{>0}$ is a tunable parameter. For an explicit pseudocode implementation, see Algorithm~\ref{algo:bayes-ucb} in Appendix F. In the following, we will refer to $Q_i^t$ as the $(1 - 1/Kt)$-\emph{upper credible limit} (UCL).

It is known~\cite{AK-CEG:12, NS-AK-SMK-MS:12} that an efficient policy to maximize the total information gained over sequential sampling of options is to pick the option with highest variance at each time. Thus, $Q_i^t$ is the weighted sum of the expected gain in the total reward (exploitation), and the gain in the total information about arms (exploration), if arm $i$ is picked at time $t$.

\subsection{Regret analysis of the deterministic UCL Algorithm}
In this section, we analyze the performance of the UCL algorithm. We first derive bounds on the 
inverse cumulative distribution function for the standard Gaussian random variable and then utilize it to derive upper bounds on the cumulative expected regret for the UCL algorithm. 
We state the following theorem about bounds on the inverse Gaussian cdf. 
\begin{theorem}[\bit{Bounds on the inverse Gaussian cdf}]\label{thm:tailBounds}
The following bounds hold for the inverse cumulative distribution function of the standard Gaussian random variable for each $\alpha\in (0, 1/\sqrt{2 \pi})$, and any $\beta \ge 1.02$:
\begin{align}
\Phi^{-1}(1-\alpha) &< \beta \sqrt{-\log(-(2\pi \alpha^2)\log(2\pi \alpha^2))}, \text{ {and} }\label{eq:QDown}\\
\Phi^{-1}(1-\alpha) &> \sqrt{-\log(2\pi \alpha^2(1-\log(2\pi \alpha^2)))}. \label{eq:QUp}
\end{align}
\end{theorem}
\begin{proof}
See Appendix A. 
\end{proof}

The bounds in equations~\eqref{eq:QDown}~and~\eqref{eq:QUp} were conjectured by Fan~\cite{Fan2012}
without the factor $\beta$. In fact, it can be numerically verified that without the factor $\beta$, the conjectured upper bound is incorrect.  We present a visual depiction of the tightness of the derived bounds in~Figure~\ref{fig:Qbounds}.

\begin{figure}
\begin{center}
\includegraphics[width=0.45\textwidth]{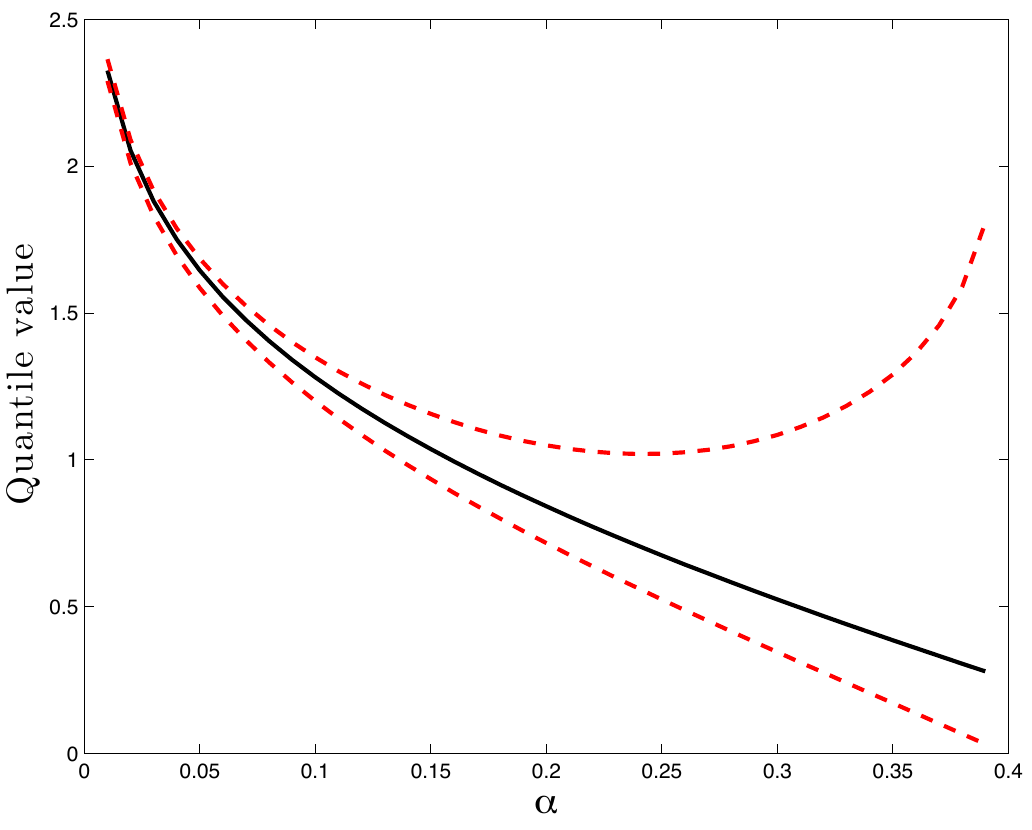}
\caption{Depiction of the normal quantile function $\Phi^{-1}(1-\alpha)$ (solid line) and the bounds (\ref{eq:QDown}) and (\ref{eq:QUp}) (dashed lines), with $\beta=1.02$.}
\label{fig:Qbounds}
\end{center}
\end{figure}

We now analyze the performance of the UCL algorithm. We define $\seqdef{\supscr{R}{UCL}_t}{t\in\until{T}}$ as the sequence of expected regret for the UCL algorithm. The UCL algorithm achieves logarithmic   regret uniformly in time as formalized in the following theorem. 

\begin{theorem}[\bit{Regret of  the deterministic UCL algorithm}]\label{thm:UCL}
The following statements hold for the Gaussian multi-armed bandit problem and the deterministic UCL algorithm with uncorrelated uninformative prior and $K=\sqrt{2 \pi e}$:
\begin{enumerate}
\item the expected number of times a suboptimal arm $i$ is chosen until time $T$ satisfies
\begin{multline*}
\E{n_{i}^T} \leq \Big( \frac{8 \beta^2 \sigma_s^2}{\Delta_i^2} + \frac{2}{\sqrt{2\pi e}} \Big) \log T\\
+ \frac{4 \beta^2 \sigma_s^2}{\Delta_i^2}(1  -\log 2 -\log\log T) 
 + 1 + \frac{2}{\sqrt{2\pi e}}\; ; 
\end{multline*}
\item the cumulative expected regret until time $T$ satisfies
\begin{multline*}
\sum_{t=1}^T \supscr{R_t}{UCL}  \leq 
\sum_{i=1}^N \Delta_i \Bigg(\Big( \frac{8 \beta^2 \sigma_s^2}{\Delta_i^2} + \frac{2}{\sqrt{2\pi e}} \Big) \log T\\
+ \frac{4 \beta^2 \sigma_s^2}{\Delta_i^2}(1  -\log 2 -\log\log T) 
 + 1 + \frac{2}{\sqrt{2\pi e}}\Bigg).
\end{multline*}
\end{enumerate}
\end{theorem}
\begin{proof}
See Appendix B.
\end{proof}

\begin{remark}[\bit{Uninformative priors with short time horizon}]\label{rmk:shortHorizon}\emph{When the deterministic UCL algorithm is used with an uncorrelated uninformative prior, Theorem \ref{thm:UCL} guarantees that the algorithm incurs logarithmic regret uniformly in horizon length $T$. However, for small horizon lengths, the upper bound on the regret can be lower bounded by a super-logarithmic curve. Accordingly, in practice, the cumulative expected regret curve may appear super-logarithmic for short time horizons.  For example, for horizon $T$ less than the number of arms $N$, the cumulative expected regret of the deterministic UCL algorithm grows at most linearly with the horizon length.
 \oprocend
}
\end{remark}

\begin{remark}[\bit{Comparison with UCB1}]\emph{
In view of the bounds in Theorem~\ref{thm:tailBounds}, for an uninformative prior, the $(1-1/Kt)$-upper credible limit  obeys
\[ Q_i^t < \bar{m}_i^t + \beta \sigma_s \sqrt{\frac{1 + 2 \log t - \log \log e t^2}{n_i^t}}.\]
This upper bound is similar to the one in UCB1, which sets
\[ Q_i^t = \bar{m}_i^t + \sqrt{\frac{2 \log t}{n_i^t}}. \tag*{\oprocend}\]
}
\end{remark}

\begin{remark}[\bit{Informative priors}]\label{rmk:priorQuality}\emph{
For an uninformative prior, i.e., very large variance $\sigma_0^2$, we established in Theorem~\ref{thm:UCL} that the deterministic UCL algorithm achieves logarithmic regret uniformly in time.  For informative priors, the cumulative expected regret depends on the quality of the prior. The quality of a prior on the rewards can be captured by the metric $\zeta:= \max\setdef{ |m_i - \mu_i^0|/\sigma_0}{i\in \until{N}}$. A \emph{good prior} corresponds to small values of $\zeta$, while a \emph{bad prior} corresponds to large values of $\zeta$. In other words, a good prior is  one that has (i) mean close to the true mean reward, or (ii) a large variance. Intuitively, a good prior either has a fairly accurate estimate of the mean reward, or has low confidence about its estimate of the mean reward. For a good prior, the parameter $K$ can be tuned such that
\begin{equation*}
\Phi^{-1}\Big( 1- \frac{1}{Kt} \Big) - \max_{i\in \until{N}}\frac{\sigma_s(|m_i-\mu_i^0|)}{\sigma_0^2} >
\Phi^{-1}\Big(1- \frac{1}{\bar K t}\Big),
\end{equation*}
where $\bar K \in \real_{>0}$ is some constant, and it can be shown, using the arguments of Theorem~\ref{thm:UCL}, that the deterministic UCL algorithm achieves logarithmic regret uniformly in time.   
A bad prior corresponds to a fairly inaccurate estimate of the mean reward and high confidence.
For a bad prior, the  cumulative expected regret may be a super-logarithmic function of the horizon length. \oprocend
}
\end{remark}

\begin{remark}[\bit{Sub-logarithmic regret for good priors}]\emph{
For a good prior with a small variance, even  uniform sub-logarithmic regret can be achieved. Specifically, if the variable $Q_i^t$ in Algorithm~\ref{algo:bayes-ucb} is set to $Q_i^t = m_i^t + \sigma_i^t \Phi^{-1}(1- 1/K t^2)$, then an analysis similar to Theorem~\ref{thm:UCL} yields an upper bound on the cumulative expected regret that is dominated by (i) a sub-logarithmic term for good priors with small variance, and (ii) a logarithmic term  for uninformative priors with a higher constant in front than the constant in Theorem~\ref{thm:UCL}. Notice that such good priors may correspond to human operators who have previous training in the task.} \oprocend
\end{remark}

\subsection{The stochastic UCL algorithm with uncorrelated priors}
To capture the inherent stochastic nature of human decision-making, we consider the UCL algorithm with stochastic arm selection.
Stochasticity has been used as a generic optimization mechanism  that does not require information about the objective function. For example, simulated annealing \cite{DB-JNT:93,DM-FR-ASV:86,SK-CDG-MPV:83} is a global optimization method that attempts to break out of local optima by sampling locations near the currently selected optimum and  accepting locations with worse objective values with a probability that decreases in time. By analogy with physical annealing processes, the probabilities are chosen from a Boltzmann distribution with a dynamic temperature parameter that decreases in time, gradually making the optimization more deterministic. An important problem in the design of simulated annealing algorithms is the choice of the temperature parameter, also known as a \emph{cooling schedule}.

Choosing a good cooling schedule is equivalent to solving the explore-exploit problem in the context of simulated annealing, since the temperature parameter balances exploration and exploitation by tuning the amount of stochasticity (exploration) in the algorithm. In their classic work, Mitra~\etal ~\cite{DM-FR-ASV:86} found cooling schedules that maximize the rate of convergence of simulated annealing to the global optimum. In a similar way, the stochastic UCL algorithm (see Algorithm \ref{algo:softmax-ucl} in Appendix F for an explicit pseudocode implementation) extends the deterministic UCL algorithm (Algorithm~\ref{algo:bayes-ucb}) to the stochastic case. The stochastic UCL algorithm chooses an arm at time $t$ using a Boltzmann distribution with temperature $\upsilon_{t}$, so the probability $P_{it}$ of picking arm $i$ at time $t$ is given by
\[P_{it} = \frac{\exp(Q_{i}^t/\upsilon_t)}{\sum_{j=1}^N \exp(Q_{j}^t/\upsilon_t)}.\]
In the case $\upsilon_t \to 0^+$ this scheme chooses $i_t = \argmax\setdef{ Q_{i}^t}{i\in\until{N}}$ and as $\upsilon_t$ increases the probability of selecting any other arm increases. Thus Boltzmann selection generalizes the maximum operation and is sometimes known as the soft maximum (or softmax) rule.

The temperature parameter might be chosen constant, i.e., $\upsilon_t = \upsilon$. In this case the performance of the stochastic UCL algorithm 
can be made arbitrarily close to that of the deterministic UCL algorithm 
by taking the limit $\upsilon \to 0^+$. However, \cite{DM-FR-ASV:86} showed that good cooling schedules for simulated annealing take the form
\[ \upsilon_t = \frac{\nu}{\log t},\]
so we investigate cooling schedules of this form.   We choose $\nu$ using a feedback rule on the values of the heuristic function $Q_i^t, i\in \until{N}$ and define the cooling schedule as
\[ \upsilon_t = \frac{\Delta Q_{\min}^t}{2 \log t},\]
where
$\Delta Q_{\min}^t = \min \setdef{|Q_{i}^t - Q_{j}^t|}{i, j\in \until{N}, i\ne j}$ is the minimum gap between the heuristic function value for any two pairs of arms. We  define $\infty - \infty = 0$, so that $\Delta Q_{\min}^t = 0$ if two arms have infinite heuristic values, and define $0/0 = 1$.

\subsection{Regret analysis of the stochastic UCL algorithm}
In this section we show that for an uninformative prior,  the stochastic UCL algorithm achieves efficient performance. We define $\seqdef{\supscr{R}{SUCL}_t}{t\in\until{T}}$ as the sequence of expected regret for the stochastic UCL algorithm. The stochastic UCL algorithm achieves logarithmic   regret uniformly in time as formalized in the following theorem.

\begin{theorem}[\bit{Regret of the stochastic UCL algorithm}]\label{thm:softmax-UCL}
The following statements hold for the Gaussian multi-armed bandit problem and the stochastic UCL algorithm with uncorrelated uninformative prior and $K=\sqrt{2 \pi e}$:
\begin{enumerate}
\item the expected number of times a suboptimal arm $i$ is chosen until time $T$ satisfies
\begin{multline*}
\E{n_{i}^T} \leq \Big( \frac{8 \beta^2 \sigma_s^2}{\Delta_i^2} + \frac{2}{\sqrt{2\pi e}} \Big) \log T + \frac{\pi^2}{6}\\
+ \frac{4 \beta^2 \sigma_s^2}{\Delta_i^2}(1  -\log 2 -\log\log T) 
 + 1 + \frac{2}{\sqrt{2\pi e}}\; ; 
\end{multline*}
\item the cumulative expected regret until time $T$ satisfies
\begin{multline*}
\sum_{t=1}^T \supscr{R_t}{SUCL}  \leq 
\sum_{i=1}^N \Delta_i \Bigg(\Big( \frac{8 \beta^2 \sigma_s^2}{\Delta_i^2} + \frac{2}{\sqrt{2\pi e}} \Big) \log T  + \frac{\pi^2}{6}\\
+ \frac{4 \beta^2 \sigma_s^2}{\Delta_i^2}(1  -\log 2 -\log\log T) 
 + 1 + \frac{2}{\sqrt{2\pi e}}\Bigg).
\end{multline*}
\end{enumerate}
\end{theorem}
\begin{proof}
See Appendix C.
\end{proof}

\subsection{The UCL algorithms with correlated priors}
In the preceding sections, we consider the case of  uncorrelated priors, i.e., the case with diagonal covariance matrix of the prior distribution for mean rewards  $\Sigma_0 = \sigma_0^2 I_N$. However, in many cases there may be dependence among the arms that we wish to encode in the form of a non-diagonal covariance matrix. In fact, one of the main advantages a human may have in performing a bandit task is their prior experience with the dependency structure across the arms resulting in a good prior correlation structure.  We show that including covariance information can improve performance and may, in some cases, lead to sub-logarithmic regret.

Let $\mc N(\bs\mu_0, \Sigma_0)$ and  $\mc N(\bs\mu_0, \Sigma_{0d})$ be correlated and uncorrelated priors on the mean rewards from the arms, respectively, where $\bs \mu_0 \in \real^N$ is the vector of prior estimates of the mean rewards from each arm, $\Sigma_0\in \real^{N\times N}$ is a positive definite matrix, and $\Sigma_{0d}$ is the same matrix with all its non-diagonal elements set equal to $0$.
 The inference procedure described in Section~\ref{subsec:deterministic-ucl-uncorr}  generalizes to a correlated prior as follows: Define $\seqdef{\boldsymbol{\phi}_t \in \mathbb{R}^{N}}{t\in \until{T}}$ to be the indicator vector corresponding to the currently chosen arm $i_t$, where $(\boldsymbol \phi_t)_k = 1$ if $k = i_t$, and zero otherwise. Then the belief state $(\boldsymbol \mu_t,\Sigma_t)$ updates as follows \cite{SMK:93}:
\begin{align}\label{eq:general-inference}
\begin{split}
\mathbf{q} &= \frac{r_t \boldsymbol \phi_t}{\sigma_s^2} + \Lambda_{t-1} \boldsymbol \mu_{t-1} \\
\Lambda_{t} &= \frac{\boldsymbol \phi_t \boldsymbol \phi_t^T}{\sigma_s^2} + \Lambda_{t-1}, \ \Sigma_t = \Lambda_t^{-1}\\
\boldsymbol \mu_t &= \Sigma_t \mathbf{q},
\end{split}
\end{align}
where $\Lambda_t = \Sigma_t^{-1}$ is the \emph{precision} matrix.

The upper credible limit for each arm $i$ can be computed based on the univariate Gaussian marginal distribution of the posterior with mean $\mu_{i}^{t}$ and variance $\left(\sigma_{i}^{t}\right)^2 = (\Sigma_t)_{ii}$. Consider the evolution of the belief state with the diagonal (uncorrelated) prior $\Sigma_{0d}$ and compare it with the belief state based on the non-diagonal $\Sigma_{0}$ which encodes information about the correlation structure of the rewards in the off-diagonal terms. The additional information means that the inference procedure will converge more quickly than in the uncorrelated case, as seen in Theorem \ref{thm:corrPriors}. If the assumed correlation structure correctly models the environment, then the inference will converge towards the correct values, and the performance of the UCL and stochastic UCL algorithms will be at least as good as that guaranteed by the preceding analyses in Theorems~\ref{thm:UCL} and \ref{thm:softmax-UCL}.

Denoting ${\sigma_{i}^t}^2 = (\Sigma_t)_{ii}$ as the posterior at time $t$ based on $\Sigma_0$ and ${\sigma_{id}^t}^2 = (\Sigma_{td})_{ii}$ as the posterior based on $\Sigma_{0d}$, for a given sequence of chosen arms $\seqdef{i_{\tau}}{\tau\in \until{T}}$, we have that the variance of the non-diagonal estimator will be no larger than that of the diagonal one, as summarized in the following theorem:

\begin{theorem}[\bit{Correlated versus uncorrelated priors}]\label{thm:corrPriors}
For the inference procedure in~\eqref{eq:general-inference}, and 
any given sequence of selected arms $\seqdef{i_{\tau}}{\tau\in \until{T}}$, ${\sigma_{i}^t}^2 \leq {\sigma_{id}^t}^2$, for any $t\in \{0, \ldots, T\}$, and for each $i\in \until{N}$. 
\end{theorem}

\begin{proof}
We use induction. By construction, ${\sigma_{i}^0}^2 = {\sigma_{id}^0}^2$, so the statement is true for $t = 0$. Suppose  the statement holds for some $t \geq 0$ and consider the update rule for $\Sigma_t$. From the Sherman-Morrison formula for a rank-$1$ update \cite{JS-WM:50}, we have

\[ (\Sigma_{t+1})_{jk} = (\Sigma_t)_{jk} - \left( \frac{\Sigma_t \boldsymbol \phi_t \boldsymbol \phi_t' \Sigma_t}{\sigma_s^2 + \boldsymbol \phi_t' \Sigma_t \boldsymbol \phi_t} \right)_{jk}.\]

We now examine the update term in detail, starting with its denominator:
\[ \boldsymbol \phi_t' \Sigma_t \boldsymbol \phi_t = (\Sigma_t)_{i_ti_t}, \]
so $\sigma_s^2 + \boldsymbol \phi_t' \Sigma_t \boldsymbol \phi_t = \sigma_s^2 + (\Sigma_t)_{i_ti_t}>0$. The numerator is the outer product of the $i_t$-th column of $\Sigma_t$ with itself, and can be expressed in index form as
\[ (\Sigma_t \boldsymbol \phi_t \boldsymbol \phi_t' \Sigma_t)_{jk} = (\Sigma_t)_{ji_t}(\Sigma_t)_{i_tk}.\]
Note that if $\Sigma_t$ is diagonal, then so is $\Sigma_{t+1}$ since the only non-zero update element will be $(\Sigma_t)_{i_ti_t}^2$. Therefore, $\Sigma_{td}$ is diagonal for all $t \geq 0$.

The update of the diagonal terms of $\Sigma$ only uses the diagonal elements of the update term, so
\[ {\sigma_{i}^{(t+1)}}^2 \!\! =  \! (\Sigma_{t+1})_{ii} \! = \! (\Sigma_t)_{ii} - \frac{1}{\sigma_s^2 + \boldsymbol \phi_t' \Sigma_t \boldsymbol \phi_t} \sum_j (\Sigma_t)_{ji_t}(\Sigma_t)_{i_tj}.\]
In the case of $\Sigma_{td}$, the sum over $j$ only includes the $j = i_t$ element whereas with the non-diagonal prior $\Sigma_t$ the sum may include many additional terms. So we have
\begin{align*}
{\sigma_{i}^{(t+1)}}^2 \!\! = \! (\Sigma_{t+1})_{ii} & \!= \! (\Sigma_t)_{ii} - \frac{1}{\sigma_s^2 + \boldsymbol \phi_t' \Sigma_t \boldsymbol \phi_t}  \! \sum_j (\Sigma_t)_{ji_t}(\Sigma_t)_{i_tj}\\
&\leq (\Sigma_{td})_{ii} - \frac{1}{\sigma_s^2 + \boldsymbol \phi_t' \Sigma_{td} \boldsymbol \phi_t}  (\Sigma_{td})_{i_ti_t}^2\\
&= {\sigma_{id}^{(t+1)}}^2,
\end{align*}
and the statement holds for $t+1$.
\end{proof}

Note that the above result merely shows that the belief state converges more quickly in the case of a correlated prior, without making any claim about the correctness of this convergence. For example, consider a case where the prior belief is that two arms are perfectly correlated, i.e., the relevant block of the prior is a multiple of $\displaystyle \left(\begin{smallmatrix} 1 & 1\\ 1 & 1 \end{smallmatrix}\right)$, but in actuality the two arms have very different mean rewards. If the algorithm first samples the arm with lower reward, it will tend to underestimate the reward to the second arm. However, in the case of a well-chosen prior the faster convergence will allow the algorithm to more quickly disregard related sets of arms with low rewards.

\section{Classification of human performance in multi-armed bandit tasks}\label{sec:humanPerformance}

In this section, we study human data from a multi-armed bandit task and show how human performance can be classified as falling into one of several categories, which we term \emph{phenotypes}. We then show that the stochastic UCL algorithm can produce performance that is analogous to the observed human performance.

\subsection{Human behavioral experiment in a  multi-armed bandit task}
In order to study human performance in multi-armed bandit tasks, we ran a spatially-embedded multi-armed bandit task through web servers at Princeton University. Human participants were recruited using Amazon's Mechanical Turk (AMT) web-based task platform \cite{buhrmester2011amazon}. Upon selecting the task on the AMT website, participants were directed to follow a link to a Princeton University website, where informed consent was obtained according to protocols approved by the Princeton University Institutional Review Board.

After informed consent was obtained, participants were shown instructions that told them they would be playing a simple game during which they could collect points, and that their goal was to collect the maximum number of total points in each part of the game.

Each participant was presented with a set of $N = 100$ options in a $10 \times 10$ grid. At each decision time $t \in \until{T}$, the participant made a choice by moving the cursor to one element of the grid and clicking.   After each choice was made a numerical reward associated to that choice was reported on the screen.   The time allowed for each choice was manipulated and allowed to take one of two values, denoted fast and slow.  If the participant did not make a choice within 1.5 (fast) or 6 (slow) seconds after the prompt, then the last choice was automatically selected again.   The reward was visible until the next decision was made and the new reward reported. The time allotted for the next decision began immediately upon the reporting of the new reward.    Figure~\ref{fig:exp} shows the screen used in the experiment.

\begin{figure}
\centering

\includegraphics[width=0.45\textwidth]{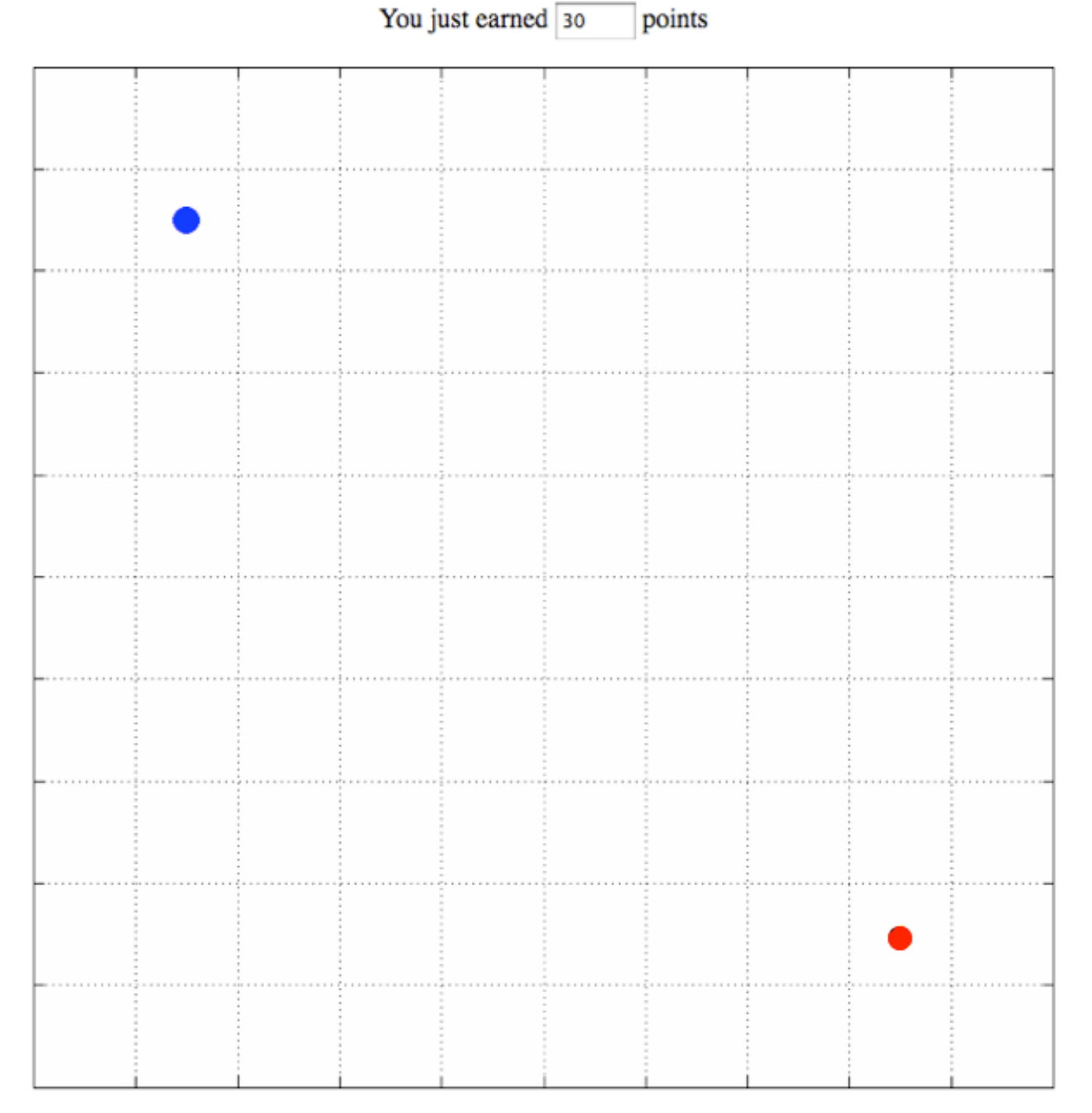}

\caption{The screen used in the experimental interface. Each square in the grid corresponded to an available option. The text box above the grid displayed the most recently received reward, the blue dot indicated the participant's most recently recorded choice, and the smaller red dot indicated the participant's next choice. In the experiment, the red dot was colored yellow, but here we have changed the color for legibility. When both dots were located in the same square, the red dot was superimposed over the blue dot such that both were visible. Initially, the text box was blank and the two dots were together in a randomly chosen square. Participants indicated a choice by clicking in a square, at which point the red dot would move to the chosen option. Until the time allotted for a given decision had elapsed, participants could change their decision without penalty by clicking on another square, and the red dot would move accordingly. When the decision time had elapsed, the blue dot would move to the new square, the text box above the grid would be updated with the most recent reward amount, and the choice would be recorded.}

\label{fig:exp}
\end{figure}

The dynamics of the game were also experimentally manipulated, although we focus exclusively here on the first dynamic condition. The first dynamic condition was a standard bandit task, where the participant could choose any option at each decision time, and the game would immediately sample that option. In the second and third dynamic conditions, the participant was restricted in choices and the game responded in different ways.   These two conditions are beyond the scope of this paper.

Participants were first trained with three training blocks of $T = 10$ choices each, one for each form of the game dynamics.  Subsequently, the participants performed two task blocks of $T = 90$ choices each in a balanced experimental design.  For each participant, the first task had parameters randomly chosen from one of the 12 possible combinations (2 timing, 3 dynamics, 2 landscapes), and the second task was conditioned on the first so that the alternative timing was used with the alternative landscape and the dynamics chosen randomly from the two remaining alternatives.   In particular, only approximately 2/3 of the participants were assigned  a standard bandit task, while other subjects were assigned other dynamic conditions.   The horizon $T<N$ was chosen so that prior beliefs would be important to performing the task.    Each training block took 15 seconds and each task block took 135 (fast) or 540 (slow) seconds.   The time between blocks was negligible, due only to network latency.

Mean rewards in the task blocks corresponded to one of two landscapes:  Landscape A (Figure \ref{fig:surfaceProfiles}(a)) and Landscape B (Figure \ref{fig:surfaceProfiles}(b)). Each landscape was flat along one dimension and followed a profile along the other dimension.  In the two task blocks, each participant saw each landscape once, presented in random order. Both landscapes had a mean value of 30 points and a maximum of approximately 60 points, and the rewards $r_t$ for choosing an option $i_t$ were computed as the sum of the mean reward $m_{i_t}$ and an integer chosen uniformly from the range $[-5,5]$.   In the training blocks, the landscape had a mean value of zero everywhere except for a single peak of 100 points in the center. The participants were given no specific information about the value or the structure of the reward landscapes.

To incentivize the participants to make choices to maximize their cumulative reward, the participants were told that they were being paid based on the total reward they collected during the tasks. As noted above, due to the multiple manipulations, not every participant performed a standard bandit task block. Data were collected from a total of 417 participants:   326 of these participants performed one standard bandit task block each, and the remaining 91 participants performed no standard bandit task blocks.

\begin{figure}[h]
   \centering
   \begin{tabular}{c}
   	{\includegraphics[width=3.5in]{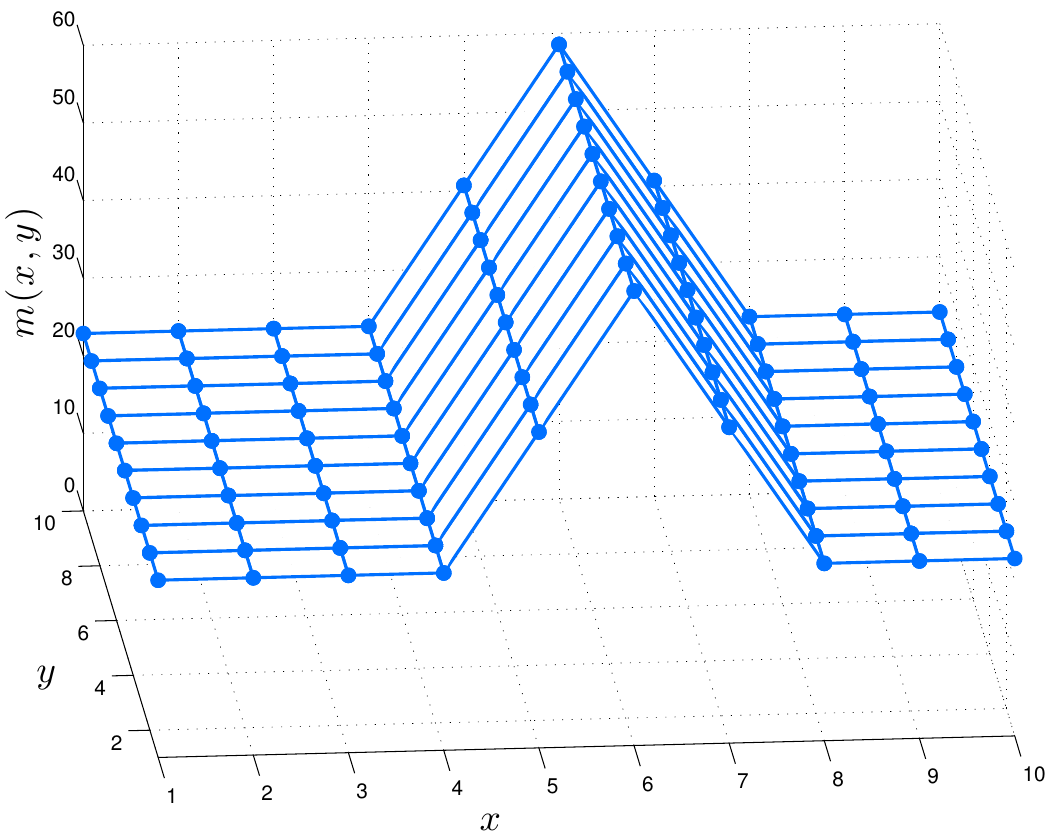}}\\
	(a)\\
   	{\includegraphics[width=3.5in]{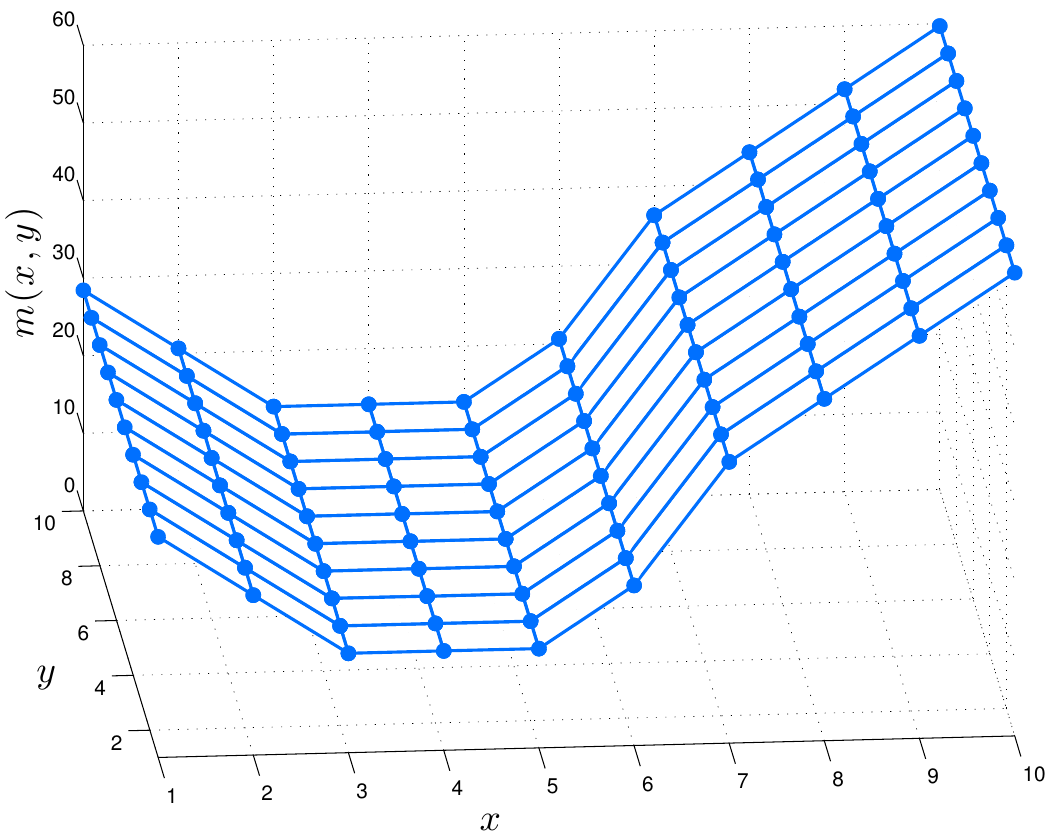}}\\
	(b)
  \end{tabular}
   \caption{The two task reward landscapes: (a) Landscape A, (b) Landscape B. The two-dimensional reward surfaces followed the profile along one dimension (here the $x$ direction) and were flat along the other (here the $y$ direction). The Landscape A profile is designed to be simple in the sense that the surface is concave and there is only one global maximum ($x=6$), while the Landscape B profile is more complicated since it features two local maxima ($x = 1$ and 10), only one of which ($x=10$) is the global maximum.}
   \label{fig:surfaceProfiles}
\end{figure}

\subsection{Phenotypes of observed performance}
For each 90 choice standard bandit task block, we computed observed regret by subtracting the maximum mean cumulative reward from the participant's cumulative reward, i.e.,
\[\mathcal{R}(t) = m_{i^*}t - \sum_{\tau = 1}^t r_{\tau}. \]
The definition of $\mathcal{R}(t)$ uses received rather than expected reward, so it is not identical to cumulative expected regret. However, due to the large number of individual rewards received and the small variance in rewards, the difference between the two quantities is small.

We study human performance by considering the functional form of $\mathcal{R}(t)$. Optimal performance in terms of regret corresponds to $\mathcal{R}(t) = \mathcal{C} \log t$, where $\mathcal{C}$ is the sum over $i$ of the factors in \eqref{eq:optBound}.   The worst-case performance, corresponding to repeatedly choosing the lowest-value option, corresponds to the form $\mathcal{R}(t) = \mathcal{K}t$, where $\mathcal{K} > 0$ is a constant. Other bounds in the bandit literature (e.g. \cite{NS-AK-SMK-MS:12}) are known to have the form $\mathcal{R}(t) = \mathcal{K} \sqrt{t}$.

To classify types of observed human performance in bandit tasks, we fit models representing these three forms to the observed regret from each task. Specifically, we fit the three models
\begin{align}
\mathcal{R}(t) &= a + bt \label{eq:linModel}\\
\mathcal{R}(t) &= at^b \label{eq:powModel}\\
\mathcal{R}(t) &= a + b\log(t) \label{eq:logModel}
\end{align}
to the data from each task and classified the behavior according to which of the models \eqref{eq:linModel}--\eqref{eq:logModel} best fit the data in terms of squared residuals. 
Model selection using this procedure is tenable given that the complexity or number of degrees of freedom of the three models is the same.

Of the 326 participants who performed a standard bandit task block, 59.2\% were classified as exhibiting linear regret (model \eqref{eq:linModel}), 19.3\% power regret \eqref{eq:powModel}, and 21.5\% logarithmic regret \eqref{eq:logModel}.     This suggests that 40.8\% of the participants performed well overall and 21.5\% performed very well.  We observed no significant correlation between performance and timing, landscape, or order (first or second) of playing the standard bandit task block.

Averaging across all tasks, mean performance was best fit by a power model with exponent $b \approx 0.9$, so participants on average achieved sub-linear regret, i.e., better than linear regret.   The nontrivial number of positive performances are noteworthy given that $T < N$, i.e., a relatively short time horizon which makes the task challenging.

Averaging, conditional on the best-fit model, separates the performance of the participants into the three categories of regret performance as can be observed in Figure \ref{fig:phenotypesErrBars}. The difference between linear and power-law performance is not statistically significant until near the task horizon at $t=90$, but log-law performance is statistically different from the other two, as seen using the confidence intervals in the figure. We therefore interpret the linear and power-law performance phenotypes as representing participants with low performance and the log-law phenotype as representing participants with high performance. Interestingly, the three models are indistinguishable for time less than sufficiently small $t \lesssim 30$. This may represent a fundamental limit to performance that depends on the complexity of the reward surface: if the surface is smooth, skilled participants can quickly find good options, corresponding to a small value of the constant $\mathcal{K}$, and thus their performance will quickly be distinguished  from  less skilled participants.  However,  if the surface is rough, identifying good options is harder and will therefore require more samples, i.e., a large value of $\mathcal{K}$, even for  skilled participants.

\begin{figure}[h]
   \centering
   \includegraphics[width=3.5in]{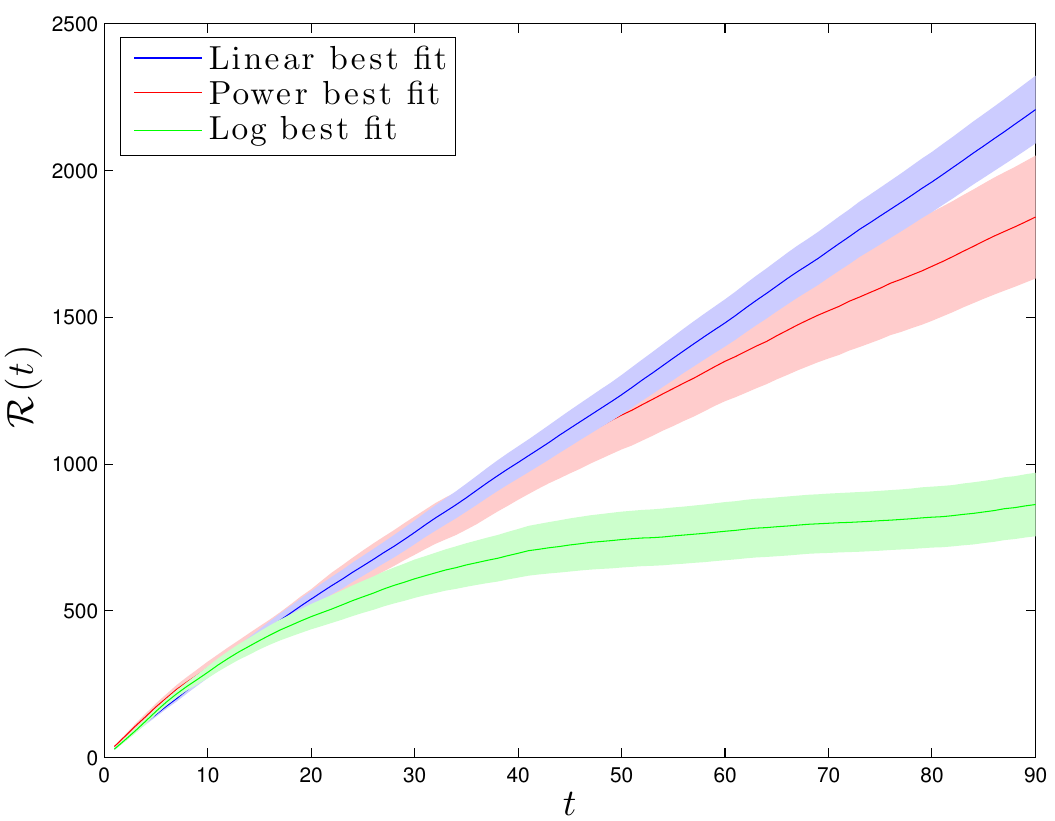}
   \caption{Mean observed regret $\mathcal{R}(t)$ conditional on the best-fit model \eqref{eq:linModel}--\eqref{eq:logModel}, along with bands representing 95\% confidence intervals. Note how the difference between linear and power-law regret is not statistically significant until near the task horizon $T=90$, while logarithmic regret is significantly less than that of the linear and power-law cases.}
   \label{fig:phenotypesErrBars}
\end{figure}

\subsection{Comparison with UCL}
Having identified the three phenotypes of observed human performance in the above section, we show that the stochastic UCL algorithm (Algorithm \ref{algo:softmax-ucl}) can produce behavior corresponding to the linear-law and log-law phenotypes by varying a minimal number of parameters.
Parameters are used to encode the prior beliefs  and the decision noise of the participant. A minimal set of parameters is given by the four scalars $\mu_0, \sigma_0, \lambda$ and $\upsilon$, defined as follows.

\medskip

\noindent
(i) {\bf  Prior mean}
The model assumes prior beliefs about the mean rewards to be a Gaussian distribution with mean $\bs \mu_0$ and covariance $\Sigma_0$. It is reasonable to assume that participants set $\bs \mu_0$ to the uniform prior $\bs \mu_0 = \mu_0 \mathbf{1}_N$, where $\mathbf{1}_N \in \real^N$ is the vector with every entry equal to 1. Thus, $\mu_0 \in \real$ is a single parameter that encodes the participants' beliefs about the mean value of rewards.

\medskip

\noindent
(ii,iii) {\bf Prior covariance}
For a spatially-embedded task, it is reasonable to assume that arms that are spatially close will have similar mean rewards. Following \cite{NEL-DAP-FL-RS-DMF-RED:2007} we choose the elements of $\Sigma_0$ to have the form

\begin{equation}
\Sigma_{ij}= \sigma_0^2 \exp(-|x_i-x_j|/\lambda),
\label{eq:spatialCovPrior}
\end{equation}
where $x_i$ is the location of arm $i$ and $\lambda \ge 0$ is the correlation length scale parameter that encodes the spatial smoothness of the reward surface. The case $\lambda = 0$ represents complete independence of rewards, i.e., a very rough surface, while as $\lambda$ increases the agent believes the surface to be smoother. The parameter $\sigma_0 \geq 0$ can be interpreted as a confidence parameter, with $\sigma_0 = 0$ representing absolute confidence in the beliefs about the mean $\bs \mu_0$, and $\sigma_0 = +\infty$ representing complete lack of confidence.

\medskip

\noindent
(iv) {\bf Decision noise}
In Theorem \ref{thm:softmax-UCL} we show that for an appropriately chosen cooling schedule, the stochastic UCL algorithm with softmax action selection achieves logarithmic regret. However, the assumption that human participants employ this particular cooling schedule is unreasonably strong.   It is of great interest in future experimental work to investigate what kind of cooling schedule best models human behavior.  The Bayes-optimal cooling schedule can be computed using variational Bayes methods~\cite{KF-etal:13}; however, for simplicity, we model the participants' decision noise by using softmax action selection with a constant temperature $\upsilon \geq 0$. This yields a single parameter representing the stochasticity of the decision-making:  in the limit $\upsilon \to 0^+$, the model reduces to the deterministic UCL algorithm, while with increasing $\upsilon$ the decision-making is increasingly stochastic.

With this set of parameters, the prior quality $\zeta$ from Remark~\ref{rmk:priorQuality} reduces to $\zeta = (\max_i |m_i-\mu_0|)/\sigma_0$. Uninformative priors correspond to very large values of $\sigma_0$.  Good priors, corresponding to small  values of $\zeta$, have $\mu_0$ close to $m_{i^*} = \max_i m_i$ or little confidence in the value of $\mu_0$, represented by large values of $\sigma_0$.

By adjusting these parameters, we can replicate both linear and logarithmic observed regret behaviors as seen in the human data. Figure \ref{fig:linlogRegretExample} shows examples of simulated observed regret $\mathcal{R}(t)$ that capture linear and logarithmic regret, respectively. In both examples, Landscape B was used for the mean rewards. The example with linear regret shows a case where the agent has fairly uninformative and fully uncorrelated  prior beliefs (i.e., $\lambda = 0$). The prior mean $\mu_0 = 30$ is set equal to the true surface mean, but with $\sigma_0^2 = 1000$, so that the agent is not very certain of this value. Moderate decision noise is incorporated by setting $\upsilon = 4$. The values of the prior encourage the agent to explore most of the $N = 100$ options in the $T = 90$ choices, yielding regret that is linear in time. As emphasized in Remark \ref{rmk:shortHorizon}, the deterministic UCL algorithm (and any agent employing the algorithm) with an uninformative prior cannot in general achieve sub-linear cumulative expected regret in a task with such a short horizon. The addition of decision noise to this algorithm will tend to increase regret, making it harder for the agent to achieve sub-linear regret.

In contrast, the example with logarithmic regret shows how an informative prior with an appropriate correlation structure can significantly improve the agent's performance. The prior mean $\mu_0 = 200$ encourages more exploration than the previous value of 30, but the smaller value of $\sigma_0^2 = 10$ means the agent is more confident in its belief and will explore less. The correlation structure induced by setting the length scale $\lambda = 4$ is a good model for the reward surface, allowing the agent to more quickly reject areas of low rewards. A lower softmax temperature $\upsilon = 1$ means that the agent's decisions are made more deterministically. Together, these differences lead to the agent's logarithmic regret curve; this agent suffers less than a third of the total regret during the task as compared to the agent with the poorer prior and linear regret.

\begin{figure}[h]
   \centering
   \includegraphics[width=3.5in]{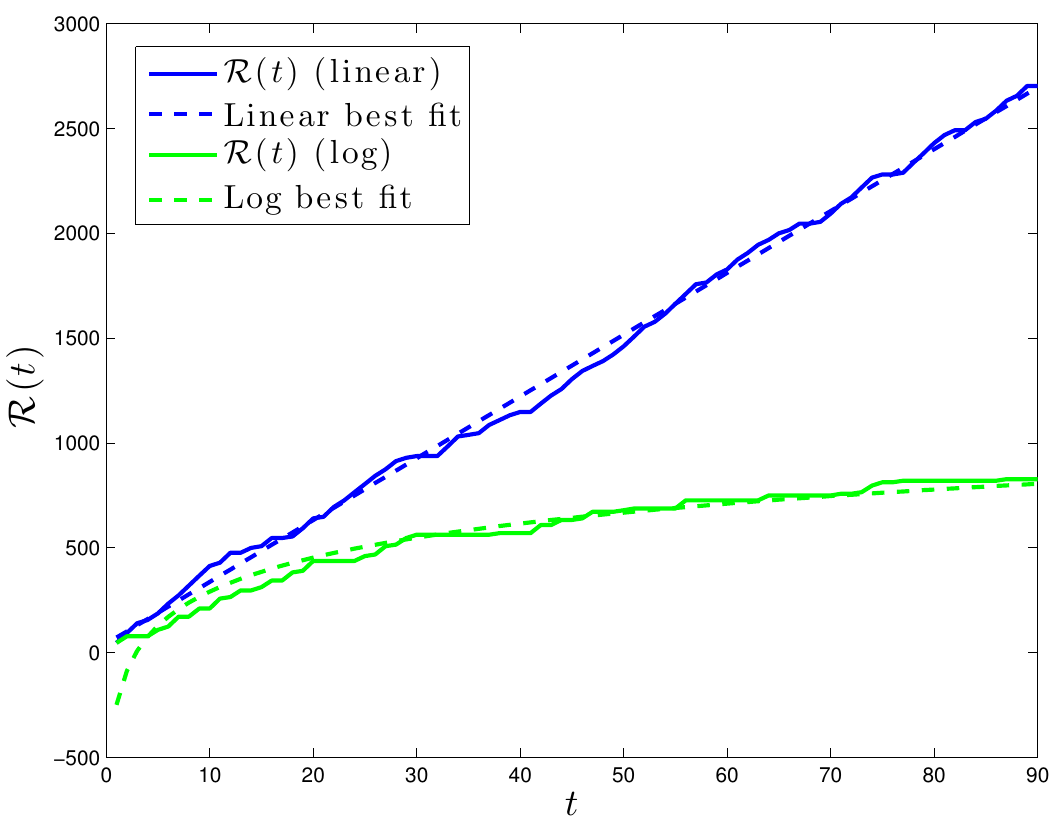}
   \caption{Observed regret $\mathcal{R}(t)$ from simulations (solid lines) that demonstrate linear \eqref{eq:linModel}, blue curves, and log \eqref{eq:logModel}, green curves, regret. The best fits to the simulations are shown (dashed lines). The simulated task parameters were identical to those of the human participant task with Landscape B from Figure~\ref{fig:surfaceProfiles}(b). In the example with linear regret, the agent's prior on rewards was the uncorrelated prior $\mu_0 = 30$, $\sigma_0^2 = 1000, \lambda = 0$. Decision noise was incorporated using softmax selection with a constant temperature $\upsilon = 4$. In the example with log regret, the agent's prior on rewards was the correlated prior with uniform $\mu_0 = 200$ and $\Sigma_0$ an exponential prior \eqref{eq:spatialCovPrior} with parameters $\sigma_0^2 = 10, \lambda = 4$. The decision noise parameter was set to $\upsilon = 1$.}
   \label{fig:linlogRegretExample}
\end{figure}

\section{Gaussian multi-armed bandit problems \\ with transition costs}
Consider an $N$-armed bandit problem as described in Section~\ref{sec:gaussian-bandit}. 
Suppose that the decision-maker incurs a random transition cost $c_{ij}\in \real_{\ge 0}$ for a transition from arm $i$ to arm $j$. No cost is incurred if the decision-maker chooses the same arm as the previous time instant, and accordingly, $c_{ii}=0$. Such a cost structure corresponds to a search problem in which the $N$ arms may correspond to $N$ spatially distributed regions and the transition cost $c_{ij}$ may correspond to the travel cost from region $i$ to region $j$. 

To address this variation of the multi-armed bandit problem, we extend the UCL algorithm to a strategy that makes use of block allocations.   Block allocations refer to sequences in which the same choice is made repeatedly;  thus, during a block no transition cost is incurred.   The UCL algorithm is used to make the choice of arm at the beginning of each block.   The design of the (increasing) length of the blocks makes the block algorithm provably efficient.   This model can be used in future experimental work to investigate human behavior in multi-armed bandit tasks with transition costs.  

\subsection{The Block UCL Algorithm}
For Gaussian multi-armed bandits with transition costs, we develop a block allocation strategy described graphically in Figure \ref{fig:blockAllocation} and in pseudocode in Algorithm~\ref{algo:block-ucb} in Appendix F. The intuition behind the strategy is as follows. The decision-maker's objective is to maximize the total expected reward while minimizing the number of transitions. As we have shown, maximizing total expected reward is equivalent to minimizing expected regret, which we know grows at least logarithmically with time. If we can bound the number of expected cumulative transitions to grow less than logarithmically in time, then the regret term will dominate and the overall objective will be close to its optimum value. Our block allocation strategy is designed to make transitions less than logarithmically in time, thereby ensuring that the expected cumulative regret term dominates.

We know from the Lai-Robbins bound \eqref{eq:optBound} that the expected number of selections of suboptimal arms $i$ is at least $\mathcal{O}(\log T)$. Intuitively, the number of transitions can be minimized by selecting the option with the maximum upper credible limit $\lceil \log T \rceil$ times in a row. However, such a strategy will have a strong dependence on $T$ and will not have a  good performance uniformly in time. To remove this dependence on $T$, we divide the set of natural numbers (choice instances) into frames $\setdef{f_{k}}{k\in \naturals}$ such that frame $f_k$ starts at time $2^{k-1}$ and ends at time $2^k-1$. Thus,  the length of frame $f_k$ is $2^{k-1}$.

We subdivide frame $f_k$ into blocks each of which will correspond to a sequence of choices of the same option. Let the first $\lfloor 2^{k-1}/k\rfloor$ blocks in frame $f_k$ have length $k$ and the remaining choices in frame $f_k$ constitute a single block of length $2^{k-1} -\lfloor 2^{k-1}/k\rfloor k$. The time associated with the choices made within frame $f_k$ is $\mathcal{O}(2^k)$. Thus, following the intuition in the last paragraph, the length of each block in frame $f_k$ is chosen  equal to $k$, which is $\mathcal{O}(\log (2^k))$.  

The total number of blocks in frame $f_k$ is $b_k = \lceil 2^{k-1}/k\rceil$.
Let $\ell\in \naturals$ be the smallest index such that $T < 2^{\ell}$. 
Each block is characterized by the tuple $(k,r)$, for some $k\in \until{\ell}$, and $r\in \until{b_{k}}$, where $k$ identifies the frame and $r$ identifies the block within the frame. 
We denote the time at the start of block $r$ in frame $f_k$ by $\tau_{kr} \in \naturals$.
The block UCL algorithm at time $\tau_{kr}$ selects the arm with the maximum $(1-1/K\tau_{kr} )$-upper credible limit and chooses it $k$ times in a row ($\le k$ times if the  block $r$ is the last  block in frame $f_k$).    The choice at time $\tau_{kr}$ is analogous to the choice at each time instant in the UCL algorithm.

\begin{figure}[h]
\centering
\subfigure[]{\includegraphics[width=0.9\linewidth]{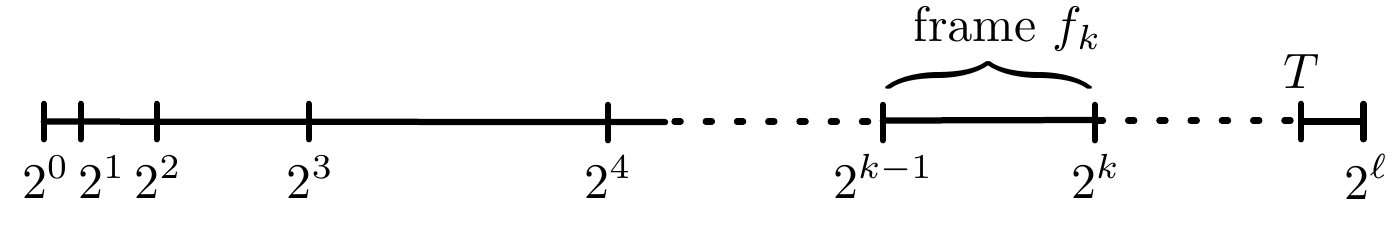}}
\subfigure[]{\includegraphics[width=0.5\linewidth]{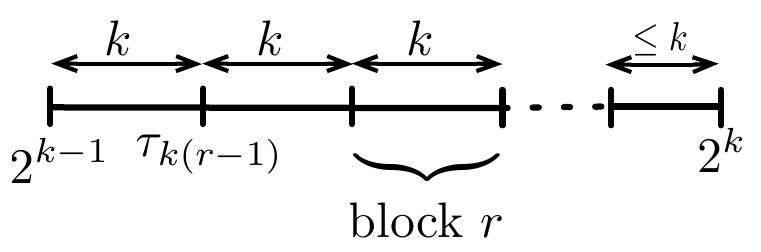}}
\caption{The block allocation scheme used in the block UCL algorithm. Decision time $t$ runs from left to right in both panels. Panel (a) shows the division of the decision times $t \in \until{T}$ into frames $k \in \until{\ell}$. Panel (b) shows how an arbitrary frame $k$ is divided into blocks. Within the frame, an arm is selected at time $\tau_{kr}$, the start of each block $r$ in frame $k$,  and that arm is selected for each of the $k$ decisions in the block.}

\label{fig:blockAllocation}
\end{figure}

Next, we analyze the regret of the block UCL algorithm. 
 We first introduce some notation. Let $Q_i^{kr}$ be the $(1-1/K\tau_{kr} )$-upper credible limit for the mean reward of arm $i$ at allocation round $(k,r)$, where $K=\sqrt{2 \pi e}$ is the credible limit parameter. 
Let $n_i^{kr}$ be the number of times arm $i$ has been chosen until time $\tau_{kr}$ (the start of block $(k,r)$).
Let $s_i^{t}$ be the number of times the decision-maker transitions to arm $i$ from another arm $j\in \until{N}\setminus \{i\}$ until time $t$.
Let the empirical mean of the rewards from arm $i$ until time $\tau_{kr}$  be $\bar m_i^{kr}$.
Conditioned on the number of visits $n_i^{kr}$ to arm $i$ and the empirical mean $\bar m_i^{kr}$, the mean reward   at arm $i$ at time  $\tau_{kr}$  is a Gaussian random variable  ($M_i$) with mean and variance
\begin{align*}
\mu_{i}^{kr}&:=\expt[M_i|n_i^{kr}, \bar m_i^{kr}] = \frac{\delta^2 \mu_{i}^0 + n_i^{kr} \bar m_i^{kr}}{\delta^2+n_{i}^{kr}}, \; \text{and}\;\\
{\sigma_i^{kr}}^2 &:= \text{Var}[M_i|n_i^{kr}, \bar m_i^{kr}] =\frac{\sigma_s^2}{\delta^2 + n_{i}^{kr}},
\end{align*}
respectively. Moreover,
\[
\expt[\mu_{i}^{kr}|n_i^{kr}] \!= \!\frac{\delta^2 \mu_{i}^0 + n_i^{kr} m_i}{\delta^2+n_{i}^{kr}} \; \text{and}\;
\text{Var}[\mu_{i}^{kr}|n_i^{kr}] \!= \!\frac{ n_i^{kr} \sigma_s^2}{(\delta^2+n_{i}^{kr})^2}.
\]
Accordingly, the $\left( 1 - 1/K \tau_{k,r} \right)$-upper credible upper limit $Q_i^{kr}$ is
\[
Q_i^{kr} =  \mu_i^{kr} + 
\frac{\sigma_s}{\sqrt{\delta^2 + n_{i}^{kr}}}  \Phi^{-1}\Big(1- \frac{1}{K \tau_{kr}}\Big).
\]
Also, for each $i\in\until{N}$, we define constants
\begin{align*}
\gamma_1^i &=  \frac{8 \beta^2 \sigma_s^2}{\Delta_i^2} +\frac{1}{\log 2} + \frac{2}{K}, \\
\gamma_2^i &=\frac{4 \beta^2 \sigma_s^2}{\Delta_i^2} (1 - \log 2) + 2 +\frac{8}{K} +\frac{\log 4}{K},\\
\gamma_3^i &= \gamma_1^i \log 2(2 - \log \log 2) \\
& \qquad \qquad - \Big(\frac{4 \beta ^2 \sigma_s^2 }{\Delta_i^2}  \log \log 2 - \gamma_2^i\Big) \Big( 1+ \frac{\pi^2}{6}\Big),\text{ and }\\
{\bar c_{i}}^{\max}&=\max\setdef{\expt[c_{ij}]}{j\in\until{N}}.
\end{align*}

Let $\seqdef{\supscr{R}{BUCL}_t}{t\in\until{T}}$ be the sequence of the expected regret of the block UCL algorithm, and  $\seqdef{\supscr{S}{BUCL}_t}{t\in\until{T}}$ be the sequence of expected transition costs. The block UCL algorithm achieves logarithmic   regret uniformly in time as formalized in the following theorem.
\begin{theorem}[\bit{Regret of block UCL algorithm}]\label{thm:block-ucl}
The following statements hold for the Gaussian multi-armed bandit problem with transition costs and the block UCL algorithm with an uncorrelated uninformative prior:
\begin{enumerate}
\item the expected  number of times a suboptimal arm $i$ is chosen until  time $T$ satisfies
\[
\expt[n_i^{T}] \le \gamma_1^i \log T -\frac{4 \beta ^2 \sigma_s^2}{\Delta_i^2} \log \log T +\gamma_2^i ;
\]
\item the expected  number of transitions to a suboptimal arm $i$ from another arm until time $T$ satisfies
\[
\expt[s_i^T] \le (\gamma_1^i \log 2)  \log \log T  + \gamma_3^i; 
\]
\item the cumulative expected regret and the cumulative transition cost until time $T$ satisfy
\begin{align*}
\sum_{t=1}^T  \supscr{R}{BUCL}_t  &\le \sum_{i=1}^{N}  \Delta_i \Big(\gamma_1^i \log T -\frac{4 \beta^2 \sigma_s^2}{\Delta_i^2} \log \log T + \gamma_2^i\Big), \\
\sum_{t=1}^T  \supscr{S}{BUCL}_t  &\le  
 \sum_{i=1, i\ne i^*}^{N} ({\bar c_{i}}^{\max}+  {\bar c_{i^*}}^{\max} ) \times \\ 
&  \qquad \qquad ((\gamma_1^i \log 2)  \log \log T  + \gamma_3^i) + {\bar c_{i^*}}^{\max}.
\end{align*}
\end{enumerate}
\end{theorem}
\begin{proof}
See Appendix D. 
\end{proof}

Figures \ref{fig:block-ucl-regret} and \ref{fig:block-ucl-transcost} show, respectively, the cumulative expected regret and the cumulative transition cost of the block UCL algorithm on a bandit task with transition costs. For comparison, the figures also show the associated bounds from statement (iii) of Theorem \ref{thm:block-ucl}. 
Cumulative expected regret was computed  using 250 runs of the block UCL algorithm. Variance of the regret was minimal. The task used the reward surface of Landscape B from Figure~\ref{fig:surfaceProfiles}(b) with sampling noise variance $\sigma_s^2 = 1$. The algorithm used an uncorrelated prior with $\mu_0 = 200$ and $\sigma_0^2 = 10^6$. Transition costs between options were equal to the distance between them on the surface.

The variance of the cumulative regret is relatively small, i.e., the cumulative regret experienced in a given task is close to the expected value. Also the bound on transition costs is quite loose. This is due to the loose bound on the expected number of transitions to the optimal arm. More detailed analysis of the total number of transitions would allow the bound to be tightened.

\begin{figure}
\centering

\includegraphics[width=0.45\textwidth]{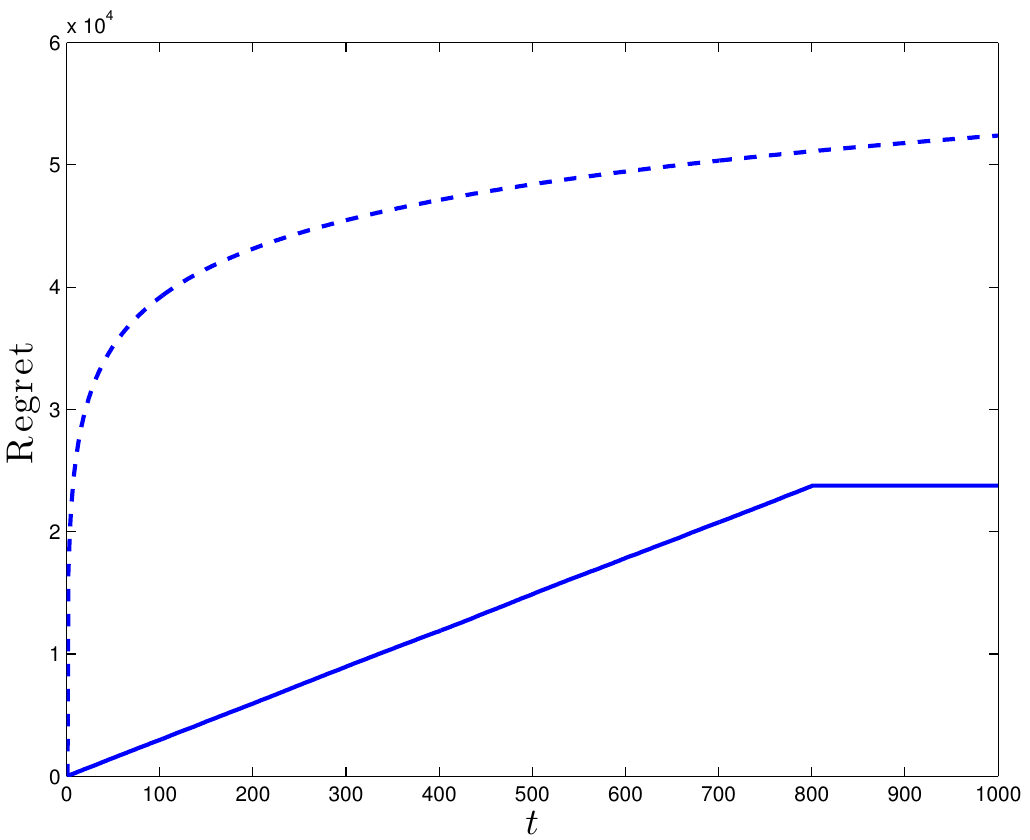}

\caption{Cumulative expected regret (solid line) and the associated bound (dashed line) from Theorem \ref{thm:block-ucl}. Expected regret was computed using 250 runs of the block UCL algorithm; variance of the regret was minimal. The task used the reward surface from Figure~\ref{fig:surfaceProfiles}(b) with sampling noise variance $\sigma_s^2 = 1$. The algorithm used an uncorrelated prior with $\mu_0 = 200$ and $\sigma_0^2 = 10^6$.}

\label{fig:block-ucl-regret}
\end{figure}

\begin{figure}
\centering

\includegraphics[width=0.45\textwidth]{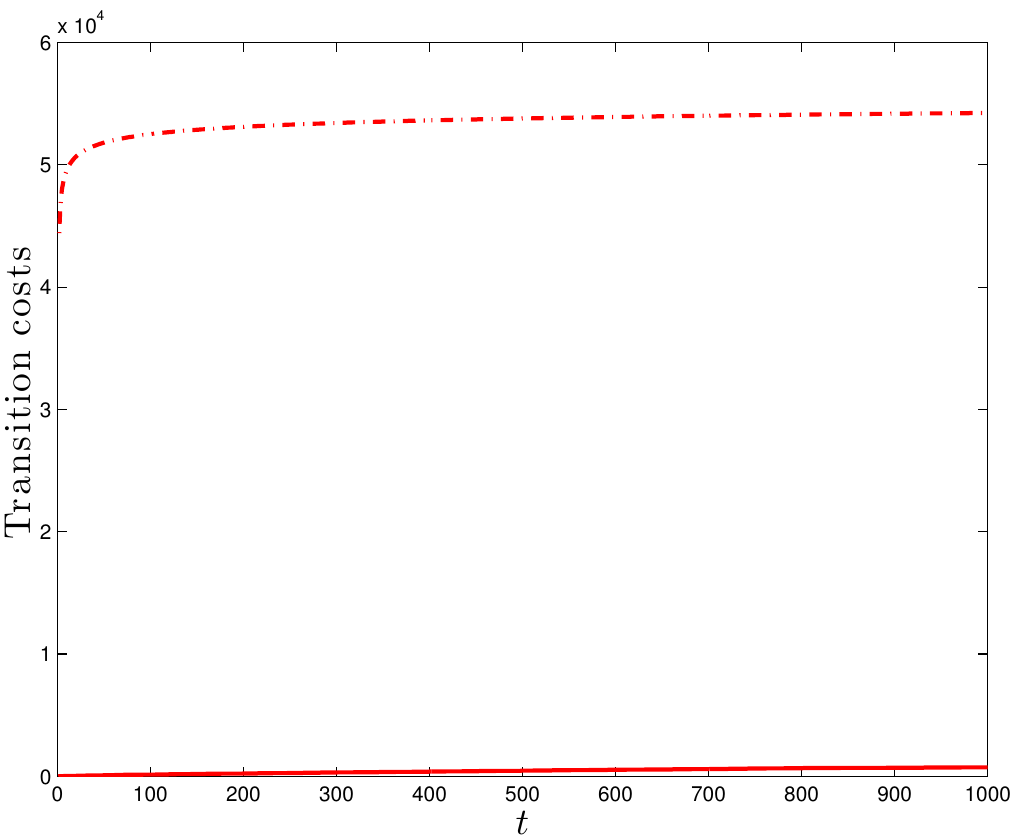}

\caption{Cumulative transition cost (solid line) and the associated bound (dashed line) from Theorem \ref{thm:block-ucl}. Transition costs were computed using 250 runs of the block UCL algorithm with the same parameters as in Figure \ref{fig:block-ucl-regret}. Transition costs between any two arms $i$ and $j$ were deterministic and set equal to $\left| x_i - x_j \right|$, where $x_i$ is the location of arm $i$ in the grid.}

\label{fig:block-ucl-transcost}
\end{figure}

\section{Graphical Gaussian multi-armed bandit problems} \label{sec:graphical-bandits}
We now consider multi-armed bandits with Gaussian rewards in which the decision-maker cannot move to every other arm from the current arm. Let the set of arms that can be visited from arm $i$ be $\mathrm{ne}(i) \subseteq\until{N}$. Such a multi-armed bandit can be represented by a graph $\mc G$ with node set $\until{N}$ and edge set $\mc E =\setdef{(i,j)}{j \in \mathrm{ne}(i), i\in \until{N}}$.  We assume that the graph is connected in the sense that there exists at least one path from each node $i \in \until{N}$ to every other node $j \in \until{N}$. Let $\mc P^{ij}$ be the set of intermediary nodes in a shortest path from node $i$ to node $j$. Note that the set $\mc P^{ij}$ does not contain node $i$ nor node $j$. We denote the cardinality of the set $\mc P^{ij}$ by $p_{ij}$ and accordingly, the elements of the set $\mc P^{ij}$ by $\{P^{ij}_1, \ldots, P^{ij}_{p_{ij}}\}$. 

\subsection{The graphical block UCL algorithm}
For graphical Gaussian multi-armed bandits, we develop an algorithm similar to the block allocation Algorithm~\ref{algo:block-ucb}, namely, the graphical block UCL algorithm, described in pseudocode in Algorithm~\ref{algo:graph-ucb} in Appendix F. Similar to the block allocation algorithm, at each block, the arm with maximum upper credible limit is determined. Since the arm with the maximum upper credible limit may not be immediately reached from the current arm, the graphical block UCL algorithm traverses a shortest path from the current arm to the arm with maximum upper credible limit.   Traversing a shortest path will mean making as many as $N-2$ visits to undesirable arms ($N-2$ is the worst case in a line graph where the current location is at one end of the line and the desired arm is at the other end of the line).  Thus, we apply a block allocation algorithm to limit the number of transitions as in the case of Algorithm~\ref{algo:block-ucb}  for the bandit problem with transition costs.

We classify the selection of arms in two categories, namely, \emph{goal} selection and \emph{transient} selection. The goal selection of an arm corresponds to the situation in which the arm  is selected because it has the maximum upper credible limit, while the transient selection corresponds to the situation in which the arm is selected because it belongs to the shortest path to the arm with the maximum credible limit. Accordingly, we define the block associated with the goal selection of an arm as the \emph{goal block}, and the block associated with arms on the shortest path between two arms associated with consecutive goal blocks as the \emph{transient block}.
The design of the blocks is pictorially depicted in Figure~\ref{fig:graph-blockAllocation}.

\begin{figure}[h]
\centering
\subfigure[Frames associated with {goal} selection of arms]{\includegraphics[width=0.9\linewidth]{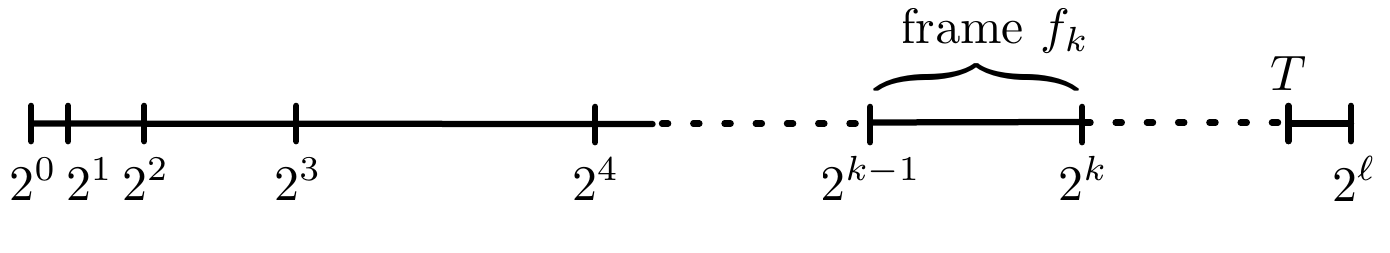}}
\subfigure[Goal and transient blocks within each frame]{\includegraphics[width=0.7\linewidth]{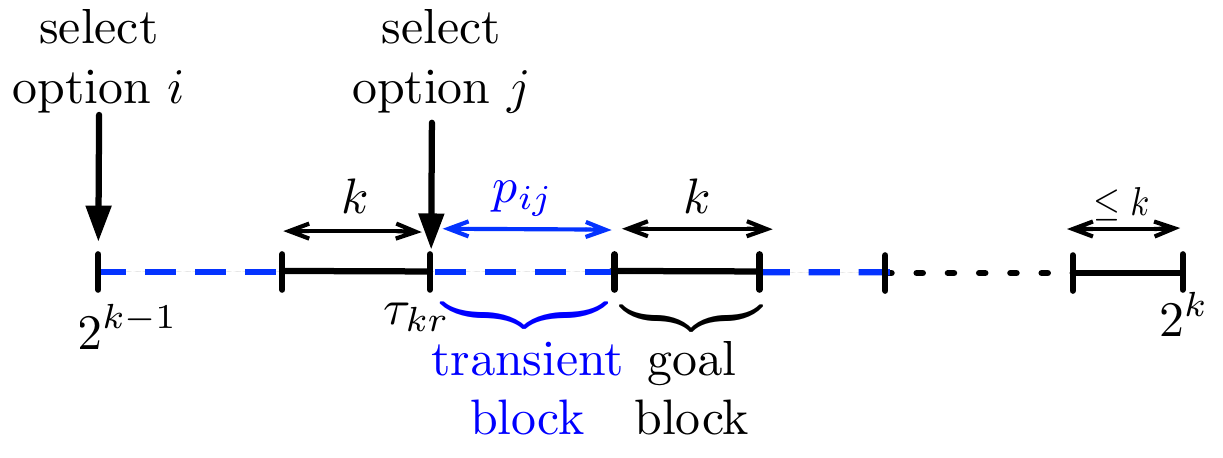}}
\caption{The block allocation scheme used in the graphical block UCL algorithm.
Decision time $t$ runs from left to right in both panels. Panel (a) shows the division of goal selection instances of the arms into frames $k \in \until{\ell}$. The frame $f_k$ corresponds to $2^{k-1}$ goal selections of the arms.
Panel (b) shows how an arbitrary frame $f_k$ is divided into goal and transient blocks. The goal blocks are selected as they are in the block allocation strategy, while the transient blocks correspond to the shortest path between two arms associated with consecutive goal blocks. Only the goal selections are counted to compute the length of the frames.}
\label{fig:graph-blockAllocation}
\end{figure}

The goal selection instances of arms are subdivided into frames $f_k, k\in \until{\ell}$, where the length of the frame $f_k$ is $2^{k-1}$. Note that only goal selections are counted to compute the length of each frame.
The length of the goal blocks within each frame is chosen as it is in  the block allocation strategy.
We denote the time at the start of the transient block before goal block $r$ in frame $f_k$ by $\tau_{kr} \in \naturals$.
The graphical block allocation algorithm at time $\tau_{kr}$ (i) determines the arm with the maximum $(1-1/K\tau_{kr})$-upper credible limit, (ii) traverses the shortest path to the arm, (iii) picks the arm $k$ times ($\le k$ times if the goal block $r$ is the last goal block in frame $f_k$).   In Figure~\ref{fig:graph-blockAllocation} the goal block only shows the choices of the goal selection.  The transient block, shown prior to the corresponding  goal block, accounts for the selections along a shortest path.

The key idea behind the algorithm is that the block allocation strategy results in an expected number of transitions that is sub-logarithmic in the horizon length. In the context of the graphical bandit, sub-logarithmic transitions result in sub-logarithmic \emph{undesired} visits to the arms on the chosen shortest path to the \emph{desired} arm with maximum upper credible limit. Consequently, the cumulative expected regret of the algorithm is dominated by a logarithmic term. 

\subsection{Regret analysis of the graphical block UCL algorithm}
We now analyze the performance of the graphical block UCL algorithm. 
Let $\seqdef{\supscr{R}{GUCL}_t }{t\in\until{T}}$ be the sequence of expected regret of the graphical block UCL algorithm. The graphical block UCL algorithm achieves logarithmic   regret uniformly in time as formalized in the following theorem.

\begin{theorem}[\bit{Regret of graphical block UCL algorithm}]\label{thm:graph-block-ucl}
The following statements hold for the graphical Gaussian multi-armed bandit problem with the graphical block UCL algorithm and an uncorrelated uninformative prior:
\begin{enumerate}
\item the expected  number of times a suboptimal arm $i$ is chosen until  time $T$ satisfies
\begin{multline*}
\expt[n_i^{T}] \le \gamma_1^i \log T -\frac{4 \beta ^2 \sigma_s^2}{\Delta_i^2} \log \log T +\gamma_2^i \\
+ \sum_{i=1, i\ne i^*}^N \big(  (2 \gamma_1^i \log 2)  \log \log T  + 2 \gamma_3^i\big) + 1;
\end{multline*}
\item the cumulative expected regret until time $T$ satisfies
\begin{multline*}
\sum_{t=1}^T \supscr{R}{GUCL}_t \le \sum_{i=1}^N  \Big( \gamma_1^i \log T -\frac{4 \beta ^2 \sigma_s^2}{\Delta_i^2} \log \log T  \\
+\gamma_2^i + \sum_{i=1, i\ne i^*}^N \big(  (2 \gamma_1^i \log 2)  \log \log T  + 2 \gamma_3^i\big) + 1 \Big)\Delta_i.
\end{multline*}
\end{enumerate}
\end{theorem}

\begin{proof}
See Appendix E. 
\end{proof}

Figure \ref{fig:graph-block-ucl} shows cumulative expected regret and the associated bound from Theorem \ref{thm:graph-block-ucl} for the graphical block UCB algorithm. The underlying graph topology was chosen to be a line graph, so the algorithm could only choose  to move one step forwards or backwards at each time. Expected regret was computed using 250 runs of the graphical block UCL algorithm. Each task consisted of $N = 10$ bandits with mean rewards set equal to the reward profile along the $x$-axis of Figure~\ref{fig:surfaceProfiles}(b). Reward variance was $\sigma_s^2 = 6.25$, while the agent used the uncorrelated prior with $\mu_0 = 40$ and $\sigma_0^2 = 10^6$. Note that the regret bound is quite loose, as in the case of transition costs for the block UCL algorithm. This is because the regret bound uses the same bound on switching costs as in Theorem~\ref{thm:block-ucl} to bound the regret incurred by traversing the graph.

\begin{figure}
\centering

\includegraphics[width=0.45\textwidth]{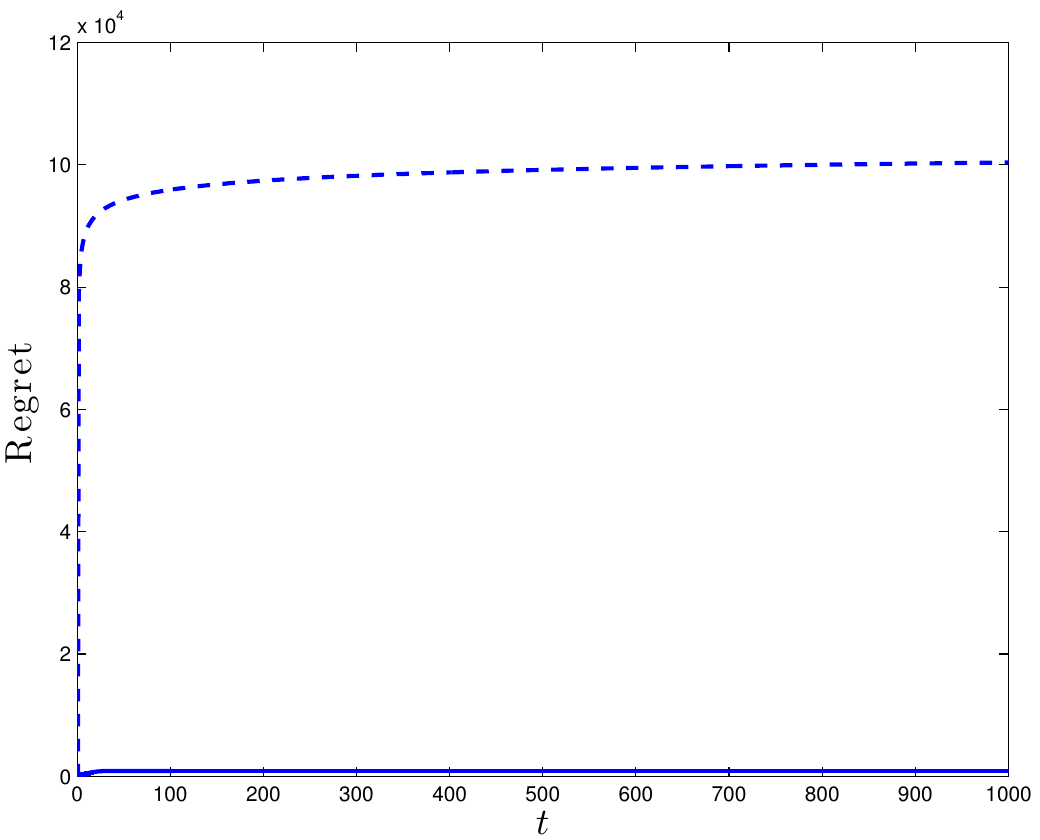}

\caption{Cumulative expected regret (solid line) and the associated bound (dashed line) from Theorem \ref{thm:graph-block-ucl}. Expected regret was computed from 250 simulated tasks played using the graphical block UCL algorithm. Each task consisted of $N = 10$ bandits with mean rewards set equal to the reward profile from Figure~\ref{fig:surfaceProfiles}(b). The graph topology was a line graph, so the agent could only move one step forwards or backwards at each time. Reward variance was $\sigma_s^2 = 6.25$, while the agent used the uncorrelated prior with $\mu_0 = 40$ and $\sigma_0^2 = 10^6$.}

\label{fig:graph-block-ucl}
\end{figure}

\section{Conclusions}
In this paper, we considered multi-armed bandit problems with Gaussian rewards and studied them from a Bayesian perspective. We considered three particular multi-armed bandit problems: the standard multi-armed bandit problem, the multi-armed bandit problem with transition costs, and the graphical multi-armed bandit problem. We developed two UCL algorithms, namely, the deterministic UCL algorithm and the stochastic UCL algorithm, for the standard multi-armed bandit problem. We extended the deterministic UCL algorithm to the block UCL algorithm and the graphical block UCL algorithm for the multi-armed bandit problem with transition costs, and the graphical multi-armed bandit problem, respectively. 
We established that for uninformative priors, each of the proposed algorithms achieves logarithmic regret uniformly in time, and moreover, the block UCL algorithm achieves a sub-logarithmic expected number of transitions among arms. We elucidated the role of general priors and the correlation structure among arms, showing how good priors and good assumptions on the correlation structure among arms can greatly enhance decision-making performance of the proposed deterministic UCL algorithm, even over short time horizons.

We drew connections between the features of the stochastic UCL and human decision-making in multi-armed bandit tasks. 
In particular, we showed how the stochastic UCL algorithm captures  five key features of human decision-making in multi-armed bandit tasks, namely, (i) familiarity with the environment, (ii) ambiguity bonus, (iii) stochasticity, (iv) finite-horizon effects, and (v) environmental structure effects. 
We then presented empirical data from human decision-making experiments on a spatially-embedded multi-armed bandit task and demonstrated that the observed performance is efficiently captured by the proposed stochastic UCL algorithm with appropriate parameters. 

This work presents several interesting avenues for future work in the design of  human-automata systems. The model phenotypes discussed in Section \ref{sec:humanPerformance} provide a method for assessing human performance in real time, and the experimental results presented in that section suggest that some humans use informative priors for spatial search tasks which allow them to achieve better performance than a similar algorithm using uninformative priors. Therefore, a useful goal for human-automata systems would be to develop a means to learn the humans' informative priors  and use them to improve the performance of the overall system.

This work also presents several interesting avenues for future psychological research.
First, in this work, we relied on certain functional forms for the parameters in the algorithms, e.g., we considered credibility parameter $\alpha_t =1/Kt$ and cooling schedule $\upsilon_t = \nu/\log t$. It is of interest to perform thorough experiments with human subjects to ascertain the correctness of these functional forms. Second,  efficient methods for estimation of parameters in the proposed algorithms need to be developed.

Overall, the proposed algorithms provide ample insights into plausible decision mechanisms involved with human decision-making in tasks with an explore-exploit tension.  We envision a rich interplay between these algorithms and psychological research.

\section*{Acknowledgements}
\addcontentsline{toc}{section}{Acknowledgement }
The authors wish to thank John Myles White, Robert C. Wilson, Philip Holmes and Jonathan D. Cohen for their input, which helped make possible the strong connection of this work to the psychology literature. We also thank the editor and two anonymous referees, whose comments greatly strengthened and clarified the paper. The first author is grateful to John Myles White and Daniel T. Swain for their help with implementing the online experiment.


\section*{Appendix}
\subsection{Proof of inverse Gaussian tail bound}
\label{AppA}

\vspace{2mm}
\begin{proof}[Proof of Theorem~\ref{thm:tailBounds}]
We start by establishing inequality~\eqref{eq:QDown}. It suffices to establish this inequality for $\beta=1.02$.
Since the cumulative distribution function for the standard normal random variable is a continuous and monotonically increasing function, it suffices to show that
\begin{equation} \label{eq:equivalent-statement}
\Phi(\beta \sqrt{- \log(-2\pi \alpha^2 \log(2\pi \alpha^2))}) + \alpha -1 \ge 0,
\end{equation}
for each $\alpha \in (0,1)$. Equation~\eqref{eq:equivalent-statement} can be equivalently written as $h(x) \ge 0$, where $x= 2\pi \alpha^2$ and $\map{h}{(0,1)}{(0, 1/\sqrt{2\pi})}$ is defined by 
\[
h(x)= \Phi(\beta \sqrt{-\log(-x\log x))})+\frac{\sqrt{x}}{\sqrt{2 \pi}}-1.
\]
Note that $\lim_{x\to 0^+} h(x) =0$ and $\lim_{x\to 1^-} h(x) =1/\sqrt{2\pi}$. Therefore, to establish the theorem, it suffices to establish that $h$ is a monotonically increasing function. It follows that
\[
g(x):= 2\sqrt{2\pi}h'(x)= \frac{1}{\sqrt{x}} + \frac{\beta (-x \log x)^{\beta^2/2 - 1} (1+\log x)}{\sqrt{-\log(-x \log x)}}.
\]
Note that $\lim_{x \to 0^+} g(x) =+\infty$ and $\lim_{x \to 1^-} g(x) =1$. Therefore, to establish that $h$ is monotonically increasing, it suffices to show that $g$ is non-negative for $x \in (0,1)$. This is the case if the following inequality holds:

\[
g(x) = \frac{1}{\sqrt{x}} +  \frac{\beta (-x \log x)^{\beta^2/2 - 1} (1+\log x)}{\sqrt{-\log(-x \log x)}} \ge 0,\]
which holds if
\[\frac{1}{\sqrt{x}} \ge -\frac{\beta (-x \log x)^{\beta^2/2 - 1} (1+\log x)}{\sqrt{-\log(-x \log x)}}.\]
The inequality holds if the right hand side is negative. If it is positive, one can take the square of both sides and the inequality holds if
\begin{align*}
-\log(-x &\log x) \ge \beta^2 x (1+\log x)^2 (-x \log x)^{\beta^2 - 2}\\
	&= \beta^2x(1+2\log x+(\log x)^2)(-x \log x)^{\beta^2 - 2}.
\end{align*}
Letting $t = -\log x$, the above inequality transforms to 
\[-\log(t e^{-t}) \ge \beta^2 e^{-t}(1-2t + t^2)(t e^{-t})^{\beta^2 - 2},\]
which holds if 
\[-\log t \ge \beta^2  t^{\beta^2-2}(1-2t + t^2)e^{-(\beta^2 - 1)t} -t.\]
Dividing by $t$, this is equivalent to
\[-\frac{\log t}{t} \ge \beta^2  t^{\beta^2-3}(1-2t + t^2)e^{-(\beta^2 - 1)t} -1,\]
which is true if
\begin{equation}
\inf_{t\in[1,+\infty)} -\frac{\log t}{t} \ge \max_{t \in [1,+\infty)}\beta^2  t^{\beta^2-3}(1-2t + t^2)e^{-(\beta^2 - 1)t} -1.\label{eq:boundFnOft}
\end{equation}

These extrema can be calculated analytically, so we have
\[ \inf_{t\in[1,+\infty)} -\frac{\log t}{t} = -\frac{1}{e} \approx -0.3679 \]
for the left hand side and 
\begin{align*}
t^* =\underset{t \in [1,+\infty)}{\argmax\; } &\beta^2  t^{\beta^2-3}(1-2t + t^2)e^{-(\beta^2 - 1)t} -1\\
 &= 1+\sqrt{2/(\beta^2-1)} \\
 \implies  \!\!\!\! \max_{t \in [1,+\infty)}& \!\! \; \beta^2  t^{\beta^2-3}(1-2t + t^2)e^{-(\beta^2 - 1)t}  \!-1 \approx -0.3729,
\end{align*}
for the right hand side of \eqref{eq:boundFnOft}. Therefore, \eqref{eq:boundFnOft} holds.
In consequence, $g(x)$ is non-negative for $x \in (0,1)$, $h(x)$ is a monotonically increasing function. This establishes inequality~\eqref{eq:QDown}. Inequality~\eqref{eq:QUp} follows analogously. 
\end{proof}

\subsection{Proof of regret of the deterministic UCL algorithm}
\label{AppB}

\begin{proof}[Proof of Theorem~\ref{thm:UCL}]
We start by establishing the first statement.
In the spirit of~\cite{PA-NCB-PF:02}, we bound $n_{i}^T$ as follows:
\begin{align*}
n_{i}^T &= \sum_{t=1}^T \indicator{i_t=i}\\
 &\leq \sum_{t = 1}^T \indicator{Q_{i}^t > Q_{i^*}^t}\\
  & \leq \eta + \sum_{t = 1}^T \indicator{Q_{i}^t > Q_{i^*}^t, n_{i}^{(t-1)} \geq \eta},
 \end{align*}
where $\eta$ is some positive integer and $\indicator{x}$ is the indicator function, with $\indicator{x} = 1$ if $x$ is a true statement and $0$ otherwise.

At time $t$, the agent picks option $i$ over $i^*$ only if
\[ Q_{i^*}^t \leq Q_{i}^t.\]
This is true when at least one of the following equations holds:
\begin{align}
\mu_{i^*}^t &\leq m_{i^*} - C_{i^*}^t \label{eq:Qi*}\\
\mu_{i}^t &\geq m_i + C_{i}^t\label{eq:Qi}\\
m_{i^*} &< m_i + 2C_{i}^t\label{eq:miComp}
\end{align}
where $C_{i}^t = \frac{\sigma_s}{\sqrt{\delta^2 + n_{it}}} \Phi^{-1}(1-\alpha_t)$ and $\alpha_t = 1/Kt$. Otherwise, if none of the equations (\ref{eq:Qi*})-(\ref{eq:miComp}) holds,
\[ Q_{i^*}^t = \mu_{i^*}^t+ C_{i^*}^t > m_{i^*} \geq m_i + 2C_{i}^t > \mu_{i}^t + C_{i}^t = Q_{i}^t,\]
and option $i^*$ is picked over option $i$ at time t.

We proceed by analyzing the probability that Equations (\ref{eq:Qi*}) and (\ref{eq:Qi}) hold. Note that the empirical mean $\bar{m}_{i}^t$ is a normal random variable with mean $m_i$ and variance $\sigma_s^2/n_{i}^t$, so, conditional on $n_{i}^t$, $\mu_{i}^t$ is a normal random variable distributed as
\[ \mu_{i}^t \sim \mathcal{N}\left(\frac{\delta^2 \mu_{i}^0+ n_{i}^t m_i}{\delta^2 + n_{i}^t}, \frac{n_{i}^t \sigma_s^2}{(\delta^2 + n_{i}^t)^2} \right).\]

Equation (\ref{eq:Qi*}) holds if
\begin{align*}
&m_{i^*} \geq \mu_{i^*}^t + \frac{\sigma_s}{\sqrt{\delta^2 + n_{i}^t}} \Phi^{-1}(1-\alpha_t)\\
\iff &m_{i^*}-\mu_{i^*}^t \geq \frac{\sigma_s}{\sqrt{\delta^2 + n_{i}^t}} \Phi^{-1}(1-\alpha_t)\\
\iff & z \leq -\sqrt{\frac{n_{i^*}^t + \delta^2}{n_{i^*}^t}}  \Phi^{-1}(1-\alpha_t) 
 + \frac{\delta^2}{\sigma_s} \frac{\Delta m_{i^*}}{\sqrt{n_{i^*}^t}},
\end{align*}
where $z \sim \mcN(0,1)$ is a standard normal random variable and $\Delta m_{i^*} = m_{i^*}-\mu_{i^*}^0$. 
For an uninformative prior $\delta^2 \to 0^+$, and consequently, equation~\eqref{eq:Qi*} holds if and only if $z \le -\Phi(1-\alpha_t)$. Therefore, for a uninformative prior, 
\[
\prob(\text{Equation}~\eqref{eq:Qi*} \; \text{holds}) = \alpha_t = \frac{1}{Kt}= \frac{1}{\sqrt{2\pi e} t}. 
\]

Similarly, Equation (\ref{eq:Qi}) holds if
\begin{align*}
& m_i \leq \mu_{i}^t - \frac{\sigma_s}{\sqrt{\delta^2 + n_{it}}} \Phi^{-1}(1-\alpha_t)\\
\iff & \mu_{i}^t -m_i \geq \frac{\sigma_s}{\sqrt{\delta^2 + n_{i}^t}} \Phi^{-1}(1-\alpha_t)\\
\iff & z \geq \sqrt{\frac{n_{i}^t + \delta^2}{n_{i}^t}} \Phi^{-1}(1-\alpha_t) + \frac{\delta^2}{\sigma_s} \frac{\Delta m_i}{\sqrt{n_{i}^t}},
\end{align*}
where $z \sim \mcN(0,1)$ is a standard normal random variable and $\Delta m_i = m_{i}-\mu_{i}^0$. The analogous argument to that for the above case shows that, for an uninformative prior, 
\[
\prob(\text{Equation}~\eqref{eq:Qi} \; \text{holds}) = \alpha_t = \frac{1}{Kt} =\frac{1}{\sqrt{2\pi e} t}. 
\]

Equation (\ref{eq:miComp}) holds if 
\begin{align}
&m_{i^*} < m_i + \frac{2 \sigma_s}{\sqrt{\delta^2 + n_{i}^t}}\Phi^{-1}(1-\alpha_t) \nonumber\\
\iff & \Delta_i <  \frac{ 2 \sigma_s}{\sqrt{\delta^2 + n_{i}^t}}\Phi^{-1}(1-\alpha_t) \nonumber\\
\iff & \frac{\Delta_i^2}{4 \beta ^2 \sigma_s^2}(\delta^2 + n_{i}^t) < -\log(-2\pi \alpha_t^2 \log(2\pi \alpha_t^2))\label{eq:third-condition}\\
\implies & \frac{\Delta_i^2}{4 \beta ^2 \sigma_s^2}(\delta^2 + n_{i}^t) < \log (et^2) -\log \log(et^2) \nonumber\\ 
\implies & \frac{\Delta_i^2}{4 \beta ^2 \sigma_s^2}(\delta^2 + n_{i}^t) < \log (eT^2) -\log \log(eT^2) \label{eq:monotonicity} \\
\implies & \frac{\Delta_i^2}{4 \beta ^2 \sigma_s^2}(\delta^2 + n_{i}^t) < 1 + 2 \log T -\log 2 -\log\log T \nonumber
\end{align}
where $\Delta_i = m_{i^*}-m_i$, the inequality~\eqref{eq:third-condition} follows from the bound~\eqref{eq:QDown}, and the inequality~\eqref{eq:monotonicity} follows from the monotonicity of the function $\log x -\log\log x$ in the interval $[e, +\infty)$. 
Therefore, for an uninformative prior, inequality~(\ref{eq:miComp}) never holds if 
\begin{align*}
n_{i}^t  \ge \frac{4 \beta^2 \sigma_s^2}{\Delta_i^2}(1 + 2 \log T -\log 2 -\log\log T).
\end{align*}

Setting $\eta = \lceil  \frac{4 \beta^2 \sigma_s^2}{\Delta_i^2}(1 + 2 \log T -\log 2 -\log\log T) \rceil $, we get
\begin{align*}
 \E{n_{i}^T} &\leq \eta + \sum_{t = 1}^T \prob(Q_{i}^t > Q_{i^*}^t, n_{i}^{(t-1)} \geq \eta) \\
 & = \eta + \sum_{t = 1}^T \prob(\text{Equation}~\eqref{eq:Qi*} \text{ holds},n_{i}^{(t-1)}\geq \eta)\\
 &\qquad \qquad +\sum_{t = 1}^T \prob ( \text{Equation}~\eqref{eq:Qi} \text{ holds}, n_{i}^{(t-1)} \geq \eta )\\
 &<  \frac{4 \beta^2 \sigma_s^2}{\Delta_i^2}(1 + 2 \log T -\log 2 -\log\log T) \\
& \qquad \qquad \qquad \qquad \qquad \qquad + 1 + \frac{2}{\sqrt{2\pi e}} \sum_{t = 1}^T \frac{1}{t}.
\end{align*}
The sum can be bounded by the integral
\[ \sum_{t = 1}^T \frac{1}{t} \leq 1 + \int_1^T \frac{1}{t} \mathrm{d}t = 1 + \log T,\]
yielding the bound in the first statement
\begin{multline*}
\E{n_{i}^T} \le  \Big( \frac{8 \beta^2 \sigma_s^2}{\Delta_i^2} + \frac{2}{\sqrt{2\pi e}} \Big) \log T\\
+ \frac{4 \beta^2 \sigma_s^2}{\Delta_i^2}(1  -\log 2 -\log\log T) 
 + 1 + \frac{2}{\sqrt{2\pi e}}.
\end{multline*}
The second statement follows from the definition of the cumulative expected regret.
\end{proof}

\subsection{Proof of regret of the stochastic UCL algorithm}
\begin{proof}[Proof of Theorem~\ref{thm:softmax-UCL}]
We start by establishing the first statement. 
We begin by bounding $\expt[n_{i}^T]$ as follows
\begin{align}
\E{n_{i}^{T}} &= \sum_{t=1}^T \E{P_{it}} \leq \eta + \sum_{t=1}^T \E{P_{it}\indicator{n_{i}^{t}\geq \eta}},\label{eq:niT}
\end{align}
where $\eta$ is a positive integer.

Now, decompose $\E{P_{it}}$ as
\begin{align}
\expt[P_{it}]  &=  \E{P_{it}|Q_{i}^t \leq Q_{i^*}^t} \prob(Q_{i}^t \leq Q_{i^*}^t) \nonumber \\
& \qquad \qquad\qquad \qquad +  \E{P_{it}|Q_{i}^t > Q_{i^*}^t} \prob(Q_{i}^t > Q_{i^*}^t) \nonumber \\
&\le   \E{P_{it}|Q_{i}^t \leq Q_{i^*}^t}
+ \prob(Q_{i}^t > Q_{i^*}^t).\label{eq:EPit}
\end{align}

The probability $P_{it}$ can itself be bounded as
\begin{align}
P_{it} &= \frac{\exp(Q_{i}^t /\upsilon_t)}{\sum_{j=1}^N \exp(Q_{j}^t /\upsilon_t)}
\leq \frac{\exp(Q_{i}^t /\upsilon_t)}{\exp(Q_{i^*}^t /\upsilon_t)}.\label{eq:prob-upper}
\end{align}
Substituting the expression for the cooling schedule  in inequality~\eqref{eq:prob-upper}, we obtain
\begin{align}
P_{it} \le \exp\left( -\frac{2(Q_{i^*}^t -Q_{i}^t)}{\Delta Q_{\min}^t} \log t \right) 
= t^{- \frac{2(Q_{i^*}^t -Q_{i}^t)}{\Delta Q_{\min}^t}}.\label{eq:PitUB}
\end{align}
For the purposes of the following analysis, define $\frac{0}{0} = 1$.

Since $\Delta Q_{\min}^t \geq 0$, with equality only if two arms have identical heuristic values, 
conditioned on $Q_{i^*}^t  \ge Q_{i}^t$ the exponent on $t$ can take the following magnitudes:
\[
 \frac{|Q_{i^*}^t -Q_{i}^t|}{\Delta Q_{\min}^t}  = 
\begin{cases}
\frac{0}{0} = 1, & \mbox{if } Q_{i^*}^t=Q_{i}^t, \\
+\infty, & \mbox{if } Q_{i^*}^t \neq Q_{i}^t \mbox{ and } \Delta Q_{\min}^t = 0,\\
x, & \mbox{if } \Delta Q_{\min}^t \neq 0,
\end{cases}
\]
where $x \in [1,+\infty)$. The sign of the exponent is determined by the sign of $Q_{i^*}^t-Q_{i}^t$.

Consequently, it follows from inequality~\eqref{eq:PitUB} that
\[
\sum_{t=1}^T \expt[P_{it}| Q_{i^*}^t \ge  Q_i^t] \le \sum_{t=1}^T \frac{1}{t^2} \le \frac{\pi^2}{6}. 
\]
It follows from inequality~\eqref{eq:EPit} that 
\begin{align*}
\sum_{i=1}^T \expt[P_{it}] &\le \frac{\pi^2}{6} + \sum_{i=1}^T \prob(Q_i^t > Q_{i^*}^t)\\
& \le \frac{\pi^2}{6} + \Big( \frac{8 \beta^2 \sigma_s^2}{\Delta_i^2} + \frac{2}{\sqrt{2\pi e}} \Big) \log T\\
& + \frac{4 \beta^2 \sigma_s^2}{\Delta_i^2}(1  -\log 2 -\log\log T) 
 + 1 + \frac{2}{\sqrt{2\pi e}},
\end{align*}
where the last inequality follows from Theorem~\ref{thm:UCL}. This establishes the first statement.

The second statement follows from the definition of the cumulative expected regret.
\end{proof}

\subsection{Proof of regret of the block UCL algorithm}
\begin{proof}[Proof of Theorem~\ref{thm:block-ucl}]
We start by establishing the first statement. 
For a given $t$, let $({k}_t,{r}_t)$ be the lexicographically maximum tuple such that $\tau_{{k}_t {r}_t} \le t$.
We note that
\begin{align}
n_i^{T} &= \sum_{t=1}^{T} \bs 1 (i_t=i)  \nonumber\\
&= \sum_{t=1}^{T} \big( \bs 1 (i_t=i\;  \& \; n_i^{{k}_t {r}_t} < \eta ) + \bs 1 (i_t=i \; \& \; n_i^{{k}_t {r}_t} \ge  \eta ) \big)\nonumber \\
&\le \eta + \ell +  \sum_{t=1}^{T} \bs 1 (i_t=i \; \& \; n_i^{{k}_t {r}_t} \ge \eta ) \nonumber \\
&\le   \eta + \ell +  \sum_{k=1}^{\ell} \sum_{r=1}^{b_k} k \bs 1 (i_{\tau_{kr}}=i \; \& \; n_i^{kr} \ge \eta ).  \label{eq:suboptimal-bound}
\end{align}
We note that $\bs 1 (i_{\tau_{kr}}=i) \le \bs 1 (Q_i^{kr} > Q_{i^*}^{kr})$, where $i^*$ is the optimal arm.
We now analyze the event $\bs 1 (Q_i^{kr} > Q_{i^*}^{kr})$. It follows that $\bs 1 (Q_i^{kr} > Q_{i^*}^{kr})=1$ if the following inequalities hold:
\begin{align}
 \mu_{i^*}^{kr} & \le m_{i^*}- C_{i^*}^{kr} \label{eq:cond-1}\\
 \mu_{i}^{kr} & \ge m_{i}+ C_{i}^{kr} \label{eq:cond-2}\\
 m_{i^*} &< m_{i} + 2 C_{i}^{kr},\label{eq:cond-3}
\end{align}
where $C_i^{kr}= \frac{\sigma_s}{\sqrt{\delta^2 + n_{i}^{kr}}}  \Phi^{-1}\big(1- \frac{1}{K \tau_{kr}}\big) $. Otherwise if none of the inequalities~\eqref{eq:cond-1}-\eqref{eq:cond-3} hold, then
\[
Q_i^{kr} =  \mu_i^{kr}+ C_{i}^{kr} < \mu_{i^*}^{kr} + C_{i^*}^{kr}= Q_{i^*}^{kr}.
\]
We now evaluate the probabilities of events~\eqref{eq:cond-1}-\eqref{eq:cond-3}. We note that
\begin{multline*}
\prob( \mu_{i^*}^{kr}  \le m_{i^*}- C_{i^*}^{kr}) \\
\le \prob\Bigg( z  \le \frac{\delta^2 (m_{i^*}- \mu_{i^*0})}{\sigma_s \sqrt{n_{i^*}^{kr}}} -\sqrt{ \frac{\delta^2 + n_{i^*}^{kr}}{n_{i^*}^{kr}}} \Phi^{-1} \Big( 1 - \frac{1}{K\tau_{kr}}\Big)\Bigg),
\end{multline*}
where $z \sim \mcN(0,1)$ is a standard normal random variable.   Since $\delta^2\to 0^+$ as $\sigma_0^2 \to +\infty$, it follows that
\begin{align*}
\prob( \mu_{i^*}^{kr} \! \le \! m_{i^*}- C_{i^*}^{kr}) 
\le \prob\Big( z  \le  -\Phi^{-1} \Big( 1 - \frac{1}{K\tau_{kr}}\Big)\Big) \!=\! \frac{1}{K\tau_{kr}}.
\end{align*}
A similar argument shows that $\prob(\mu_{i}^{kr} \ge m_{i}+ C_{i}^{kr}) \le 1/ K \tau_{kr}$.
We note that inequality~\eqref{eq:cond-3} holds if
\begin{align*}
 m_{i^*} &< m_{i} + 2 \frac{\sigma_s}{\sqrt{\delta^2 + n_{i}^{kr}}}  \Phi^{-1}\big(1- \frac{1}{K \tau_{kr}}\big) \\
 \implies \Delta_i & < 2 \frac{\sigma_s}{\sqrt{\delta^2 + n_{i}^{kr}}}  \Phi^{-1}\big(1- \frac{1}{K \tau_{kr}}\big)\\
 \implies \Delta_i^2 & < -4 \frac{\sigma_s^2}{\delta^2 + n_{i}^{kr}} \beta^2 \log\Big(-\frac{2\pi}{K\tau_{kr}}\log\Big(\frac{2\pi}{K^2\tau_{kr}^2}\Big)\Big)\\
 & < \frac{ 4  \beta^2 \sigma_s^2}{\delta^2 + n_{i}^{kr}} \Big( \log (e\tau_{kr}^2) - \log \log (e\tau_{kr}^2)  \Big).
\end{align*}
Since $\log x -\log \log x$ achieves its minimum at $x=e$, it follows that $ \log (e\tau_{kr}^2) - \log \log (e\tau_{kr}^2)  \le \log (eT^2) - \log \log (e T^2)$. Consequently, inequality~\eqref{eq:cond-3} holds if
\begin{align*}
\Delta_i^2 &< \frac{ 4  \beta^2 \sigma_s^2}{\delta^2 + n_{i}^{kr}} \Big( 1 + 2 \log T  - \log( 1 + 2 \log T)   \Big) \\
&< \frac{ 4 \beta^2 \sigma_s^2}{\delta^2 + n_{i}^{kr}} \Big( 1 + 2 \log T  - \log \log T -\log 2  \Big).
\end{align*}
Since $\delta^2 \to 0^+$, it follows that inequality~\eqref{eq:cond-3} does not hold if
\begin{align*}
n_i^{kr} & \ge  
\frac{8 \beta^2 \sigma_s^2}{\Delta_i^2} \Big(\log T -\frac{1}{2} \log \log T \Big) +  \frac{4 \beta^2 \sigma_s^2}{\Delta_i^2} (1 - \log 2).
\end{align*}
Therefore, if we choose 
$\eta = \lceil \frac{8 \beta^2 \sigma_s^2}{\Delta_i^2} (\log T -\frac{1}{2} \log \log T) + \frac{4 \beta^2 \sigma_s^2}{\Delta_i^2} (1 - \log 2)\rceil$, it follows from equation~\eqref{eq:suboptimal-bound} that
\begin{align}\label{eq:suboptimal-bound-exp}
\expt[n_i^T] \le  \eta + \ell + \frac{2}{K} \sum_{k=1}^{\ell} \sum_{r=1}^{b_k} \frac{k}{\tau_{kr}}.
\end{align}
We now focus on the term $\sum_{k=1}^{\ell} \sum_{r=1}^{b_k} \frac{k}{\tau_{kr}}$. We note that 
$\tau_{kr} = 2^{k-1} +(r-1) k $, and hence
\begin{align}
\sum_{r=1}^{b_k} \frac{k}{\tau_{kr}}& =  \sum_{r=1}^{b_k} \frac{k}{ 2^{k-1} +(r-1) k }\nonumber \\
& \le \frac{k}{ 2^{k-1}} + \int_{1}^{b_k} \frac{k}{k(x-1) +{2^{k-1}}} \mathrm{d} x \nonumber \\
& = \frac{k}{2^{k-1}} +\log \frac{ 2^{k-1}+ k(b_k-1)}{ 2^{k-1}} \nonumber \\ 
& \le \frac{k}{2^{k-1}} +\log  2 \label{eq:bound-1}.
\end{align}
Since $T\ge 2^{\ell-1}$, it follows that $\ell \le 1 + \log_2 T =: \bar \ell$. 
Therefore,  inequalities~\eqref{eq:suboptimal-bound-exp} and \eqref{eq:bound-1} yield
\begin{align*}
\expt[n_i^T] & \le \eta + \bar \ell + \frac{2}{K} \sum_{k=1}^{\bar \ell}\Big( \frac{k}{2^{k-1}} +\log  2 \Big)\\
& \le \eta + \bar \ell + \frac{8}{K} +  \frac{2\log 2 }{K} \bar \ell \\
& \le \gamma_1^i \log T -\frac{4 \beta ^2 \sigma_s^2}{\Delta_i^2} \log \log T +\gamma_2^i.
\end{align*}

We now establish the second statement. In the spirit of~\cite{RA-MVH-DT:88}, we note that the number of times the decision-maker transitions to arm $i$ from another arm in frame $f_k$ is equal to the number of times arm $i$ is selected in frame $k$ divided by the length of each block in frame $f_k$. Consequently, 
\begin{align*}
s_i^T & \le \sum_{k=1}^\ell \frac{n_i^{2^k}-n_i^{2^{k-1}}}{k}
 = \sum_{k=1}^\ell \frac{n_i^{{2^k}}}{k} - \sum_{k=1}^{\ell-1} \frac{n_i^{2^{k}}}{k+1}\\
 &=\frac{n_i^{2^\ell }}{\ell} +  \sum_{k=1}^{\ell-1}n_i^{2^k} \Big( \frac{1}{k} -\frac{1}{k+1} \Big) \le \frac{n_i^{2^\ell}}{\ell} +  \sum_{k=1}^{\ell-1}  \frac{n_i^{2^k}}{k^2}.
\end{align*}
Therefore, it follows that
\begin{equation}\label{eq:switching-bound}
\expt[s_i^T] \le  \frac{\expt[n_i^{2^\ell}]}{\ell} +  \sum_{k=1}^{\ell-1}  \frac{\expt[n_i^{2^k}]}{k^2}.
\end{equation}
We now analyze inequality~\eqref{eq:switching-bound} separately for the three terms in the upper bound on $\expt[n_i^T]$. For the first logarithmic term, the right hand side of inequality~\eqref{eq:switching-bound} yields
\begin{multline}\label{eq:switching-term-1}
\frac{\gamma_1^i \log 2^\ell}{\ell} + \sum_{k=1}^{\ell-1} \frac{\gamma_1^i \log 2^k}{k^2} = \gamma_1^i \log 2 \Big( 1 + \sum_{k=1}^{\ell-1} \frac{1}{k} \Big) \\
\le \gamma_1^i \log 2 (\log \log T + 2 - \log \log 2). 
\end{multline}
For the second sub-logarithmic term, the right hand side of inequality~\eqref{eq:switching-bound} is equal to 
\begin{align}
&-\frac{4 \beta ^2 \sigma_s^2}{\Delta_i^2} \Big (\frac{(\log \ell + \log \log 2) }{\ell} + \sum_{k=1}^{\ell-1} \frac{(\log k + \log \log 2) }{k^2} \Big) \nonumber \\
& \le -\frac{4 \beta ^2 \sigma_s^2}{\Delta_i^2} \Big (\frac{ (\log \log 2) }{\ell}+
\sum_{k=1}^{\ell-1} \frac{ \log \log 2}{k^2} \Big) \nonumber\\
& \le -\frac{4 \beta ^2 \sigma_s^2  }{\Delta_i^2}  \Big(1+ \frac{\pi^2}{6}\Big) \log \log 2.  \label{eq:switching-term-2}
\end{align}
Similarly, for the constant term $\gamma_2$, the right hand side of inequality~\eqref{eq:switching-bound} is equal to 
\begin{align}\label{eq:switching-term-3}
\frac{\gamma_2^i}{\ell} + \sum_{k=1}^{\ell-1} \frac{\gamma_2^i}{k^2} \le \gamma_2^i \Big( 1+ \frac{\pi^2}{6}\Big).
\end{align}
Collecting the terms from inequalities~\eqref{eq:switching-term-1}-\eqref{eq:switching-term-3}, it follows from inequality~\eqref{eq:switching-bound} that
\begin{align*}
\expt[s_i^T] \le (\gamma_1^i \log 2)  \log \log T  + \gamma_3^i.
\end{align*}

We now establish the last statement. The bound on the cumulative expected regret follows from its definition and the first statement. To establish the bound on the cumulative switching cost, we note that
\begin{align}
\sum_{t=1}^T  \supscr{S}{B}_t  &\le \sum_{i=1, i\ne i^*}^{N} {\bar c_{i}}^{\max} \expt[s_i^T] + {\bar c_{i^*}}^{\max} \expt[s_{i^*}^T] \nonumber \\
&\le  \sum_{i=1, i\ne i^*}^{N} ({\bar c_{i}}^{\max}+  {\bar c_{i^*}}^{\max} ) \expt[s_i^T] + {\bar c_{i^*}}^{\max},  \label{eq:cumulative-switching}
\end{align}
where the second inequality follows from the observation that $s_{i^*}^T \le \sum_{i=1, i\ne i^*}^T s_{i}^T +1$. The final expression follows from inequality~\eqref{eq:cumulative-switching} and the second statement.
\end{proof}

\subsection{Proof of regret of the graphical block UCL algorithm}
\begin{proof}[Proof of Theorem~\ref{thm:graph-block-ucl}]
We start by establishing the first statement. 
Due to transient selections, the number of frames until time $T$ are at most equal to the number of frames if there are no transient selections. Consequently, the expected number of goal selections of a suboptimal arm $i$ are upper bounded by the expected number of selections of arm $i$ in the block UCL Algorithm~\ref{algo:block-ucb}, i.e.,
\[
\expt[n_{\text{goal},i}^T] \le  \gamma_1^i \log T -\frac{4 \beta ^2 \sigma_s^2}{\Delta_i^2} \log \log T +\gamma_2^i.
\]
Moreover, the number of transient selections of arm $i$ are upper bounded by the total number of  transitions from an arm to another arm in the block UCL Algorithm~\ref{algo:block-ucb}, i.e.,
\[
\expt[n_{\text{transient},i}^T] \le  \sum_{i=1, i\ne i^*}^N \big(  (2 \gamma_1^i \log 2)  \log \log T  + 2 \gamma_3^i\big) + 1.
\]
The expected number of selections of arm $i$ is the sum of the expected number of transient selections and the expected number of goal selections, and thus the first statement follows.

The second statement follows immediately from the definition of the cumulative regret.
\end{proof}

\subsection{Pseudocode implementations of the UCL algorithms}
\IncMargin{.3em}
\begin{algorithm}[ht!]
  {\footnotesize
   \SetKwInOut{Input}{Input}
   \SetKwInOut{Set}{Set}
   \SetKwInOut{Title}{Algorithm}
   \SetKwInOut{Require}{Require}
   \SetKwInOut{Output}{Output}
   \Input{prior $\mc N(\bs \mu_0, \sigma_0^2 I_N)$, variance $\sigma_s^2$\;}
   \Output{allocation sequence $\seqdef{i_t}{t\in \until{T}}$\;}

   \medskip

\nl {\bf set} $n_i \leftarrow 0, \bar m_i \leftarrow 0$, for each $i\in \until{N}$\;
\smallskip

\nl {\bf set} $\displaystyle \delta^2= \frac{\sigma_s^2}{\sigma_0^2}$;  $K \leftarrow \sqrt{2 \pi e}$; $\supscr{T}{end}_0 \leftarrow 0$\;
\smallskip

\emph{\% at each time pick the arm with maximum upper credible limit}
\smallskip

   \nl \For {$\tau \in \until{T}$ }{

\nl \For {$i \in \until{N}$ }{

\nl 
$\displaystyle Q_i  \leftarrow \frac{\delta^2 \mu_{i}^0 + n_i \bar m_i}{\delta^2+n_{i}} + 
\frac{\sigma_s}{\sqrt{\delta^2 + n_{i}}}  \Phi^{-1}\Big(1- \frac{1}{K \tau}\Big)$ \;
}
\nl $i_{\tau} \leftarrow \text{argmax}\setdef{Q_i}{i\in \until{N}}$\;
\smallskip

\nl collect reward $\supscr{m}{real}$\;
\smallskip

\nl $\displaystyle \bar m_{i_\tau} \leftarrow \frac{n_{i_\tau} \bar m_{i_\tau} + m}{n_{i_\tau}+1}$\;

\smallskip

\nl $n_{i_{\tau}}\leftarrow n_{i_{\tau}}+1$ \;

}

    \caption{\textit{Deterministic UCL Algorithm}}
  \label{algo:bayes-ucb}}
\end{algorithm} 
\DecMargin{.3em}

\IncMargin{.3em}
\begin{algorithm}[ht!]
  {\footnotesize
   \SetKwInOut{Input}{Input}
   \SetKwInOut{Set}{Set}
   \SetKwInOut{Title}{Algorithm}
   \SetKwInOut{Require}{Require}
   \SetKwInOut{Output}{Output}
   \Input{prior $\mc N(\bs \mu_0, \sigma_0^2 I_N)$, variance $\sigma_s^2$\;}
   \Output{allocation sequence $\seqdef{i_t}{t\in \until{T}}$\;}

   \medskip

\nl {\bf set} $n_i \leftarrow 0, \bar m_i \leftarrow 0$, for each $i\in \until{N}$\;
\smallskip

\nl {\bf set} $\displaystyle \delta^2= \frac{\sigma_s^2}{\sigma_0^2}$;  $K \leftarrow \sqrt{2 \pi e}$; $\supscr{T}{end}_0 \leftarrow 0$\;
\smallskip

\emph{\% at each time pick an arm using Boltzmann probability distribution}
\smallskip

   \nl \For {$\tau \in \until{T}$ }{

\nl \For {$i \in \until{N}$ }{

\nl 
$\displaystyle Q_i  \leftarrow \frac{\delta^2 \mu_{i}^0 + n_i \bar m_i}{\delta^2+n_{i}} + 
\frac{\sigma_s}{\sqrt{\delta^2 + n_{i}}}  \Phi^{-1}\Big(1- \frac{1}{K \tau}\Big)$ \;
}

\nl $\Delta Q_{\min} = \min_{i,t} |Q_{i}-Q_j|$\;

\nl $\displaystyle \upsilon_{\tau} \leftarrow \frac{\Delta Q_{\min}}{2 \log \tau}$\;

\nl select $i_{\tau}$ with probability $p_i \propto \exp(Q_i/\upsilon_{\tau})$\;
\smallskip

\nl collect reward $\supscr{m}{real}$\;
\smallskip

\nl $\displaystyle \bar m_{i_\tau} \leftarrow \frac{n_{i_\tau} \bar m_{i_\tau} + m}{n_{i_\tau}+1}$\;

\smallskip

\nl $n_{i_{\tau}}\leftarrow n_{i_{\tau}}+1$ \;

}

    \caption{\textit{Stochastic UCL Algorithm}}
  \label{algo:softmax-ucl}}
\end{algorithm} 
\DecMargin{.3em}

\IncMargin{.3em}
\begin{algorithm}[ht!]
  {\footnotesize
   \SetKwInOut{Input}{Input}
   \SetKwInOut{Set}{Set}
   \SetKwInOut{Title}{Algorithm}
   \SetKwInOut{Require}{Require}
   \SetKwInOut{Output}{Output}
   \Input{prior $\mc N(\bs \mu_0, \sigma_0^2 I_N)$, variance $\sigma_s^2$ \;}
   \Output{allocation sequence $\seqdef{i_t}{t\in \until{T}}$\;}

   \medskip

\nl {\bf set} $n_i \leftarrow 0, \bar m_i \leftarrow 0$, $\forall$ $i\in \until{N}$; $\displaystyle \delta^2 \leftarrow \frac{\sigma_s^2}{\sigma_0^2}$; $K \leftarrow \sqrt{2 \pi e}$\;
\smallskip

\emph{\% at each allocation round pick the arm with maximum UCL}
\smallskip

\nl \For {$k \in \until{\ell}$ }{

\nl \For {$r \in \until{b_k}$ }{

\nl $\tau \leftarrow 2^{k-1} + (r-1) k$
\smallskip

\nl 
$\displaystyle Q_i  \leftarrow \frac{\delta^2 \mu_{i}^0 + n_i \bar m_i}{\delta^2+n_{i}} + 
\frac{ \sigma_s}{\sqrt{\delta^2 + n_{i}}} \Phi^{-1}\Big(1- \frac{1}{K \tau }\Big)$ \;
\nl $\hat i \leftarrow \text{argmax}\setdef{Q_i}{i\in \until{N}}$\;
\smallskip

\nl \If{$2^k-\tau \ge k$}{{\bf set} $i_t \leftarrow \hat i$, for each $t \in \{\tau, \ldots, \tau+k\}$\;

\smallskip

\nl collect reward $\supscr{m}{real}_t,\;  t \in \until{k} $\;
\smallskip

\nl $\displaystyle \bar m_{\hat i} \leftarrow \frac{n_{\hat i}\bar m_{\hat i} + \sum_{t=1}^k \supscr{m}{real}_t}{n_{\hat i}+k}$\;
\smallskip

\nl $n_{\hat i}\leftarrow n_{\hat i}+k$ \;

}
  \smallskip
  \nl \Else{{\bf set} $i_t \leftarrow \hat i$, for each $t \in \{\tau, \ldots, 2^k-1\}$\;
  
\smallskip

\nl collect reward $\supscr{m}{real}_t,\;  t \in \until{2^k-\tau} $\;
\smallskip

\nl $\displaystyle \bar m_{\hat i} \leftarrow \frac{n_{\hat i}\bar m_{\hat i} + \sum_{t=1}^{2^k-\tau} \supscr{m}{real}_t}{n_{\hat i}+2^k-\tau}$\;
\smallskip

\nl $n_{\hat i}\leftarrow n_{\hat i}+2^k-\tau$ \;
  
  }

}

}

    \caption{\textit{Block UCL Algorithm}}
  \label{algo:block-ucb}}
\end{algorithm} 
\DecMargin{.3em}

\IncMargin{.3em}
\begin{algorithm}[ht!]
  {\footnotesize
   \SetKwInOut{Input}{Input}
   \SetKwInOut{Set}{Set}
   \SetKwInOut{Title}{Algorithm}
   \SetKwInOut{Require}{Require}
   \SetKwInOut{Output}{Output}
   \Input{prior $\mc N(\bs \mu_0, \sigma_0^2 I_N)$, variance $\sigma_s^2$ \;}
   \Output{allocation sequence $\seqdef{i_t}{t\in \until{T}}$\;}

   \medskip

\nl {\bf set} $n_i \leftarrow 0, \bar m_i \leftarrow 0$, $\forall$ $i\in \until{N}$; $\displaystyle \delta^2 \leftarrow \frac{\sigma_s^2}{\sigma_0^2}$; $K \leftarrow \sqrt{2 \pi e}$\;
\smallskip
\nl {\bf set} $\tau \leftarrow 1$; $i_0 \leftarrow 1$\;

\smallskip


%
%
%
%

\emph{\% at each allocation round pick the arm with maximum UCL}
\smallskip

\nl \For {$k \in \until{\ell}$ }{

\nl \For {$r \in \until{b_k}$ }{


\nl 
$\displaystyle Q_i  \leftarrow \frac{\delta^2 \mu_{i}^0 + n_i \bar m_i}{\delta^2+n_{i}} + 
\frac{ \sigma_s}{\sqrt{\delta^2 + n_{i}}} \Phi^{-1}\Big(1- \frac{1}{K \tau }\Big)$ \;
\nl $\hat i \leftarrow \text{argmax}\setdef{Q_i}{i\in \until{N}}$\;
\smallskip

\emph{\% reach node $\hat i$ using the shortest path} \\
\smallskip

\nl \For{$t \in \{\tau, \ldots, \tau+p_{i_\tau \hat i} - 1\}$}{

{\bf set} $i_{t} \leftarrow P^{i_{\tau} \hat i}_{t-\tau+1}$\;
\smallskip

\nl collect rewards $\supscr{m}{real}$\;
\smallskip

\nl $\displaystyle \bar m_{i_t} \leftarrow \frac{n_{i_t}\bar m_{i_t} + \supscr{m}{real}}{n_{i_t}+1}$\;
\smallskip

\nl $n_{i_t}\leftarrow n_{i_t}+1$ \;

}
\smallskip

\nl {\bf set } $\tau \leftarrow \tau + p_{i_\tau \hat i}$\;

\smallskip

\nl \If{$2^{k-1} - (r-1) k \ge k$}
{

{\bf set} $i_t \leftarrow \hat i$, for each $t \in \{\tau, \ldots, \tau +k - 1\}$\;

\smallskip

\nl collect reward $\supscr{m}{real}_t,\;  t \in \until{k} $\;
\smallskip

\nl $\displaystyle \bar m_{\hat i} \leftarrow \frac{n_{\hat i}\bar m_{\hat i} + \sum_{t=1}^k \supscr{m}{real}_t}{n_{\hat i}+k}$\;
\smallskip

\nl $n_{\hat i}\leftarrow n_{\hat i}+k$ \;

\nl $\tau \leftarrow \tau + k$ \;
}
  \smallskip
  \nl \Else{{\bf set} $i_t \leftarrow \hat i$, for each $t \in \{2^{k-1} + (r-1)k, \ldots, 2^k-1\}$\;
  
\smallskip

\nl collect reward $\supscr{m}{real}_t,\;  t \in \until{2^{k-1} - (r-1) k} $\;
\smallskip

\nl $\displaystyle \bar m_{\hat i} \leftarrow \frac{n_{\hat i}\bar m_{\hat i} + \sum_{t=1}^{2^{k-1} - (r-1) k} \supscr{m}{real}_t}{n_{\hat i}+2^{k-1} - (r-1) k}$\;
\smallskip

\nl $n_{\hat i}\leftarrow n_{\hat i}+2^{k-1} - (r-1) k$ \;
  
\nl $\tau \leftarrow \tau + 2^{k-1} - (r-1)k$ \;
  }
}
}
    \caption{\textit{Graphical Block UCL Algorithm}}
  \label{algo:graph-ucb}}
\end{algorithm} 
\DecMargin{.3em}
\newpage

\subsection{Correction to published proofs}
\begin{abstract}
The published proofs of performance for the various UCL algorithms utilized a decomposition which studied the empirical mean reward associated with each arm conditioned on the number of times that arm was selected. In the proofs, the number of selections was taken to be independent of the other sources of stochasticity in the algorithm, which is incorrect since the selections are made as a function of the stochastic observed rewards. Here, we show that a minor modification of the proofs using a tail bound due to Garivier and Moulines \cite{AG-EM:08} corrects for this oversight.
\end{abstract}

\subsubsection{Introduction}

The UCL algorithms described in the main body of this text and published in \cite{PBR-VS-NEL:14} are heuristic-based algorithms which compute a heuristic value $Q_i^t$ for each option $i$ at each time $t$. The heuristic takes the functional form
\beq
Q_i^t = \mu_i^t + C_i^t = \mu_i^t + \sigma_i^t \Phi^{-1}(1-\alpha_t), \label{eq:heuristic}
\eeq
which defines $C_i^t$, and where $\mu_i^t$ and $\sigma_i^t$ are the algorithm's belief about the mean and standard deviation, respectively, of the reward associated with option $i$ at time $t$. The error in the proof arises from the fact that to apply concentration inequalities, we condition on the number $n_i^t$ of times that the algorithm has selected option $i$ up to time $t$. Since the arm selection policy depends on the rewards accrued, $n_i^t$ and the rewards are dependent random variables. Here, we build upon an alternative concentration inequality that accounts for this dependence and show that proofs of all the performance bounds follow a similar pattern with slight modification in the choice of $\alpha_t$ in the deterministic UCL algorithm and then show how to extend the correction to the other UCL algorithms published in \cite{PBR-VS-NEL:14}.

We employ the following concentration inequality from Garivier and Moulines \cite{AG-EM:08}  to fix the proof. Let $(X_t)_{t \geq 1}$ be a sequence of independent sub-Gaussian random variables with $\expt[X_t] = \mu_t$, i. e., $\expt[\exp(\lambda (X_t-\mu_t))] \le \exp(\lambda^2\sigma^2 /2)$ for some variance parameter $\sigma >0$. 
Consider a previsible sequence $(\epsilon_t)_{t \geq 1}$ of Bernoulli variables, i.e., for all $t > 0, \epsilon_t$ is deterministically known given $\seqdef{X_\tau}{0 < \tau <t }$. Let
\[
s^t = \sum_{s=1}^t X_s \epsilon_s , 
m^t= \sum_{s=1}^t \mu_s \epsilon_s, 
n^t = \sum_{s=1}^t \epsilon_s. 
\]

\begin{theorem} [{\cite[Theorem 22]{AG-EM:08}}]  \label{thm:boundSubGaussian}
Let $(X_t)_{t \geq 1}$ be a sequence of sub-Gaussian\footnote{The result in \cite[Theorem 22]{AG-EM:08} is stated for bounded rewards, but it extends immediately to sub-Gaussian rewards by noting that the upper bound on the moment generating function for a bounded random variable obtained using a Hoeffding inequality has the same functional form as the sub-Gaussian random variable.} independent random variables with common variance parameter $\sigma$ and let $(\epsilon_t)_{t \geq 1}$ be a previsible sequence of Bernoulli variables. Then, for all integers $t$ and all $\delta, \epsilon > 0$,
\begin{align} \Pr{\frac{s^t - m^t}{\sqrt{n^t}} > \delta}& \\
 \leq \left\lceil \frac{\log t}{\log(1 + \epsilon)} \right\rceil & \exp \left( -\frac{\delta^2}{2\sigma^2} \left( 1 - \frac{\epsilon^2}{16} \right) \right). \nonumber
\end{align}
\end{theorem}

We will also use the following bounds for $\Phi^{-1}(1-\alpha)$, the quantile function of the normal distribution. 
\begin{prop} \label{prop:QuantileBound}
For any $t \in \bbN$ and $a > 1$, the following hold:
\begin{align} \label{eq:quantileLowerBound}
\Phi^{-1}\left( 1-\frac{1}{\sqrt{2 \pi e} t^a} \right) &\geq \sqrt{\nu  \log t^a},\\
\Phi^{-1}\left( 1-\frac{1}{\sqrt{2 \pi e} t^a} \right) &\leq \sqrt{2a \log t + 2 \log \sqrt{2 \pi e}},
\label{eq:quantileUpperBound}
\end{align}
for any $0< \nu \le 1.59$.
\end{prop}
\begin{proof}
We first prove \eqref{eq:quantileLowerBound}. We begin with the inequality $\Phi^{-1}(1-\alpha)   > \sqrt{-\log(2\pi \alpha^2(1-\log(2\pi \alpha^2)))}$ established in~\cite{PBR-VS-NEL:14}.  It suffices to show that
\[
- \log \left( \frac{1}{e t^2} \left( 1 -  \log \left( \frac{1}{e t^2} \right) \right)\right) - \nu \log t  \ge 0,
\]
for $0< \nu \le 1.59$. The left hand side of the above inequality is 
\[
g(t) := 1 - \log 2 + (2- \nu) \log t - \log(1+\log t). 
\]
It can be verified that $g$ admits a unique minimum at $t = e^{(\nu -1)/(2-\nu)}$ and the minimum value is $\nu -\log 2 + \log(2-\nu)$, which is positive for $0< \nu \le 1.59$.  

Equation \eqref{eq:quantileUpperBound} is a straightforward application of Lemma 9 of \cite{PBR-VS-NEL:17}, which itself follows from \cite{MA-IAS:64}.
\end{proof}
In the following, we choose $\nu =3/2$.

\subsubsection{Correction for the deterministic UCL algorithm}
We now show that a minor change of the functional form for $\alpha_t$ corrects the oversight in the proof of \cite[Theorem 2]{PBR-VS-NEL:14}. 
Let $a > 1$ and set $\alpha_t = 1/\sqrt{2 \pi e} t^a$. Note that the originally-published version of the deterministic UCL algorithm used $\alpha_t = 1/\sqrt{2 \pi e} t$, which is equivalent to taking $a=1$ in the modified functional form of $\alpha_t$. Then the following slightly-modified version of \cite[Theorem 2]{PBR-VS-NEL:14} holds.
\begin{theorem} \label{thm:UCLCorrected}
Let $\epsilon > 0$. The following statements hold for the Gaussian multi-armed bandit problem and the deterministic UCL algorithm with uncorrelated uninformative prior and $\alpha_t = 1/Kt^a$, with $K = \sqrt{2 \pi e}$ and $a > 4/(3(1-\epsilon^2/16))$:
\begin{enumerate}
\item the expected number of times a suboptimal arm $i$ is chosen until time $T$ satisfies
\begin{align*}
\E{n_{i}^T} \le  \frac{8 a \sigma_s^2}{\Delta_i^2} \log T + o(\log T)
\end{align*}
\item the cumulative expected regret until time $T$ satisfies
\begin{multline*}
\sum_{t=1}^T R_t^{UCL} \leq \sum_{i=1}^N \Delta_i \left(  \frac{8 a \sigma_s^2}{\Delta_i^2}  \log T + o(\log T) \right).
\end{multline*}
\end{enumerate}
\end{theorem}

\begin{proof}
The structure of the proof of \cite[Theorem 2]{PBR-VS-NEL:14} carries through. As in \cite{PBR-VS-NEL:14}, a suboptimal arm $i \neq i^*$ is picked only if $Q_{i^*}^t \leq Q_i^t$. We bound $\E{n_i^T}$ as follows:
\[ \E{n_i^T} \leq \eta + \sum_{t=1}^T \Pr{Q_{i^*}^t | n_i^{(t-1)} \geq \eta}, \]
where $\eta$ is some positive integer. The condition $Q_{i^*}^t \leq Q_i^t$ is true when at least one of the three equations labeled (15--17) in \cite{PBR-VS-NEL:14} hold. For an uninformative prior and applying Proposition \ref{prop:QuantileBound}, inequality (17) in \cite{PBR-VS-NEL:14} never holds if 
\[ n_i^t \geq \frac{8 a \sigma_s^2}{\Delta_i^2} \log T + o(\log T). \]
Setting $\eta = \left \lceil \frac{8 a \sigma_s^2}{\Delta_i^2} \log T + o(\log T) \right \rceil$, we again get
\begin{align*}
 \E{n_{i}^T} &\leq \eta + \sum_{t = 1}^T \Pr{Q_{i}^t > Q_{i^*}^t, n_{i}^{(t-1)} \geq \eta} \\
 & = \eta + \sum_{t = 1}^T \Pr{\text{Equation} (15) \text{ holds},n_{i}^{(t-1)}\geq \eta}\\
 &\qquad \qquad +\sum_{t = 1}^T \Pr{\text{Equation}~(16) \text{ holds}, n_{i}^{(t-1)} \geq \eta }.
\end{align*}

Consider the probability that (15) holds. For an uncorrelated uninformative prior, $\mu_i^t = \bar{m}_i^t$ and $\sigma_i^t = \sigma_s/\sqrt{n_i^t}$. Then,
\begin{align*}
\Pr{ (15) \text{ holds}} &= \Pr{\mu_i^t - m_i \geq C_i^t} \\
&= \Pr{\frac{\mu_i^t - m_i}{\sigma_i^t} \geq \Phi^{-1}(1-\alpha_t)}.
\end{align*}
For an uncorrelated uninformative prior,
\[ \frac{\mu_i^t - m_i}{\sigma_i^t} = \frac{s_i^t/n_i^t - m_i^t/n_i^t}{\sigma_s/\sqrt{n_i^t}} = \frac{s_i^t - m_i^t}{\sigma_s \sqrt{n_i^t}}, \]
so $\Pr{ (15)\text{ holds}} = \Pr{ \frac{s_i^t - m_i^t}{\sqrt{n_i^t}} \geq \sigma_s \Phi^{-1}(1-\alpha_t)}.$ Applying Theorem \ref{thm:boundSubGaussian} with $\delta = \sigma_s \Phi^{-1}(1-\alpha_t)$, we conclude that 
\begin{align*} &\Pr{ (15) \text{ holds}} \\ &\leq  \left \lceil \frac{\log t}{\log(1 + \epsilon)} \right \rceil  \exp \left( - \frac{(\Phi^{-1}(1-\alpha_t))^2}{2} \left( 1 - \frac{\epsilon^2}{16} \right) \right). \end{align*}

Suppose that $\alpha_t = 1/\sqrt{2 \pi e} t^a$ for some $a > 1$. Then, \eqref{eq:quantileLowerBound} implies that at each time $t$,
\begin{align} \Pr{ (15) \text{ holds}} &\leq \left \lceil \frac{\log t}{\log(1 + \epsilon)} \right \rceil \exp \left( - \frac{3 a \log t}{4} \left( 1 - \frac{\epsilon^2}{16} \right) \right) \nonumber \\
& = \left \lceil \frac{\log t}{\log(1 + \epsilon)} \right \rceil t^{-\frac{3a (1-\epsilon^2/16)}{4}}. \nonumber
\end{align}
The same bound holds for (16). Thus, for all $\epsilon > 0$,  we have
\begin{align*}
\expt[n_i^T] &<  \eta + 1 
 + 2 \sum_{t = 1}^T \left \lceil \frac{\log t}{\log (1 + \epsilon)} \right \rceil t^{-\frac{3a(1-\epsilon^2/16)}{4}}.
\end{align*}
Note that the sum $\sum_{t=1}^T \Pr{ (15) \text{ holds at } t}$ can be upper bounded by the integral
\begin{align}
& \int_1^T \left( \frac{\log t}{\log(1 + \epsilon)} + 1 \right) t^{-\frac{3a (1-\epsilon^2/16)}{4}} \dd t  + 1. \label{eq:sum}
\end{align}

Furthermore, note that for $\varepsilon > 0$, the following bounds hold:
\begin{align*}
\int_1^T \log(t) t^{-(1+\varepsilon)} \dd t &= \frac{1-T^{-\varepsilon} - \varepsilon T^{-\varepsilon} \log T}{\varepsilon^2} = o(\log T)
 \\
\int_1^T t^{-(1+\varepsilon)} \dd t &= \frac{1-T^{-\varepsilon}}{\varepsilon} = o(\log T).
\end{align*}

This implies that the integral \eqref{eq:sum} is of class $o(\log T)$ as long as the exponent $3a(1-\epsilon^2/16)/4 > 1$. Let $\epsilon > 0$. Then it suffices to take $a > 4/(3(1-\epsilon^2/16))$, and $1+\varepsilon = (3a(1-\epsilon^2/16))/4$, so that $\varepsilon = (3a(1-\epsilon^2/16)/4 - 1 > 0$. Putting everything together, we have
\begin{align*}
\E{n_{i}^T} \le   \frac{8 a \sigma_s^2}{\Delta_i^2}  \log T + o(\log T).
\end{align*}
The second statement follows from the definition of the cumulative expected regret.
\end{proof}

The same correction holds for the other UCL algorithms developed in \cite{PBR-VS-NEL:14}, namely, the stochastic, block, and graphical block UCL algorithms. Let $\epsilon > 0, K = \sqrt{2\pi e},$ and set $\alpha_t = 1/(Kt^a)$, where $a > 4/(3(1-\epsilon^2/16))$. The resulting revised performance bounds are given in Table \ref{tab:correctedBounds}. 


\begin{table}[ht]
\caption{Summary of the bounds for the UCL algorithms from \cite{PBR-VS-NEL:14} after correction.}
\begin{center}
\begin{tabular}{|c|c|}
Algorithm & $\E{n_i^T} \leq $ \\ 
\hline
Deterministic UCL & $ \frac{8 a \sigma_s^2}{\Delta_i^2}  \log T + o(\log T)$\\
Stochastic UCL & $\frac{8 a \sigma_s^2}{\Delta_i^2}  \log T + \frac{\pi^2}{6} + o(\log T)$\\
Block UCL & $\left(\frac{8 a \sigma_s^2}{\Delta^2_i}+ \frac{1}{\log 2} + \frac{2}{K}\right) \log T + o(\log T)$\\
Graphical Block UCL & $\left(\frac{8 a \sigma_s^2}{\Delta^2_i}+ \frac{1}{\log 2} + \frac{2}{K}\right) \log T + o(\log T)$\\
\hline
\end{tabular}
\end{center}
\label{tab:correctedBounds}
\end{table}


\begin{thebibliography}{10}

\bibitem{PR-RCW-PH-NEL:12}
P.~Reverdy, R.~C. Wilson, P.~Holmes, and N.~E. Leonard.
\newblock Towards optimization of a human-inspired heuristic for solving
  explore-exploit problems.
\newblock In {\em Proceedings of the IEEE Conference on Decision and Control},
  pages 2820--2825, Maui, HI, USA, December 2012.

\bibitem{VS-PR-NEL:13}
V.~Srivastava, P.~Reverdy, and N.~E. Leonard.
\newblock On optimal foraging and multi-armed bandits.
\newblock In {\em Proceedings of the 51st Annual Allerton Conference on
  Communication, Control, and Computing}, pages 494--499, Monticello, IL, USA,
  2013.

\bibitem{FLL-DV-KGV:12}
F.~L. Lewis, D.~Vrabie, and K.G. Vamvoudakis.
\newblock Reinforcement learning and feedback control: Using natural decision
  methods to design optimal adaptive controllers.
\newblock {\em IEEE Control Systems Magazine}, 32(6):76--105, 2012.

\bibitem{RSS-AGB:98}
R.~S. Sutton and A.~G. Barto.
\newblock {\em Reinforcement Learning: An Introduction}.
\newblock MIT Press, 1998.

\bibitem{RB:57}
R.~Bellman.
\newblock {\em Dynamic Programming}.
\newblock Princeton University Press, Princeton, NJ, 1957.

\bibitem{PLK-MLL-AWM:96}
L.~P. Kaelbling, M.~L. Littman, and A.~W. Moore.
\newblock Reinforcement learning: A survey.
\newblock {\em Journal of Artificial Intelligence Research}, 4:237--285, 1996.

\bibitem{CW-PD:92}
C.~J. C.~H Watkins and P.~Dayan.
\newblock Q-learning.
\newblock {\em Machine learning}, 8(3-4):279--292, 1992.

\bibitem{PA-RO:07}
P.~Auer and R.~Ortner.
\newblock Logarithmic online regret bounds for undiscounted reinforcement
  learning.
\newblock In B.~Sch\"{o}lkopf, J.~Platt, and T.~Hoffman, editors, {\em Advances
  in Neural Information Processing Systems 19}, pages 49--56, Cambridge, MA,
  2007. MIT Press.

\bibitem{JDC-SMM-AJY:07}
J.~D. Cohen, S.~M. McClure, and A.~J. Yu.
\newblock Should {I} stay or should {I} go? {H}ow the human brain manages the
  trade-off between exploitation and exploration.
\newblock {\em Philosophical Transactions of the Royal Society B: Biological
  Sciences}, 362(1481):933--942, 2007.

\bibitem{JG-KG-RW:11}
J.~Gittins, K.~Glazebrook, and R.~Weber.
\newblock {\em Multi-armed Bandit Allocation Indices}.
\newblock Wiley, second edition, 2011.

\bibitem{JCG:79}
J.~C. Gittins.
\newblock Bandit processes and dynamic allocation indices.
\newblock {\em Journal of the Royal Statistical Society. Series B
  (Methodological)}, 41(2):148--177, 1979.

\bibitem{TLL-HR:85}
T.~L. Lai and H.~Robbins.
\newblock Asymptotically efficient adaptive allocation rules.
\newblock {\em Advances in Applied Mathematics}, 6(1):4--22, 1985.

\bibitem{WRT:33}
W.~R. Thompson.
\newblock On the likelihood that one unknown probability exceeds another in
  view of the evidence of two samples.
\newblock {\em Biometrika}, 25(3/4):285--294, 1933.

\bibitem{HR:52}
H.~Robbins.
\newblock Some aspects of the sequential design of experiments.
\newblock {\em Bulletin of the American Mathematical Society}, 58:527--535,
  1952.

\bibitem{MB-YS-AS:09}
M.~Babaioff, Y.~Sharma, and A.~Slivkins.
\newblock Characterizing truthful multi-armed bandit mechanisms.
\newblock In {\em Proceedings of the 10th ACM Conference on Electronic
  Commerce}, pages 79--88, Stanford, CA, USA, July 2009.

\bibitem{FR-RK-TJ:08}
F.~Radlinski, R.~Kleinberg, and T.~Joachims.
\newblock Learning diverse rankings with multi-armed bandits.
\newblock In {\em Proceedings of the 25th International Conference on Machine
  Learning}, pages 784--791, Helsinki, Finland, July 2008.

\bibitem{JLN-MD-EF:08}
J.~L. Ny, M.~Dahleh, and E.~Feron.
\newblock Multi-{UAV} dynamic routing with partial observations using restless
  bandit allocation indices.
\newblock In {\em Proceedings of the American Control Conference}, pages
  4220--4225, Seattle, Washington, USA, June 2008.

\bibitem{BPM-JJM:87}
B.~P. McCall and J.~J. McCall.
\newblock A sequential study of migration and job search.
\newblock {\em Journal of Labor Economics}, 5(4):452--476, 1987.

\bibitem{MYC-JL-FSH:13}
M.~Y. Cheung, J.~Leighton, and F.~S. Hover.
\newblock Autonomous mobile acoustic relay positioning as a multi-armed bandit
  with switching costs.
\newblock In {\em IEEE/RSJ International Conference on Intelligent Robots and
  Systems}, pages 3368--3373, Tokyo, Japan, November 2013.

\bibitem{JRK-AK-PT:78}
J.~R. Krebs, A.~Kacelnik, and P.~Taylor.
\newblock Test of optimal sampling by foraging great tits.
\newblock {\em Nature}, 275(5675):27--31, 1978.

\bibitem{RA:95}
R.~Agrawal.
\newblock Sample mean based index policies with ${O} (\log n)$ regret for the
  multi-armed bandit problem.
\newblock {\em Advances in Applied Probability}, 27(4):1054--1078, 1995.

\bibitem{PA-NCB-PF:02}
P.~Auer, N.~Cesa-Bianchi, and P.~Fischer.
\newblock Finite-time analysis of the multiarmed bandit problem.
\newblock {\em Machine Learning}, 47(2):235--256, 2002.

\bibitem{SB-NCB:12}
S.~Bubeck and N.~Cesa-Bianchi.
\newblock Regret analysis of stochastic and nonstochastic multi-armed bandit
  problems.
\newblock {\em Machine Learning}, 5(1):1--122, 2012.

\bibitem{audibert2009exploration}
J.-Y. Audibert, R.~Munos, and C.~Szepesv{\'a}ri.
\newblock Exploration--exploitation tradeoff using variance estimates in
  multi-armed bandits.
\newblock {\em Theoretical Computer Science}, 410(19):1876--1902, 2009.

\bibitem{NCB-PF:98}
N.~Cesa-Bianchi and P.~Fischer.
\newblock Finite-time regret bounds for the multiarmed bandit problem.
\newblock In {\em Proceedings of the Fifteenth International Conference on
  Machine Learning}, pages 100--108, Madison, Wisconsin, USA, July 1998.

\bibitem{AG-OC:11}
A.~Garivier and O.~Capp{\'e}.
\newblock The {KL-UCB} algorithm for bounded stochastic bandits and beyond.
\newblock In {\em JMLR: Workshop and Conference Proceedings}, volume 19: COLT
  2011, pages 359--376, 2011.

\bibitem{RD-NF-SR:98}
R.~Dearden, N.~Friedman, and S.~Russell.
\newblock Bayesian {Q}-learning.
\newblock In {\em Proceedings of the Fifteenth National Conference on
  Artificial Intelligence, AAAI-98}, pages 761--768, 1998.

\bibitem{NS-AK-SMK-MS:12}
N.~Srinivas, A.~Krause, S.~M. Kakade, and M.~Seeger.
\newblock Information-theoretic regret bounds for {G}aussian process
  optimization in the bandit setting.
\newblock {\em IEEE Transactions on Information Theory}, 58(5):3250--3265,
  2012.

\bibitem{SA-NG:12}
S.~Agrawal and N.~Goyal.
\newblock Analysis of {T}hompson {S}ampling for the multi-armed bandit problem.
\newblock In S.~Mannor, N.~Srebro, and R.~C. Williamson, editors, {\em JMLR:
  Workshop and Conference Proceedings}, volume 23: COLT 2012, pages
  39.1--39.26, 2012.

\bibitem{EK-OC-AG:12}
E.~Kaufmann, O.~Capp{\'e}, and A.~Garivier.
\newblock On {B}ayesian upper confidence bounds for bandit problems.
\newblock In {\em International Conference on Artificial Intelligence and
  Statistics}, pages 592--600, La Palma, Canary Islands, Spain, April 2012.

\bibitem{RA-MVH-DT:88}
R.~Agrawal, M.~V. Hedge, and D.~Teneketzis.
\newblock Asymptotically efficient adaptive allocation rules for the
  multi-armed bandit problem with switching cost.
\newblock {\em IEEE Transactions on Automatic Control}, 33(10):899--906, 1988.

\bibitem{JSB-RKS:94}
J.~S. Banks and R.~K. Sundaram.
\newblock Switching costs and the gittins index.
\newblock {\em Econometrica: Journal of the Econometric Society},
  62(3):687--694, 1994.

\bibitem{MA-DT:96}
M.~Asawa and D.~Teneketzis.
\newblock Multi-armed bandits with switching penalties.
\newblock {\em IEEE Transactions on Automatic Control}, 41(3):328--348, 1996.

\bibitem{TJ:04}
T.~Jun.
\newblock A survey on the bandit problem with switching costs.
\newblock {\em De Economist}, 152(4):513--541, 2004.

\bibitem{RK-AMN-YS:10}
R.~Kleinberg, A.~Niculescu-Mizil, and Y.~Sharma.
\newblock Regret bounds for sleeping experts and bandits.
\newblock {\em Machine Learning}, 80(2-3):245--272, 2010.

\bibitem{DA-PS:08}
D.~Acu{\~n}a and P.~Schrater.
\newblock {B}ayesian modeling of human sequential decision-making on the
  multi-armed bandit problem.
\newblock In B.~C. Love, K.~McRae, and V.~M. Sloutsky, editors, {\em
  Proceedings of the 30th Annual Conference of the Cognitive Science Society},
  pages 2065--2070, Washington, DC, USA, July 2008.

\bibitem{DEA-PS:10}
D.~E. Acu{\~n}a and P.~Schrater.
\newblock Structure learning in human sequential decision-making.
\newblock {\em PLoS Computational Biology}, 6(12):e1001003, 2010.

\bibitem{MS-MDL-EJW:09}
M.~Steyvers, M.~D. Lee, and E.~Wagenmakers.
\newblock A {B}ayesian analysis of human decision-making on bandit problems.
\newblock {\em Journal of Mathematical Psychology}, 53(3):168--179, 2009.

\bibitem{MDL-SZ-MM-MS:11}
M.~D. Lee, S.~Zhang, M.~Munro, and M.~Steyvers.
\newblock Psychological models of human and optimal performance in bandit
  problems.
\newblock {\em Cognitive Systems Research}, 12(2):164--174, 2011.

\bibitem{SZ-JYA:13}
S.~Zhang and A.~J. Yu.
\newblock Cheap but clever: Human active learning in a bandit setting.
\newblock In {\em Proceedings of the 35th Annual Conference of the Cognitive
  Science Society}, pages 1647--1652, Berlin, Germany, Aug 2013.

\bibitem{RCW-AG-etal:11}
R.~C. Wilson, A.~Geana, J.~M. White, E.~A. Ludvig, and J.~D. Cohen.
\newblock Why the grass is greener on the other side: Behavioral evidence for
  an ambiguity bonus in human exploratory decision-making.
\newblock In {\em Neuroscience 2011 Abstracts}, Washington, DC, November 2011.

\bibitem{DT-AN-etal:12}
D.~Tomlin, A.~Nedic, R.~C. Wilson, P.~Holmes, and J.~D. Cohen.
\newblock Group foraging task reveals separable influences of individual
  experience and social information.
\newblock In {\em Neuroscience 2012 Abstracts}, New Orleans, LA, October 2012.

\bibitem{AK-CEG:12}
A.~Krause and C.~E. Guestrin.
\newblock Near-optimal nonmyopic value of information in graphical models.
\newblock In {\em Proceedings of the Twenty-First Conference on Uncertainty in
  Artificial Intelligence}, pages 324--331, Edinburgh, Scotland, July 2005.

\bibitem{Fan2012}
P.~Fan.
\newblock New inequalities of {M}ill's ratio and its application to the inverse
  {Q}-function approximation.
\newblock {\em arXiv preprint arXiv:1212.4899}, Dec 2012.

\bibitem{DB-JNT:93}
D.~Bertsimas and J.~N. Tsitsiklis.
\newblock Simulated annealing.
\newblock {\em Statistical Science}, 8(1):10--15, 1993.

\bibitem{DM-FR-ASV:86}
D.~Mitra, F.~Romeo, and A.~Sangiovanni-Vincentelli.
\newblock Convergence and finite-time behavior of simulated annealing.
\newblock {\em Advances in Applied Probability}, 18(3):747--771, 1986.

\bibitem{SK-CDG-MPV:83}
S.~Kirkpatrick, C.~D.~Gelatt Jr., and M.~P. Vecchi.
\newblock Optimization by simulated annealing.
\newblock {\em Science}, 220(4598):671--680, 1983.

\bibitem{SMK:93}
S.~M. Kay.
\newblock {\em Fundamentals of Statistical Signal Processing, Volume I :
  Estimation Theory}.
\newblock Prentice Hall, 1993.

\bibitem{JS-WM:50}
J.~Sherman and W.~J. Morrison.
\newblock Adjustment of an inverse matrix corresponding to a change in one
  element of a given matrix.
\newblock {\em Annals of Mathematical Statistics}, 21(1):124--127, 1950.

\bibitem{buhrmester2011amazon}
M.~Buhrmester, T.~Kwang, and S.~D. Gosling.
\newblock {A}mazon's {M}echanical {T}urk: A new source of inexpensive, yet
  high-quality, data?
\newblock {\em Perspectives on Psychological Science}, 6(1):3--5, 2011.

\bibitem{NEL-DAP-FL-RS-DMF-RED:2007}
N.~E. Leonard, D.~A. Paley, F.~Lekien, R.~Sepulchre, D.~M Fratantoni, and R.~E
  Davis.
\newblock Collective motion, sensor networks, and ocean sampling.
\newblock {\em Proceedings of the IEEE}, 95(1):48--74, Jan. 2007.

\bibitem{KF-etal:13}
K.~Friston, P.~Schwartenbeck, T.~Fitzgerald, M.~Moutoussis, T.~Behrens, and
  R.~J. Dolan.
\newblock The anatomy of choice: Active inference and agency.
\newblock {\em Frontiers in Human Neuroscience}, 7:598, 2013.

\end{thebibliography}

\begin{thebibliography}{10}
\makeatletter
\addtocounter{\@listctr}{52}
\makeatother

\bibitem{AG-EM:08}
A.~Garivier and E.~Moulines.
\newblock On upper-confidence bound policies for non-stationary bandit
  problems.
\newblock {\em arXiv preprint arXiv:0805.3415}, 2008.

\bibitem{PBR-VS-NEL:14}
P.~B. Reverdy, V.~Srivastava, and N.~E. Leonard.
\newblock Modeling human decision making in generalized {G}aussian multiarmed
  bandits.
\newblock {\em Proceedings of the IEEE}, 102(4):544--571, 2014.

\bibitem{MA-IAS:64}
M.~Abramowitz and I.~A. Stegun, editors.
\newblock {\em Handbook of Mathematical Functions: with Formulas, Graphs, and
  Mathematical Tables}.
\newblock Dover Publications, 1964.

\bibitem{PBR-VS-NEL:17}
P.~B.~Reverdy, V.~Srivastava, and N.~E.~Leonard.
\newblock Satisficing in Multi-Armed Bandit Problems.
\newblock {\em IEEE Transactions on Automatic Control}, 62(8):3788--3803, Aug.
  2017.

\end{thebibliography}

\end{document}